\documentclass[sigconf,screen]{acmart}
\usepackage{xspace,balance,tabularx,multirow}
\usepackage{flushend}
\usepackage{xcolor}
\usepackage{tikz}
\usetikzlibrary{plotmarks}
\usepackage{pgfplots}
\pgfplotsset{compat=1.16}
\usetikzlibrary{patterns}
\usepackage{subfig}
\usepackage[ruled, vlined, linesnumbered]{algorithm2e}
\usepackage{colortbl}
\usepackage{bbold}
\SetKwComment{Comment}{$\triangleright$\ }{}
\usepackage{enumitem}
\usepackage{tablefootnote}
\usepackage{upgreek,textgreek}
\usepackage{pifont}
\usepackage[noabbrev]{cleveref}
\usepackage{titlecaps}
\usepackage{lipsum}
\usepackage{makecell}
\usepackage{subcaption}
\usepackage[ruled, vlined, linesnumbered]{algorithm2e}
\usepackage{booktabs}

\usepackage{amsmath}

\pgfplotsset{every tick label/.append style={font=\tiny}}

\newlength{\starsize}
\newlength{\starspread}
\tikzset{starsize/.code={\setlength{\starsize}{#1}},
         starspread/.code={\setlength{\starspread}{#1}}}
\tikzset{starsize=1mm,
         starspread=3mm}
\pgfdeclarepatternformonly[\starspread,\starsize]
  {my fivepointed stars}
  {\pgfpointorigin}
  {\pgfqpoint{\starspread}{\starspread}}
  {\pgfqpoint{\starspread}{\starspread}}
  {
   \pgftransformshift{\pgfqpoint{\starsize}{\starsize}}
   \pgfpathmoveto{\pgfqpointpolar{18}{\starsize}}
   \pgfpathlineto{\pgfqpointpolar{162}{\starsize}}
   \pgfpathlineto{\pgfqpointpolar{306}{\starsize}}
   \pgfpathlineto{\pgfqpointpolar{90}{\starsize}}
   \pgfpathlineto{\pgfqpointpolar{234}{\starsize}}
   \pgfpathclose%
   \pgfusepath{fill}
  }

\newcommand{\renchi}[1]{{\color{red}{[Renchi: #1]}}}

\newcommand{\taiyan}[1]{{\color{blue}{[Taiyan: #1]}}}

\newcommand{\argmax}[1]{\underset{#1}{\operatorname{arg}\,\operatorname{max}}\;}
\newcommand{\argmin}[1]{\underset{#1}{\operatorname{arg}\,\operatorname{min}}\;}

\makeatletter
\newcommand*\bigcdot{\mathpalette\bigcdot@{.5}}
\newcommand*\bigcdot@[2]{\mathbin{\vcenter{\hbox{\scalebox{#2}{$\m@th#1\bullet$}}}}}
\makeatother

\newcommand{\stitle}[1]{\vspace*{0.5em}\noindent{\underline{\bf #1.\/}}}

\newcommand{\V}{\mathcal{V}\xspace}
\newcommand{\G}{\mathcal{G}\xspace}

\newcommand{\N}{\mathcal{N}\xspace}
\newcommand{\Y}{\mathcal{Y}\xspace}
\newcommand{\EDG}{\mathcal{E}\xspace}

\newcommand{\WM}{\boldsymbol{W}\xspace}
\newcommand{\AM}{\boldsymbol{A}\xspace}
\newcommand{\DM}{\boldsymbol{D}\xspace}
\newcommand{\IM}{\boldsymbol{I}\xspace}

\newcommand{\PM}{\boldsymbol{P}\xspace}
\newcommand{\CM}{\boldsymbol{C}\xspace}
\newcommand{\YM}{\boldsymbol{Y}\xspace}
\newcommand{\XM}{\boldsymbol{X}\xspace}
\newcommand{\LM}{\boldsymbol{L}\xspace}
\newcommand{\UM}{\boldsymbol{U}\xspace}
\newcommand{\VM}{\boldsymbol{V}\xspace}
\newcommand{\HM}{\boldsymbol{H}\xspace}
\newcommand{\ZM}{\boldsymbol{Z}\xspace}

\newcommand{\muvec}{\boldsymbol{\mu}\xspace}
\newcommand{\SigM}{\boldsymbol{\Sigma}\xspace}

\newcommand{\RM}{\boldsymbol{R}\xspace}

\newcommand{\NAM}{\boldsymbol{\tilde{A}}\xspace}
\newcommand{\HAM}{\boldsymbol{\hat{A}}\xspace}
\newcommand{\NCM}{\boldsymbol{\tilde{C}}\xspace}

\newcommand{\HLM}{\boldsymbol{\hat{L}}\xspace}

\newcommand{\algo}{\texttt{ClustGDD}\xspace}

\newcommand{\caar}{\texttt{CAAR}\xspace}

\newcommand{\eat}[1]{}

\newenvironment{customlegend}[1][]{%
    \begingroup
    \csname pgfplots@init@cleared@structures\endcsname
    \pgfplotsset{#1}%
}{%
    \csname pgfplots@createlegend\endcsname
    \endgroup
}%

\def\addlegendimage{\csname pgfplots@addlegendimage\endcsname}

\makeatletter
\newcommand\footnoteref[1]{\protected@xdef\@thefnmark{\ref{#1}}\@footnotemark}
\makeatother

\let\oldnl\nl
\newcommand{\nonl}{\renewcommand{\nl}{\let\nl\oldnl}}

\DeclareMathOperator{\Tr}{Tr}

\definecolor{myred}{HTML}{fd7f6f}
\definecolor{myred_new}{HTML}{D8D8D8}
\definecolor{myred_new2}{HTML}{D7191C}
\definecolor{myblue}{HTML}{7eb0d5}
\definecolor{mygreen}{HTML}{D0F0C0}
\definecolor{mypurple}{HTML}{bd7ebe}
\definecolor{myorange}{HTML}{ffb55a}
\definecolor{myyellow}{HTML}{ffee65}
\definecolor{mypurple2}{HTML}{beb9db}
\definecolor{mypink}{HTML}{fdcce5}
\definecolor{mycyan}{HTML}{8bd3c7}

\definecolor{myblue2}{HTML}{AFDBF5}
\definecolor{myred2}{HTML}{c23728}

\definecolor{myorange-new}{HTML}{E56400}
\definecolor{mygreen-new}{HTML}{00788B}

\definecolor{Blue}{RGB}{173, 216, 230} %
\definecolor{DarkBlue}{RGB}{0, 191, 255}

\AtBeginDocument{%
  \providecommand\BibTeX{{%
    \normalfont B\kern-0.5em{\scshape i\kern-0.25em b}\kern-0.8em\TeX}}}

\setcopyright{acmcopyright}
\copyrightyear{2025}
\acmYear{2025}
\acmDOI{XXXXXXX.XXXXXXX}

\acmISBN{978-1-4503-XXXX-X/18/06}

\acmSubmissionID{1099}

\begin{document}




\title{Simple yet Effective Graph Distillation via Clustering}
\subtitle{Technical Report}
\author{Yurui Lai}
\affiliation{%
  \institution{Hong Kong Baptist University}
  \country{Hong Kong, China}
}
\email{csyrlai@comp.hkbu.edu.hk}

\author{Taiyan Zhang}
\affiliation{%
  \institution{Shanghaitech University}
  \country{Shanghai, China}
}
\email{zhangty2022@shanghaitech.edu.cn}

\author{Renchi Yang}
\affiliation{%
  \institution{Hong Kong Baptist University}
  \country{Hong Kong, China}
}
\email{renchi@hkbu.edu.hk}

\renewcommand{\shortauthors}{Yurui Lai, et al.}

\begin{abstract}

Despite plentiful successes achieved by graph representation learning in various domains, the training of {\em graph neural networks} (GNNs) still remains tenaciously challenging due to the tremendous computational overhead needed for sizable graphs in practice. 
Recently, {\em graph data distillation} (GDD), which seeks to distill large graphs into compact and informative ones, has emerged as a promising technique to enable efficient GNN training. However, most existing GDD works rely on heuristics that align model gradients or representation distributions on condensed and original graphs, leading to compromised result quality, expensive training for distilling large graphs, or both.
Motivated by this, this paper presents an efficient and effective GDD approach, \algo{}. Under the hood, \algo{} resorts to synthesizing the condensed graph and node attributes through fast and theoretically-grounded clustering that minimizes the {\em within-cluster sum of squares} and maximizes the {\em homophily} on the original graph. The fundamental idea is inspired by our empirical and theoretical findings unveiling the connection between clustering and empirical condensation quality using {\em Fr\'echet Inception Distance}, a well-known quality metric for synthetic images. Furthermore, to mitigate the adverse effects caused by the homophily-based clustering, \algo{} refines the nodal attributes of the condensed graph with a small augmentation learned via class-aware graph sampling and consistency loss.
Our extensive experiments exhibit that GNNs trained over condensed graphs output by \algo{} consistently achieve superior or comparable performance to state-of-the-art GDD methods in terms of node classification on five benchmark datasets, while being orders of magnitude faster. 

\end{abstract}




\maketitle

\section{Introduction}
In the past decade, {\em Graph Neural Networks} (GNNs) have emerged as a powerful model for learning on graph-structured data and found extensive practical applications in various fields, including 
molecular chemistry \cite{jiang2021could,wang2022molecular}, bioinformatics \cite{stokes2020deep,fout2017protein}, transportation \cite{derrow2021eta,jiang2022graph}, finance \cite{zhang2022efraudcom,chen2018incorporating}, recommendation systems \cite{gao2023survey,wu2022graph,borisyuk2024lignn}, etc.
Despite the remarkable success achieved, training GNN models over large-scale graphs with millions of nodes/edges, 
still remains highly challenging due to the expensive {\em message passing} operations~\cite{gilmer2017neural,zhou2023slotgat} therein, which demand tremendous computational resources and incur significant time costs \cite{hashemi2024comprehensive,10476767}.


In recent years, inspired by the success of {\em dataset distillation}~\cite{radosavovic2018data,lei2023comprehensive} in computer vision, a series of {\em graph data distillation} (GDD, a.k.a., {\em graph condensation})~\cite{jin2021graph,zheng2024structure,yang2023does,liu2024graph,liugraph,xiao2024simple, fang2024exgc,gao2024rethinking,gao2025graph}, techniques have been proposed for expediting GNN training.
In particular, GDD aims to distill a compact yet informative graph $\G^\prime$ from the original large graph $\G$ as its surrogate to train the GNN models, such that the models trained on $\G^\prime$ can achieve comparable performance to those on $\G$.
In doing so, we can circumvent the significant expense required for model training on $\G$ and enable efficient inference~\cite{gao2024graph}, unlearning~\cite{li2024tcgu}, graph continual learning~\cite{liu2023cat}, hyperparameter search~\cite{ding2022faster}, and broader applicability of GNNs.

As reviewed in \cite{gao2025graph,xu2024survey}, a major category of existing works focuses on optimizing heuristic objectives that are not directly correlated with the condensation quality, e.g., aligning the model gradients~\cite{jin2021graph,yang2023does,liu2024graph}, distributions of node representations in each GNN layer~\cite{liu2022graph}, long-term training trajectories~\cite{zheng2024structure,liu2023cat}, eigenbasis of graph structures~\cite{liugraph}, or performance of models (e.g., node classification loss)~\cite{xu2023kernel,wang2024fast} between the original graph $\G$ and synthetic graph $\G^\prime$. 
However, such complex optimizations lead to intensive computations in the course of condensation, rendering them impractical for the distillation of large graphs that pervade the real world. For instance, on the well-known {\em Reddit} graph with around $233$ thousand nodes~\cite{hamilton2017inductive}, these approaches~\cite{liu2024graph,jin2021graph,zheng2024structure} consume hours to generate the condensed graph, while the training of GNNs on the entire $\G$ merely takes a few minutes. 
Recently, several attempts~\cite{gao2024rethinking,jin2022condensing,wang2024fast} have been made towards enhancing condensation efficiency. 
\citet{jin2022condensing} propose a one-step scheme to curtail the steps needed for gradient matching, while \citet{wang2024fast} transform the bi-level optimization architecture for GDD as a {\em kernel ridge regression} (KRR) task so as to avert iteratively training GNNs.
Moreover, \eat{Ref.\taiyan{citet?}} \citet{gao2024rethinking} introduces a {\em training-free} GDD framework 
that reduces the costly node distribution matching to a tractable class partition problem.
Unfortunately, these methods trade effectiveness for higher efficiency, and hence, produce compromised condensation quality.

To overcome the aforementioned deficiencies, this paper presents \algo{} (\underline{Clust}ering-based \underline{G}raph \underline{D}ata \underline{D}istillation), an effective and efficient solution for GDD via two simple steps: {\em clustering} and {\em attribute refinement}.
\algo{} takes inspiration from our empirical observation that {\em Fr\'echet Inception Distance} (FID) \cite{dowson1982frechet}, a prominent metric for measuring the quality of synthetic images created by generative models with real images, can also accurately evaluate the condensation quality (i.e., node classification performance of GNNs) of real graphs without ground-truth node labels. 
In particular, on a condensed graph $\G^\prime$ with a low FID, GNN models trained on $\G^\prime$ always yield a high prediction accuracy in node classification, as illustrated in Fig.~\ref{fig:Acc-FID}.
Our theoretical analysis further unveils that mapping node clusters of $\G$ with the minimum {\em within-cluster sum of squares} \eat{\taiyan{capital first letter?}} (WCSS) of node representations as synthetic nodes can generate $\G^\prime$ with a bounded and optimized FID. This motivates us to construct condensed graphs with high quality through proper clustering of nodes in $\G$.


Specifically, we first develop a clustering method in \algo{} that seeks to minimize the WCSS of node representations, which can be framed as a standard $K$-Means task.
As such, the key then lies in the construction of node representations, such that nodes in $\G$ with the same ground-truth class labels are close, whereas those in distinct classes are distant. 
Doing so facilitates not only the minimization of the WCSS via $K$-Means clustering for a lower FID, but also the accurate classification of the corresponding nodes.
Using the {\em homophily assumption}~\cite{zhu2020beyond} for graphs, \algo{} generates the above-mentioned node embeddings for clustering by optimizing the {\em graph Laplacian smoothing} (GLS)~\cite{dong2016learning} and training a linear layer with label supervision. Based thereon, condensed graph topology, synthetic node attributes, and labels, can be easily derived.
On top of that, in order to mitigate the {\em heterophilic over-smoothing issue} caused by this homophily-based clustering, \algo{} additionally includes a lightweight module \caar (\underline{C}lass-\underline{A}ware \underline{A}ttribute \underline{R}efinement). More concretely, \caar refines the synthetic node attributes by injecting class-relevant features learned on class-specific graphs from $\G^\prime$ using our carefully-designed sampling technique. 
Such a small augmentation enlarges the attribute distance of heterophilic nodes (i.e., nodes with distinct class labels), thereby effectively alleviating the over-smoothing problem in node representations.
In summary, our contributions in this paper are as follows:
\begin{itemize}[leftmargin=*]
    \item We are the first to extend the FID for image data to assess the GDD quality and validate its empirical effectiveness. We establish the theoretical connection between clustering and FID optimization, which inspires the design of effective GDD.
    \item Methodologically, we propose \algo{}, an effective and efficient GDD solution that leverages a simple clustering approach aiming at optimizing the WCSS and homophily for graph synthesis, followed by a module \caar focusing on a slight refinement of the synthetic node attributes.
    \item Empirically, we conduct extensive experiments on multiple benchmark datasets to demonstrate the superior condensation effectiveness and efficiency of \algo{} over existing GDD methods.
\end{itemize}

\section{Preliminaries}

\subsection{Symbols and Terminology}\label{eq:notation}

Let $\G=(\V,\EDG,\XM)$ be an attributed graph, where $\V$ is a set of $N=|\V|$ nodes, $\EDG$ is a set of $M=|\EDG|$ edges, and $\XM\in \mathbb{R}^{N\times d}$ symbolizes the node attribute matrix. For each edge $e_{i,j}\in \EDG$, we say $v_i$ and $v_j$ are neighbors to each other and use $\N(v_i)$ to denote the set of neighbors of $v_i$, whose degree is $d(v_i)=|\N(v_i)|$. Each node $v_i\in \V$ is associated with a $d$-dimensional attribute vector $\XM_i$. The adjacency matrix of $\G$ is denoted as $\AM \in \{0,1\}^{N\times N}$, wherein $\AM_{i,j}=\AM_{j,i}=1$ if $(v_i,v_j)\in \EDG$, and $0$ otherwise. The degree matrix of $\G$ is represented by $\DM$, whose $i$-th diagonal entry $\DM_{i,i}:=d(v_i)$. 
Accordingly, the normalized adjacency matrix of $\G$ is represented by $\NAM=\DM^{-1/2}\AM\DM^{-1/2}$ and $\IM-\NAM$ is known as its graph Laplacian.
We use $\textsf{tr}(\cdot)$ to denote the trace of a matrix.

\stitle{Homophily Ratio} 
Given a set of $K$ classes $\mathcal{Y}$ ($K:=|\Y|$) and the class label $y_i\in \mathcal{Y}$ for each node $v_i\in \V$, the {\em homophily ratio} \cite{zhu2020beyond} of $\G$ is calculated by \(\Omega(\G)=\frac{|\{(v_i,v_j)\in \EDG: y_i=y_j\}|}{m}\), which quantifies the fraction of homophilic edges that connect nodes of the same classes~\cite{zhu2020beyond}. The label matrix of $\G$ is denoted as $\YM\in \{0,1\}^{N\times K}$, wherein $\YM_{i,j}=1$ if $y_i=j$ and $0$ otherwise.


\stitle{Fr\'echet Inception Distance} {\em Fr\'echet Inception Distance} (FID)~\cite{dowson1982frechet} is originally used to measure the discrepancy between two multivariate normal distributions and is later adopted as the standard metric for assessing the quality of synthetic images from generative models~\cite{heusel2017gans}. Let $\muvec^{\text{org}}$ and $\muvec^{\text{syn}}$ be the mean of the Inception embeddings of real and synthetic images, respectively, and $\SigM^{\text{org}}$, $\SigM^{\text{syn}}$ be their respective covariances between features. The FID is expressed in closed form as
\begin{equation}\label{eq:fid}
\textsf{FID} = \|\muvec^{\text{org}} - \muvec^{\text{syn}}\|^2_2 + \textsf{tr}\left(\SigM^{\text{org}} + \SigM^{\text{syn}} - 2(\SigM^{\text{org}} \SigM^{\text{syn}})^{\frac{1}{2}}\right),
\end{equation}
where the first term calculates the average distance between two embedding spaces to represent the overall distribution shift, while the second term involves the trace of the covariance matrices and the geometric mean of the covariance matrices, measuring the difference in the shape and spread of the feature distributions of original and synthetic graphs. 
\eat{\renchi{?}}
A lower FID value indicates better synthetic quality, and the generative model is perfect when $\Phi=0$. 


\stitle{Graph Laplacian Smoothing}
Given graph $\G$, {\em graph Laplacian smoothing} (GLS)~\cite{dong2016learning} is to optimize $\ZM$ such that
\begin{small}
\begin{equation}\label{eq:GNN-obj}
\min_{\ZM}{(1-\alpha)\cdot\|\ZM-\XM\|^2_F+\alpha \sum_{(v_i,v_j)\in \EDG}{\left\|\frac{\ZM_i}{\sqrt{d(v_i)}}-\frac{\ZM_j}{\sqrt{d(v_j)}}\right\|_F^2}}.
\end{equation}
\end{small}
Therein, 
the fitting term $\|\ZM-\XM\|^2_F$ in Eq.~\eqref{eq:GNN-obj} seeks to make the node features $\ZM$ close to the initial attributes $\XM$, while the graph Laplacian regularization term  $\textstyle \sum_{(v_i,v_j)\in \EDG}{\left\|\frac{\ZM_i}{\sqrt{d(v_i)}}-\frac{\ZM_j}{\sqrt{d(v_j)}}\right\|_F^2}$ forces feature vectors of two adjacent nodes in $\G$ to be similar. The hyperparameter $\alpha\in [0,1]$ controls the smoothness of $\ZM$ over $\G$.
As revealed in recent studies~\cite{ma2021unified,Zhu2021InterpretingAU}, after removing non-linear operations and linear transformations, graph convolutional operations in most GNN models~\cite{gasteiger2018predict,kipf2016semi,defferrard2016convolutional,xu2018representation} essentially are to optimize the GLS objective in Eq.~\eqref{eq:GNN-obj}.


\begin{figure*}[!t]
    \centering
    \includegraphics[width=\textwidth]{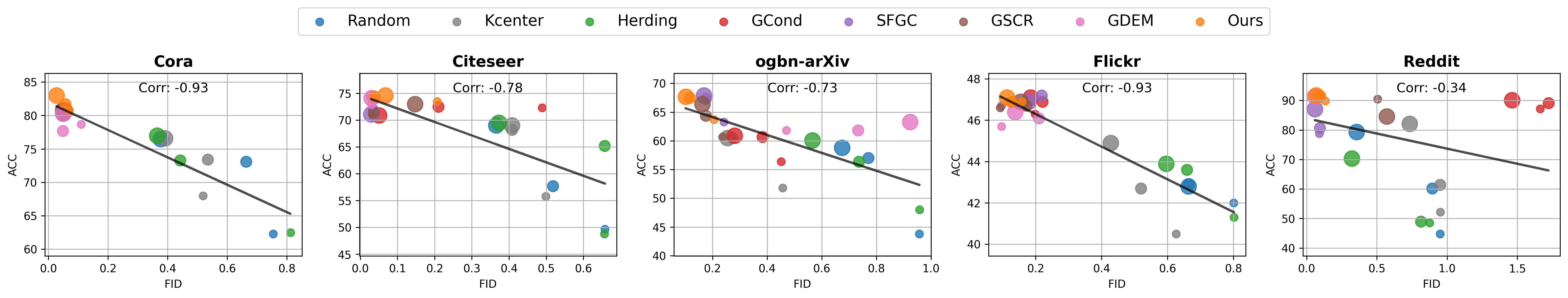}
    \vspace{-5ex}
    \caption{Classification accuracy v.s. FID under various condensation ratios.}
    \label{fig:Acc-FID}
    \vspace{-2ex}
\end{figure*}

\subsection{Graph Data Distillation (GDD)}\label{sec:GDD-background}
Given an input graph $\G=(\V,\EDG,\XM)$, GDD is to distill from $\G$ a condensed graph $\G^\prime$ with a set $\VM^\prime=\{u_1,u_2,\ldots,u_n\}$ of $n$ ($n\ll N$) nodes, a set $\EDG^\prime$ of $m$ ($m\ll M$) edges, and synthetic attribute vectors $\XM^\prime\in \mathbb{R}^{n\times d}$ for $n$ nodes, and the adjacent matrix $\AM^\prime\in \{0,1\}^{n\times n}$. \eat{\taiyan{Add one sentence to explain $A^\prime$, like: Accordingly, the adjacent matrix for $G^\prime$ is $A^\prime \in \{0,1\}^{n\times n}$}}
Each node $u_i\in \G^\prime$ is associated with a synthetic class label $y_i\in \Y$ and the label matrix of $\G^\prime$ is symbolized by $\YM^\prime\in \mathbb{R}^{n\times K}$.
In the meantime, GNNs trained on the condensed graph $\G^\prime$ should achieve comparable performance to those on the original graph $\G$. Let $\mathcal{L}(\cdot,\cdot)$ be an inaccuracy metric for node classification tasks. In mathematical terms, the optimization goal of GDD can be formulated as follows~\cite{liu2024graph}:
\begin{equation*}
\begin{gathered}
\min \quad\mathcal{L}(\textsf{GNN}_{\boldsymbol{\Theta}_{\G^\prime}}(\AM, \XM), \YM) \\
s.t. \quad \boldsymbol{\Theta}_{\G^\prime} = \argmin{\boldsymbol{\Theta}}{\mathcal{L}(\textsf{GNN}_{\boldsymbol{\Theta}}(\AM^\prime, \XM^\prime), \YM^\prime)},
\end{gathered}
\end{equation*}
where $\textsf{GNN}_{\boldsymbol{\Theta}}(\cdot)$ denotes a GNN model parameterized with $\boldsymbol{\Theta}$ and $\boldsymbol{\Theta}_{\G^\prime}$ are the model parameters obtained via training the GNN over the condensed graph $\G^\prime$. To get detailed information of representative types of GDD, please check Appendix ~\ref{appendix:formalize_GDD}. 

\eat{\renchi{give a categorization based on Appendix C in \url{https://arxiv.org/pdf/2310.09202} and Section 2.2 in \url{https://arxiv.org/pdf/2405.13707} (only these representative GDD solutions. The comprehensive review should be in Appendix A.). Summarize them in a concise way and better to include the formulas therein. For each category, point out their limitations and why they are slow.}}




\section{Condensation Quality Analysis}
In this section, we first extend FID to assess the quality of condensed graphs by GDD methods and validate its empirical effectiveness on real datasets. Next, we conduct a theoretical analysis of FID to pinpoint a design principle for effective GDD.


\subsection{FID as Condensation Quality Metric}
Given the original graph $\G$ and the condensed graph $\G^\prime$, we propose to evaluate the condensation quality of $\G^\prime$ 
via the FID of node representations $\HM$ and $\HM^\prime$ learned by GNNs (e.g., \texttt{GCN}) over $\G$ and $\G^\prime$.
Specifically, we define
\begin{equation}\label{eq:mu-org-syn}
\muvec^{\text{org}} = \frac{1}{N}\sum_{v_i\in  \G}{\HM_{i}}\ \text{and}\ \muvec^{\text{syn}} = \frac{1}{n}\sum_{u_i\in  \G'}{\HM'_{i}}.
\end{equation}
Let $\hat{\HM}$ and $\hat{\HM'}$ represent the centered versions of $\HM$ and $\HM'$, respectively. The covariance matrices $\SigM^{\textnormal{org}}$ and $\SigM^{\textnormal{syn}}$ of $\HM$ and $\HM^{\prime}$ then can be formulated as
\begin{equation}\label{eq:conv}
\SigM^{\textnormal{org}} = \frac{1}{N} {\hat{\HM}}^{\top}\hat{\HM}\ \text{and}\ \SigM^{\textnormal{syn}} = \frac{1}{n} {\hat{\HM'}}^\top \hat{\HM'}.
\end{equation}
The FID of $\HM$ and $\HM^\prime$ can thus be computed as per Eq.~\eqref{eq:fid}.

Accordingly, the first term \(\|\muvec^{\text{org}} - \muvec^{\text{syn}}\|^2_2\) in the FID quantifies the dissimilarity of these two node representations, reflecting how much on average the condensed graph $\G^\prime$ resembles $\G$ and captures the key structural and attribute characteristics therein.
On the other hand, recall that a larger covariance value in \(\SigM^{\textnormal{org}}\) or \(\SigM^{\textnormal{syn}}\) indicates a wider spread of feature values, while off-diagonal elements represent correlations between different feature dimensions. The second term \(\Tr\left(\SigM^{\text{org}} + \SigM^{\text{syn}} - 2(\SigM^{\text{org}} \SigM^{\text{syn}})^{\frac{1}{2}}\right)\) is thus the difference in the spread of individual feature dimensions of $\HM$ and $\HM^\prime$ and correlations between them, measuring how different the ``shape'' of the feature distributions underlying $\G$ and $\G^\prime$ are.
Intuitively, a lower FID connotes a closer match between these two data distributions, and hence, a higher condensation quality of $\HM^\prime$.

\eat{
In order to measure the quality of condense graph, we firstly calculate the Fr\'echet Inception Distance (FID)~\cite{dowson1982frechet}, which computes the distance in the embedding space between two multivariate Gaussian distributions fitted to the condensed graph and original graph. The embedding can be obtained by pretraining the \texttt{GCN} or other GNNs. 
\begin{small}
\begin{equation}
    \label{eq:fid}
    \psi = \|\muvec^{\text{org}} - \muvec^{\text{syn}}\|^2_2 + \Tr\left(\SigM^{\text{org}} + \SigM^{\text{syn}} - 2(\SigM^{\text{org}} \SigM^{\text{syn}})^{\frac{1}{2}}\right)
\end{equation}
\end{small}
where $\mu$ is the mean of GNN's representations, $\Sigma$ is covariance matrix between dimensions. A lower FID value indicates better condensation quality and diversity. 

For the \texttt{GCN} trained on the synthetic dataset, and evaluate its accuracy $\xi$ on the training set of raw data. The final condensation-quality can be calculated by
\begin{equation}
\textsf{FCQ}(\G,\G^\prime) = \xi \log({1}/{\Phi}+1)
\end{equation}

As the Eq.(\ref{eq:fid}) shows, in order to decrease between the FID of original graph and synthesized graphs, we not only need to lower the distance between the mean of representations, but also make the their covariance matrices close. A straightforward way is to cluster node representations and merge nodes with similar representations into a supernode.
}

\stitle{Empirical Study} To verify the empirical effectiveness of FID in measuring the condensation quality, we experimentally evaluate the GNN performance on the condensed graphs $\G^\prime$ obtained via various GDD approaches over real graph datasets in downstream tasks (i.e., node classification) and their corresponding FID values.\footnote{In practice, we normalize the representations to control the range of FID values.} 
Fig.~\ref{fig:Acc-FID} plots the classification accuracies and FID scores of the \texttt{Random}~\cite{welling2009herding}, \texttt{Kcenter}~\cite{sener2017active}, and \texttt{Herding}~\cite{welling2009herding} coreset methods, as well as GDD approaches including \texttt{GCond}~\cite{jin2021graph}, \texttt{SFGC}~\cite{zheng2024structure}, \texttt{CSCR}~\cite{liu2024graph}, \texttt{GDEM}~\cite{liugraph}, and our proposed \algo{} on {\em Citeseer}, {\em Cora}, {\em arXiv}, and {\em Flickr}. 
Each method is represented by three bubbles, which correspond to their results under three different condensation ratios commonly used in previous works~\cite{jin2021graph}. 
It can be observed from Fig.~\ref{fig:Acc-FID} that in almost all cases, a method with a low FID will yield high accuracy in the downstream node classification tasks, implying a linear correlation between the FID of $\G^\prime$ and the GNN model performance over it. In turn, FID can serve as a precise quality metric for condensed graphs practically, even without ground-truth labels.

\subsection{Theoretical Inspiration}\label{sec:inspire}

Ideally, a high-quality condensed graph $\G^\prime$ should minimize the first term \(\|\muvec^{\text{org}} - \muvec^{\text{syn}}\|^2_2\) in the FID, meaning that
\begin{equation*}
\sum_{u_i\in  \G'}{\HM'_{i}} \approx \sum_{v_i\in  \G}{\frac{n}{N}\cdot\HM_{i}}.
\end{equation*}
Intuitively, the minimization can be attained when the representation $\HM'_{i}$ of each node $u_i\in \G^\prime$ is a weighted summation of the representations of a subset $C_i$ of nodes in $\G$, i.e., $\HM'_{i}=\sum_{v_j\in  C_i}{w_j\cdot\HM_{j}}$,
such that $\cup_{i=1}^n C_i=\V$ and $C_i\cap C_j=\emptyset\ \forall{i\neq j,\ 1\le i,j\le n}$.
As such, each node $u_i$ in $\G^\prime$ is a supernode merged from the nodes inside a cluster $C_i$. 
Along this line, 
a simple and straightforward approach to constructing the condensed graph $\G^\prime$ is to partition the input graph $\G$ into $n$ disjoint clusters $\{C_1,\ldots,C_{n}\}$.

\begin{theorem}\label{lem:mu}
\(\|\muvec^{\textnormal{org}} - \muvec^{\textnormal{syn}}\|^2_2 \le \frac{1}{N^2}\sum_{i=1}^n{\left( \frac{N}{n}-{|C_i|}\right)^2}\).
\end{theorem}

When we let $u_i$'s representation be the averaged embeddings of the nodes inside cluster $C_i$, i.e.,
${\HM^\prime_{i}}=\sum_{v_j\in  C_i}{\frac{\HM_{j}}{|C_i|}}$, our Theorem~\ref{lem:mu} reveals that the first term \(\|\muvec^{\text{org}} - \muvec^{\text{syn}}\|^2_2\) in the FID can be bounded by \(\frac{1}{N^2}\sum_{i=1}^n{\left( \frac{N}{n}-{|C_i|}\right)^2}\). Since $\sum_{i=1}^n{|C_i|}=N$, the mean cluster size is thus $\frac{N}{n}$. This upper bound is essentially the {\em variance} of the sizes of $n$ clusters $\{C_1,\ldots,C_{n}\}$ scaled by $\frac{1}{N}$. In particular, \(\|\muvec^{\text{org}} - \muvec^{\text{syn}}\|^2_2\) is minimized, i.e., \(\|\muvec^{\text{org}} - \muvec^{\text{syn}}\|^2_2\approx 0\), when the size of every cluster $C_i$ is approximately $\frac{N}{n}$.




\begin{theorem}\label{lem:conv-bound}
Let $c_{\max}=\underset{1\le i\le n}{\max}{|C_i|}$ and $c_{\min}=\underset{1\le i\le n}{\min}{|C_i|}$. Then, 
\(\Tr\left(\SigM^{\textnormal{org}} + \SigM^{\textnormal{syn}} - 2(\SigM^{\textnormal{org}} \SigM^{\textnormal{syn}})^{\frac{1}{2}}\right) \leq \frac{1}{N} \sum_{i=1}^n\sum_{v_j\in C_i} \|\HM_j -\HM'_i\|^2_2 + \frac{nc_{\max}}{N}\cdot\|\muvec^{\textnormal{org}} - \muvec^{\textnormal{syn}}\|^2_2 + \left(\frac{c_{\max}}{c_{\min}}+\frac{N}{nc_{\min}}\right)\cdot \Tr(\SigM^{\textnormal{org}})\) holds.
\end{theorem}

As for the second term \(\Tr\left(\SigM^{\text{org}} + \SigM^{\text{syn}} - 2(\SigM^{\text{org}} \SigM^{\text{syn}})^{\frac{1}{2}}\right)\) in the FID, Theorem~\ref{lem:conv-bound} states that its upper bound is positively correlated with \(\frac{1}{N} \sum_{i=1}^n\sum_{v_j\in C_i} \|\HM_j -\HM'_i\|^2_2\), the first term in FID (i.e., \(\|\muvec^{\textnormal{org}} - \muvec^{\textnormal{syn}}\|^2_2\)), and $\Tr(\SigM^{\textnormal{org}})$. Notice that \(\|\muvec^{\textnormal{org}} - \muvec^{\textnormal{syn}}\|^2_2\) can be bounded by Theorem~\ref{lem:mu}, and $\Tr(\SigM^{\textnormal{org}})$ is solely determined by the input graph $\G$, and hence, can be regarded as a constant. 
Therefore, we can reduce the second term \(\Tr\left(\SigM^{\text{org}} + \SigM^{\text{syn}} - 2(\SigM^{\text{org}} \SigM^{\text{syn}})^{\frac{1}{2}}\right)\) and further the FID, by minimizing \(\frac{1}{N} \sum_{i=1}^n\sum_{v_j\in C_i} \|\HM_j -\HM'_i\|^2_2\), 
which is essentially the WCSS, i.e., the variance of feature representations within each cluster.

\eat{
it can be upper bounded by $(1+\frac{N}{n c_{\min}})\cdot \Tr(\SigM^{\textnormal{org}})$, where $\Tr(\SigM^{\textnormal{org}})$ is only determined by the input graph $\G$. 
Notice that $c_{\min}$ denotes the size of the smallest cluster, which is at most $\lfloor\frac{N}{n}\rfloor$. 
Accordingly, 
\(\Tr\left(\SigM^{\text{org}} + \SigM^{\text{syn}} - 2(\SigM^{\text{org}} \SigM^{\text{syn}})^{\frac{1}{2}}\right)\) will be small when $c_{\min}$ approximates $\lfloor\frac{N}{n}\rfloor$, indicating that the sizes of all clusters are comparable.
}

In a nutshell, the foregoing analyses suggest a promising way to construct the condensed graph $\G^\prime$ with a bounded and low FID through partitioning the original node set $\V$ of $\G$ into clusters, particularly with balanced sizes and small WCSS, as the $n$ synthetic nodes in $\G^\prime$. For the proofs of the theorems, please refer to the Appendix~\ref{sec:proof} for details.

\section{Methodology}
In this section, we present \algo{} for distilling $\G$ as $\G^\prime$. We first elaborate on our clustering method for creating the condensed graph structure $\AM^\prime$, synthetic attribute matrix $\XM^\prime$, and node labels $\YM^\prime$ in Section~\ref{sec:clustering}. In Section~\ref{sec:refinement}, we further pinpoint the defects of $\XM^\prime$ generated by the clustering and delineate a class-aware scheme \caar for refining $\XM^\prime$. To facilitate a better comprehension of \algo{}, we present its overall framework in Fig. \ref{fig:overview}.


\begin{figure}[!t]
    \centering
    \includegraphics[width=0.9\columnwidth]{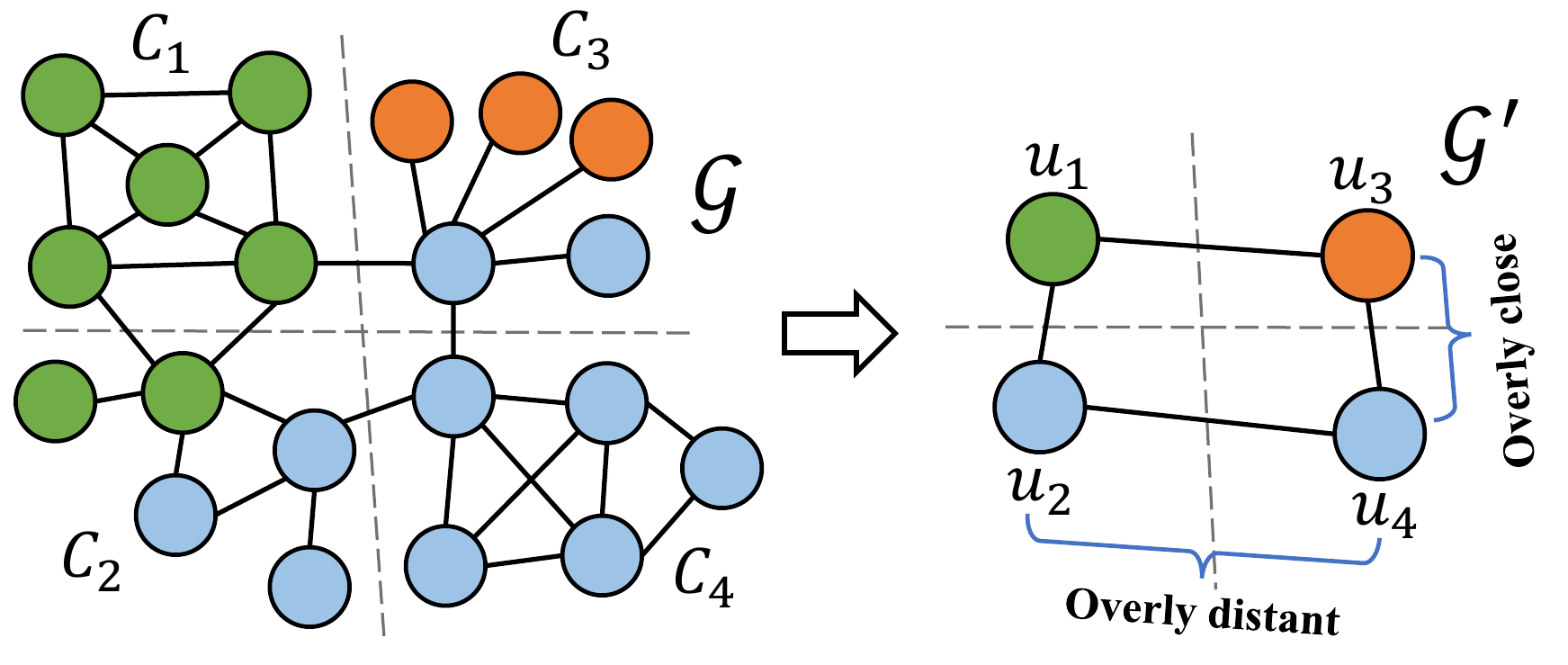}
    \vspace{-3ex}
    \caption{Example of a balanced clustering of $\G$.}
    \label{fig:imbalance-class}
    \vspace{-2ex}
\end{figure}

\begin{figure}[!t]
    \centering
    \includegraphics[width=1.0\columnwidth]{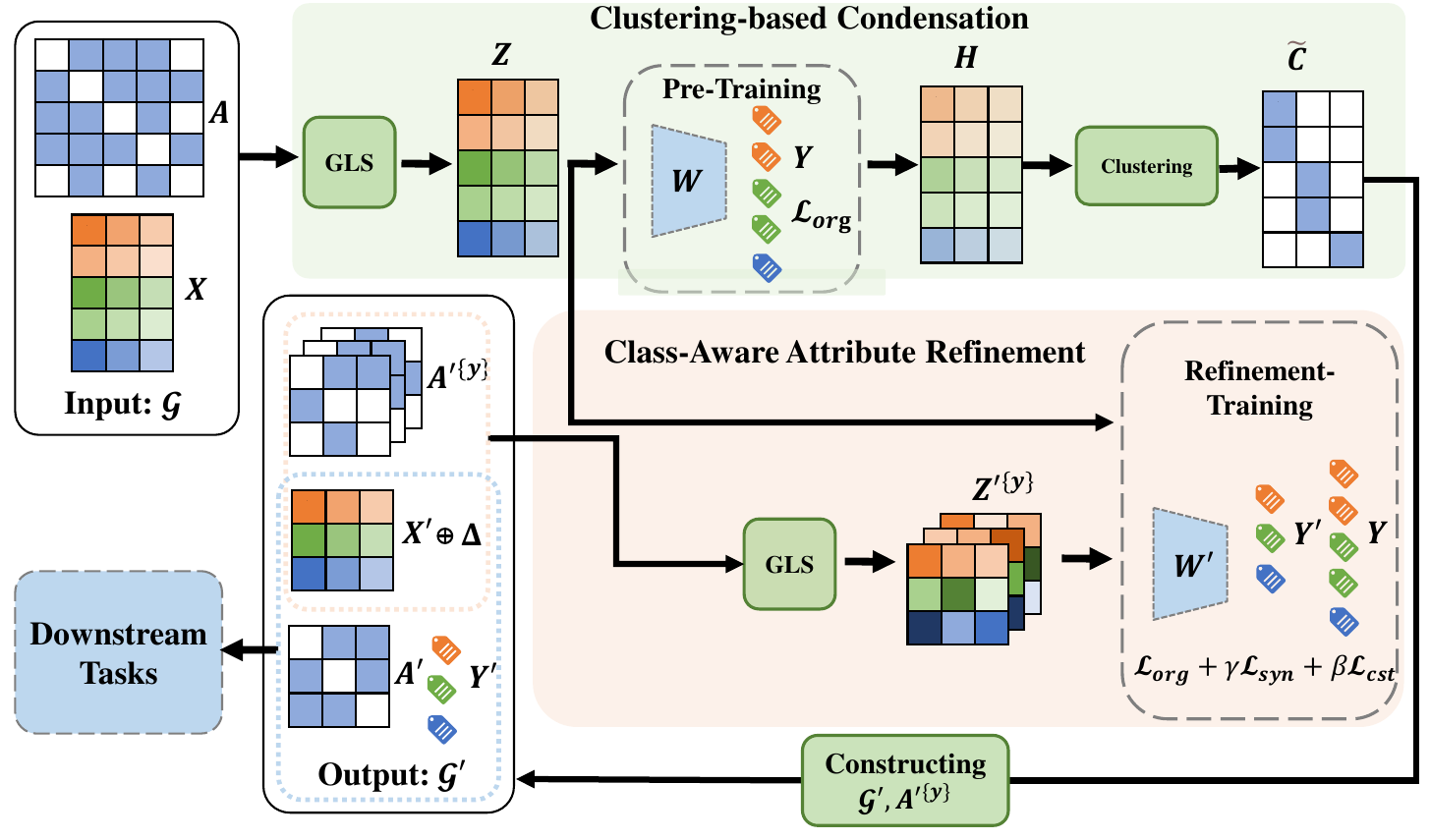}
    \vspace{-3ex}
    \caption{An overview of \algo{}.}
    \label{fig:overview}
    \vspace{-2ex}
\end{figure}

\subsection{Clustering-based Condensation}\label{sec:clustering}
Let $\{C_1,\ldots,C_n\}$ be a partition of $n$ clusters of $\G$ and each synthetic node $v_i$ in $\G^\prime$ be merged from the cluster $C_i$, i.e., 
\begin{equation}\label{eq:Hi}
{\HM^\prime_{i}}=\sum_{v_j\in  C_i}{\frac{\HM_{j}}{|C_i|}}
\end{equation}
According to our analysis in Section~\ref{sec:inspire}, a low FID between $\G$ and $\G^\prime$ can be obtained with clusters $\{C_1,\ldots,C_n\}$ that optimize
\begin{equation}\label{eq:WCSS}
\min_{C_1,C_2,\ldots,C_n} \sum_{i=1}^n\sum_{v_j\in C_i} \|\HM_j -\HM'_i\|^2_2
\end{equation}
and meanwhile minimize the cluster size variance \(\sum_{i=1}^n{\left( \frac{N}{n}-{|C_i|}\right)^2}\). However, the latter optimization objective is likely to group nodes with various class labels into $C_i$ (i.e., the synthetic node $u_i$) when the ground-truth classes in $\G$ are imbalanced, which is often the case in practice. As exemplified in Fig.~\ref{fig:imbalance-class}, this will render some synthetic nodes with distinct synthetic class labels in $\G^\prime$ overly distant (e.g., $u_2$ and $u_4$) or close (e.g., $u_3$ and $u_4$), and thus, result in degraded condensation quality.
Instead, an ideal cluster $C_i$ (i.e., synthetic node $u_i$) should comprise original nodes associated with the same ground-truth label, which requires the representations of nodes within $C_i$ or the same ground-truth class to be highly similar. By generating such node representations, each term $\sum_{v_j\in C_i} \|\HM_j -\HM'_i\|^2_2$ in our objective in Eq.\eqref{eq:WCSS} can be implicitly minimized.


\stitle{Clustering via Optimizing the WCSS and Homophily}
Given node representations $\HM$ of $\G$, the minimization of the WCSS in Eq.\eqref{eq:WCSS} can be efficiently solved using $K$-Means algorithm~\cite{lloyd1982least}. Next, we elucidate the construction of $\HM$ and how we can enforce nodes with the same class labels to be close in $\HM$.

\begin{lemma}\label{lem:GSL-sol}
The closed-form solution to the GLS optimization problem in Eq.~\eqref{eq:GNN-obj} is $\sum_{t=0}^{\infty}{(1-\alpha)\alpha^t\NAM^t}\XM$.
\end{lemma}
As pinpointed in Section~\ref{eq:notation}, the representation learning in most GNNs can be unified into a framework optimizing the GLS problem in Eq.~\eqref{eq:GNN-obj} or its variants. 
As per the result in Lemma~\ref{lem:GSL-sol}, we then construct $\HM$ by Eq.~\eqref{eq:h=zw} as $\ZM$ is the closed-form solution to Eq.~\eqref{eq:GNN-obj} when $T\rightarrow\infty$.
\begin{equation}\label{eq:h=zw}
\HM = \ZM\WM\ \text{where $\ZM=\sum_{t=0}^{T}{(1-\alpha)\alpha^t\NAM^t}\XM$}.
\end{equation}
Particularly, the weights $\WM\in \mathbb{R}^{d\times K}$ are learned by minimizing the following cross-entropy loss:
\begin{equation}\label{eq:org-loss}
\mathcal{L}_{org} =  -\frac{1}{N}\sum^{N}_{i=1} \sum^{K}_{y=1} \YM_{i,y}\cdot\textsf{log}(\PM_{i,y})
\end{equation}
with the label predictions $\PM\in \mathbb{R}^{N\times K}$ transformed from $\HM$ via
\begin{equation}\label{eq:logit_softmax}
\PM = \textsf{Softmax}(\HM),
\end{equation}
and ground-truth labels for the training set in a supervised fashion



\begin{lemma}\label{lem:homophily-DE}
\(\Omega(\G)=1 - \frac{1}{2M}\cdot \sum_{(v_i,v_j)\in \EDG}{\left\|\frac{(\DM^{1/2}\YM)_i}{\sqrt{d_i}}-\frac{(\DM^{1/2}\YM)_j}{\sqrt{d_j}}\right\|_2^2}\).
\end{lemma}
If we regard $\HM$ as the label predictions, our Lemma~\ref{lem:homophily-DE} implies that the second term \(\sum_{(v_i,v_j)\in \EDG}{\left\|\frac{\HM_i}{\sqrt{d(v_i)}}-\frac{\HM_j}{\sqrt{d(v_j)}}\right\|_F^2}\) in GLS (Eq.~\eqref{eq:GNN-obj}) is essentially equivalent to maximizing the homophily ratio of the predicted labels over the input graph $\G$, whose true homophily ratio is usually assumed to be high in practice, i.e., adjacent nodes are very likely to share the same class labels.
Based on this homophily assumption, our way of computing $\HM$ in Eq.~\eqref{eq:h=zw} tends to render the representations of nodes with the same class labels close to each other and easier to be grouped into the same clusters by the $K$-Means, thereby minimizing the WCSS in Eq.~\eqref{eq:WCSS}.

\stitle{Constructing $\XM^{\prime}$, $\AM^{\prime}$, and $\YM^\prime$}
Denote by $\CM\in \{0,1\}^{N\times n}$ the node-cluster membership matrix, where $\CM_{j,i}=1$ if node $v_j\in C_i$ and $0$ otherwise. Accordingly, we can form a sketching matrix $\NCM$ by normalizing $\CM$:
\(\NCM = \CM \textsf{diag}(\mathbf{1}\CM)^{-1},\)
where $\NCM_{j,i}=\frac{1}{|C_i|}$ if node $v_j\in C_i$ and $0$ otherwise. By Eq.~\eqref{eq:Hi}, \(\HM^\prime = \NCM^\top \HM\).

Intuitively, \(\HM^\prime\) should be built from the condensed adjacency matrix $\AM^\prime$ and attribute matrix $\XM^\prime$ using Eq.~\eqref{eq:h=zw}. Along this line, we can formulate the following objective function:
\begin{equation}\label{eq:XA-obj}
\min_{\AM^\prime,\XM^\prime}\left\|\NCM^\top \HM- \sum_{t=0}^{T^\prime}{(1-\alpha^\prime)\alpha^{\prime t}{\AM}^{\prime t}}\XM^\prime\WM^\prime\right\|^2_F,
\end{equation}
where hyperparameters $T^\prime$, $\alpha^\prime$, and weights $\WM^\prime$ can be configured or learned accordingly in the GNN model trained on $\G^\prime$. 
Due to the scale and complexity of the input graph structures, synthesizing the \({\AM}^{\prime}\) such that $\sum_{t=0}^{T^\prime}{(1-\alpha^\prime)\alpha^{\prime t}{\AM}^{\prime t}}$ is close to $\NCM\sum_{t=0}^{T}{(1-\alpha)\alpha^t\NAM^t}$ is highly challenging and usually engenders substantial structural information loss.
As a workaround, we turn to incorporate most features in $\HM$ into $\XM^\prime$ as follows: 
\begin{equation}\label{eq:X'=C^TZ}
\XM^{\prime} = \NCM^{\top} \ZM
\end{equation}
and compress $\NAM$ as $\AM^\prime$ using $\AM^\prime = \NCM^{\top}\NAM\NCM$.
Additionally, the synthetic label for each node $u_i\in \G^\prime$ is constructed by
\begin{equation*}
\YM^\prime_{i,j} =
\begin{cases}
1,\quad \text{if $j=\argmax{1\le \ell\in K}{\HM^\prime_{i,\ell}}$;}\\
0,\quad \text{otherwise}.
\end{cases}
\end{equation*}

\subsection{Class-Aware Attribute Refinement (\caar)}\label{sec:refinement}

\begin{table}[!t]
\centering
\caption{The homophilic ratio $\Omega(\G)$ and ICAD of $\XM$ , $\XM^\prime$, and $\XM' + \beta\cdot \boldsymbol{\Delta}$ on real-world graph datasets.}\label{tbl:heterophilic}
\vspace{-2ex}
\renewcommand{\arraystretch}{0.9}
\begin{small}
    \begin{tabular}{cccccc}
        \toprule
        {\bf Dataset} & {\em Cora} & {\em Citeseer} & {\em arXiv} & {\em Reddit} \\
        \midrule
        $\Omega(\G)$ &  $0.81$ & $0.74$ & $0.66$ & $0.78$ \\
        \midrule
        $\phi(\XM, \YM)$ & $0.95$ & $0.96$ & $0.99$ & $0.99$ \\
        $\phi(\XM^\prime, \YM^\prime)$ & $0.56$  & $0.51$ & $0.97$ & 0.82 \\
        $\phi(\XM' + \beta\cdot \boldsymbol{\Delta}, \YM^\prime)$ & $0.77$ & $0.56$ & $1.00$ & $0.96$ \\
        \bottomrule
    \end{tabular}
\end{small}
\end{table}

\stitle{Heterophilic Over-smoothing Issue}
Although the clustering approach in the preceding section can produce high-quality condensed graphs, it strongly relies on the homophily assumption. However, according to the homophily ratios of real graphs in Table~\ref{tbl:heterophilic}, even in {\em homophilic graphs}, we can observe a small moiety of {\em heterophilic links}, i.e., adjacent nodes with distinct ground-truth labels.
Recall that \(\ZM\) in $\XM^\prime$ computed via Eq.~\eqref{eq:X'=C^TZ} is to minimize \(\sum_{(v_i,v_j)\in \EDG}{\left\|\frac{\ZM_i}{\sqrt{d(v_i)}}-\frac{\ZM_j}{\sqrt{d(v_j)}}\right\|_F^2}\).
As an aftermath, the feature vectors $\ZM$ of {\em heterophilic nodes} (i.e., nodes in different classes) in $\G$ will be made closer (i.e., {\em over-smoothed}) due to the existence of heterophilic links, and after the clustering operation $\NCM^{\top}\ZM$, the attribute distances of synthetic heterophilic node pairs in $\XM^\prime$ are further reduced. A figurative example to illustrate such a {\em heterophilic over-smoothing} issue is presented in Fig.~\ref{fig:HL-example}, wherein the attribute vectors in $\XM^\prime$ of synthetic nodes $u_1$ (merged from $C_i=\{v_1,\ldots,v_5\}$) and $u_2$ (merged from $C_2=\{v_6,v_7,v_8\}$) are finally similar to each other, making them hard to distinguish by GNNs.
In sum, constructing $\XM^\prime$ by Eq.~\eqref{eq:X'=C^TZ} yields over-smoothed attribute vectors and compromises condensation quality.

To validate this issue, we conduct an empirical statistical analysis of $\XM^\prime$ on real datasets. Given attribute vectors $\XM$, we define the {\em inter-class attribute distance} (ICAD) of $\G$ as follows:
\begin{small}
\begin{equation*}
\phi(\XM, \YM) = \frac{1}{2\sum_{x\neq y}{\|\YM_{\cdot,x}\|_1\cdot\|\YM_{\cdot,y}\|_1}}\sum_{v_i,v_j\in \G,\& y_i\neq y_j}{\left\|\frac{\XM_i}{\|\XM_i\|_2}-\frac{\XM_j}{\|\XM_j\|_2}\right\|_2^2},
\end{equation*}
\end{small}
where $\|\YM_{\cdot,x}\|_1$ (resp. $\|\YM_{\cdot,y}\|_1$) signifies the size of the $x$-th (resp. $y$-th) class. $\phi(\XM, \YM)$ quantifies the averaged distance of attribute vectors $\XM$ of heterophilic node pairs. In the same vein, we can define the ICAD of the condensed graph $\G^\prime$ with the synthetic attribute vectors $\XM^\prime$ and label matrix $\YM^\prime$.
As reported in Table~\ref{tbl:heterophilic}, we can observe a decrease in the ICAD on each dataset, i.e., $\phi(\XM^\prime, \YM^\prime)<\phi(\XM, \YM)$, which corroborates the foregoing issue of heterophilic over-smoothing in $\XM^\prime$.

\begin{figure}[!t]
    \centering
    \includegraphics[width=\columnwidth]{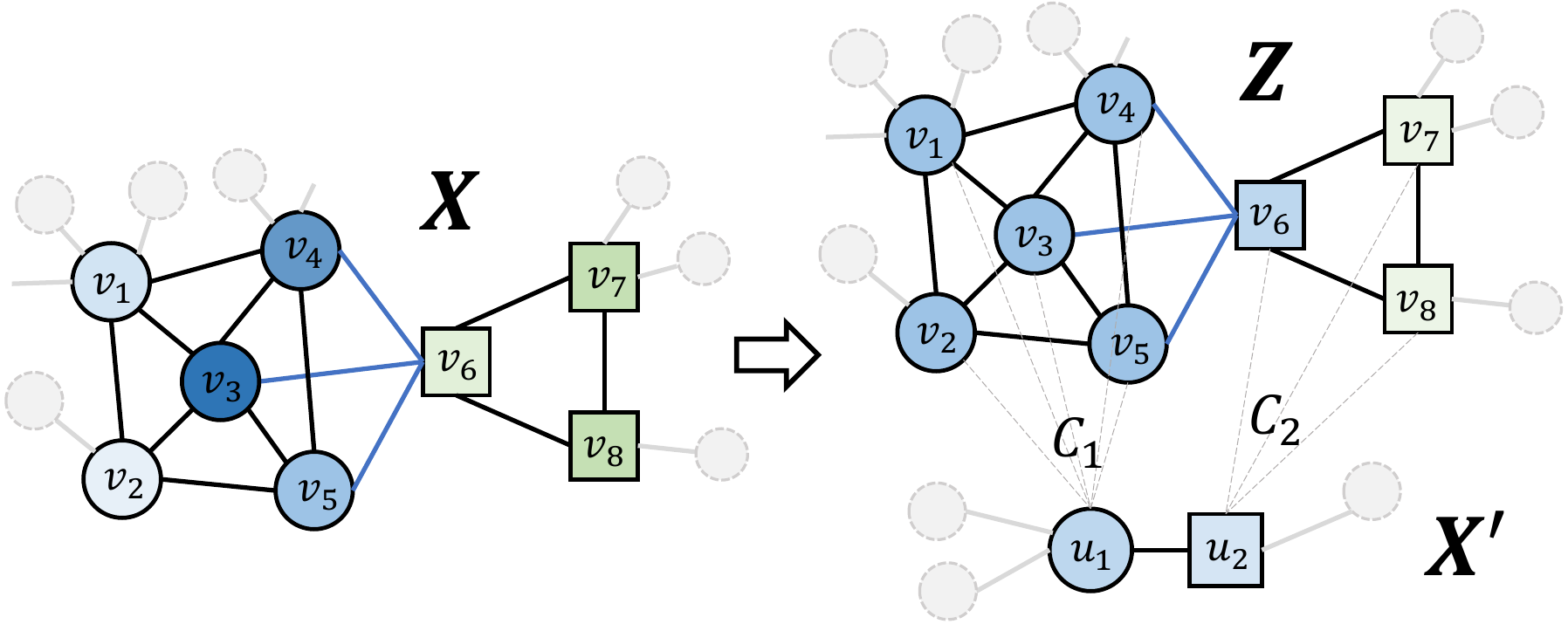}
    \vspace{-5ex}
    \caption{Example of the heterophilic over-smoothing issue}
    \label{fig:HL-example}
    \vspace{-2ex}
\end{figure}

\stitle{Class-Specific Graph Sampling and Representations}
To cope with the problematic $\XM^\prime$, we propose to learn a small attribute augmentation $\boldsymbol{\Delta}$ to inject class-relevant features into it to build a new attribute matrix for nodes and increase the ICAD, i.e.,
\begin{equation}\label{eq:X'=X'+X''}
\XM' + \beta\cdot \boldsymbol{\Delta},
\end{equation}
where $\beta$ stands for the weight of the augmentation. As displayed in Table~\ref{tbl:heterophilic}, the ICAD of the synthetic attributes of nodes with this refinement (i.e., Eq.~\eqref{eq:X'=X'+X''}) can be considerably increased, mitigating the heterophilic over-smoothing issue. 

Specifically, we sample $K$ graphs $\{\AM^{\circ(y)}\}_{y=1}^{K}$ from the original graph $\G$ for the respective classes in $\Y$, each of which aims to capture the key structural patterns pertinent to its respective class in $\G$. To achieve this goal, for each class $y\in \Y$, \caar first reweight each edge $(v_i,v_j)$ in $\G$ by the following weighting scheme:
\begin{equation}\label{eq:new-weight}
w(v_i,v_j) = \PM_{i,y}\cdot \PM_{j,y}\cdot r(v_i,v_j),
\end{equation}
where $\PM_{i,y}$ (defined in Eq.~\eqref{eq:logit_softmax}) is the predicted probability that the node $v_i$ belongs to the $y$-th class.
As such, $\PM_{i,y}\cdot \PM_{j,y}$ in Eq.~\eqref{eq:new-weight} leads to a low weight $w(v_i,v_j)$ for adjacent nodes $(v_i,v_j)$ that are unlikely to share the same class label, thereby eliminating the impact of heterophilic links.

To account for the importance of each edge $(v_i,v_j)$ in the entire graph topology, we further introduce the well-known {\em effective resistance}~\cite{lovasz1993random} $r(v_i,v_j)$ for $(v_i,v_j)$ in Eq.~\eqref{eq:new-weight}. Instead of computing the exact $r(v_i,v_j)$ over $\G$, \caar approximates it on the graph $\widetilde{\G}$ constructed from $\G$ with each edge $(v_i,v_j)$ reweighted by $\textsf{cos}(\HM_i,\HM_j)$. Note that this reweighting is also to lessen the negative effects of heterophilic links, and the approximation is due to the fact that the exact calculation of $r(v_i,v_j)$ for all edges is prohibitively expensive, particularly for large graphs.
More concretely, leveraging the Lemma 4.2 in ~\cite{lai2024efficient}, we can estimate $r(v_i,v_j)$ efficiently by the following equation: 
\begin{small}
\begin{equation*}
r(v_i,v_j) \approx \frac{1}{2}\left(\frac{1}{\widetilde{d}(v_i)}+\frac{1}{\widetilde{d}(v_j)}\right)\ \text{where $\widetilde{d}(v_i) = \sum_{v_\ell\in \N(v_i)}{\textsf{cos}(\HM_i, \HM_\ell)}$}.
\end{equation*}
\end{small}

\eat{
\begin{equation}\label{eq:d-weight}
d_w(v_i)=\sum_{v_j\in \N(v_i)}{w(e_{i,j})}
\end{equation}
\begin{lemma}[~\cite{lai2024efficient}]\label{lem:ER}
Let $\G_w=(\V,\EDG_w)$ be a weighted graph whose node degrees are defined as in Eq. \eqref{eq:d-weight}. The ER $r_w(e_{i,j})$ of each edge $e_{i,j}\in \EDG_w$ is bounded by
\begin{small}
\begin{equation*}
\frac{1}{2} \left( \frac{1}{d_w(v_i)}+\frac{1}{d_w(v_j)} \right) \le r_w(e_{i,j}) \le \frac{1}{1-\lambda_2} \left( \frac{1}{d_w(v_i)}+\frac{1}{d_w(v_j)} \right),
\end{equation*}
\end{small}
where $\lambda_2\le 1$ stands for the second largest eigenvalue of the normalized adjacency matrix of $\G_w$.
\begin{proof}
We defer the proof to ~\cite{lai2024efficient}.
\end{proof}
\end{lemma}
}

After obtaining the weight as in Eq.~\eqref{eq:new-weight} for each edge in $\G$, we sample the top-$M\cdot\rho$ ($\rho\in (0,1)$) ones from $\EDG$ to construct $\AM^{\circ(y)}$ instead of retaining all of them. This is to distill the crucial structural features (or homophilic structures) in $\G$ that relate to the class $y$, while purging noisy and disruptive connections (e.g., heterophilic links).

Based on the $K$ sampled class-specific graphs $\{\AM^{\circ(y)}\}_{y=1}^{K}$, we then generate their corresponding $K$ condensed graphs through
\begin{equation}\label{eq:mvg}
\AM'^{(y)} = \NCM^{\top}\AM^{\circ(y)}\NCM,
\end{equation}
which leads to $K$ classic-specific representation matrices for the $n$ synthetic nodes in $\G^\prime$:
\begin{equation}\label{eq:mvh}
\HM^{\prime (y)} = \sum_{t=0}^{T^\prime}{(1-\alpha)\alpha^t( \AM'^{(y)})^t}(\XM' + \beta\cdot \boldsymbol{\Delta})\WM^\prime.
\end{equation}
The task then is to learn the features $\boldsymbol{\Delta}$ that are commonly missing in the $K$ classic-specific representations $\{\HM^{\prime (y)}\}_{y=1}^{K}$ as augmentation attributes.

\stitle{Training Losses}
To learn $\boldsymbol{\Delta}$ and weights $\WM^\prime$, \caar transforms each $\HM'^{(y)}$ into label predictions $\PM^{(y)}=\textsf{Softmax}(\HM^{\prime (y)})$, and adopt the supervision losses $\mathcal{L}_{org},\mathcal{L}_{syn}$ and consistency loss function $\mathcal{L}_{cst}$ for training:
\begin{equation}\label{eq:final-loss}
\mathcal{L}_{org} + \gamma\cdot \mathcal{L}_{syn} + \lambda\cdot\mathcal{L}_{cst}
\end{equation}
where $\gamma$ and $\lambda$ are coefficients. $\mathcal{L}_{org}$ is defined as in Eq.~\eqref{eq:org-loss} and $\mathcal{L}_{syn}$ are implemented using the cross-entropy functions over $\G^\prime$:
\begin{small}
\begin{equation}\label{eq:suploss}
\begin{gathered}
\mathcal{L}_{syn} = -\frac{1}{n}\sum^{n}_{i=1} \sum^{K}_{y=1} \sum^{K}_{k=1} \YM^\prime_{i,k}\cdot\textsf{log}(\PM^{\prime (y)}_{i,k}),
\end{gathered}
\end{equation}
\end{small}
where $\mathcal{L}_{org}$ measures the difference between the predicted probability distribution $\PM$ and the actual distribution $\YM$ of ground-truth classes on $\G$, while $\mathcal{L}_{syn}$ calculates that on $\G^\prime$ with the consideration of the $K$ label predictions $\{\PM^{\prime (y)}\}_{y=1}^{K}$.

Additionally, we include the following consistency loss function $\mathcal{L}_{cst}$ \cite{feng2020graph} as a regularization to encourage an agreement among all class-specific predictions for learning a robust $\boldsymbol{\Delta}$. 
\begin{small}
\begin{equation}\label{eq:consistency_loss}
\mathcal{L}_{cst} = \frac{1}{n\cdot K} \sum^{n}_{i=1} \sum^{K}_{y=1}\|\PM^{(y)}_{i}-\overline{\PM}_{i}\|^2_2\ \text{where $\overline{\PM} = \frac{1}{K} \sum^{K}_{y=1} {\PM}^{(y)}$}.
\end{equation}
\end{small}


\eat{
\begin{equation}
\mathcal{L}_{sup} =  -\frac{1}{N}\sum^{N}_{i=1} \sum^{C}_{c=1} \YM_{i,c}\textsf{log}(\PM_{i,c})-\frac{\gamma}{N'}\sum^{N'}_{i=1} \sum^{C}_{c=1} \YM'_{i,c}\textsf{log}({\PM'}_{i,c})
\label{eq:suploss}
\end{equation}

The second term is the consistency loss \cite{feng2020graph} for multi-view condensed graph prediction. This loss is important in multi-view learning as it encourages agreement among predictions from different views by averaging them across all views.  Acting as a regularization technique, consistency loss improves the model's generalization ability and mitigates overfitting. Ultimately, it enables the model to learn more coherent and robust representations, leading to better performance in multi-view learning tasks.
\begin{equation}
\hat{\PM'}_{i} = \frac{1}{C} \sum^{C}_{k=1} {\PM'}_{i}^{(k)}
\label{eq:avg_pred}
\end{equation}
\begin{equation}
\mathcal{L}_{cst} = \frac{1}{N'C} \sum^{N'}_{i=1} \sum^{C}_{k=1} \|{\PM'}_{i}^{(k)}-\hat{\PM}_{i}\|^{2}_{2}
\label{eq:consistency_loss}
\end{equation}
}


\begin{algorithm}[h]
\caption{Clustering-based Graph Data Distillation}\label{alg:clusterGDD}
\KwIn{The original graph $\G$ and label matrix $\YM$}
\KwOut{$\XM'$, $\AM'$, and $\YM'$}
Generate $\ZM=\sum_{t=0}^{T}{(1-\alpha)\alpha^t\NAM^t}\XM$ \;
\For{$i = 1$ \KwTo $E_1$}{
    Predict $\YM$ via $\HM = \WM\ZM$\;
    Compute $\mathcal{L}_{org}$ by Eq. ~\eqref{eq:org-loss};
}
Generate $\NCM$ by $K$-Means on $\HM$ with $E_2$ iterations\;
Get $\XM^{\prime} = \NCM^{\top}\ZM$, $\AM' = \NCM^{\top}\NAM\NCM $ 
\;
Generate $\YM^{\prime}$ by taking the argmax of $\NCM^{\top}\HM$\;
Initialize $\boldsymbol{\Delta}$, update $\XM'$ by $\XM' = \XM' + \beta \cdot \boldsymbol{\Delta} $ \; 
Generate $\{ \AM^{\circ(y)}\}_{y=1}^{K}$ according to Eq. ~\eqref{eq:new-weight} \;
Get $\{ \AM'^{(y)}\}_{y=1}^{K}$ and $\{ \HM'^{(y)}\}_{y=1}^{K}$ by Eq. ~\eqref{eq:mvg} and~\eqref{eq:mvh}\;
\For{$i = 1$ \KwTo $E_3$}{
    Predict $\YM$ via $\HM\WM'$ \;
    Predict $\YM'$ via $\HM^{'(y)}\WM'$\;
    Compute $\mathcal{L}_{org}$, $\mathcal{L}_{syn}$ and $\mathcal{L}_{cst}$ as Eq.~\eqref{eq:org-loss},~\eqref{eq:suploss},~\eqref{eq:consistency_loss} \;
}
\Return $\XM'$, $\AM'$, $\YM'$  \label{line:return}
\end{algorithm}

\section{Pseudo-code and Time Complexity}
\label{sec:pseudo}
This section presents the algorithm of \algo{} as detailed in Algorithm \ref{alg:clusterGDD}. Consider an original graph with $M$ edges, $N$ nodes, 
$d$-dimensional attributes, and $K$ classes. Let $n$ denote the clustering number (i.e., synthetic node count), $E_1,E_2,E_3$ denote epochs for pretraining, clustering, and refinement, $T,T'$ denote propagation iterations in pretraining and refinement. The computational complexity of our framework includes: propagation ($\mathcal{O}\big(MdT\big)$), mini-batch pretraining ($\mathcal{O}\big(E_1NdK\big)$), K-means clustering ($\mathcal{O}\big(E_2Nnd\big)$)\footnote{For large-scale datasets, the mini-batch K-means algorithm is employed with a computational complexity of $\mathcal{O}\big(E_2bnd\big)$, where $b$ is the batch size.}, generating $\XM'$, $\AM'$, and $\YM'$ ($\mathcal{O}\big(Nd + M + NK\big)$), multi-view graphs generation ($\mathcal{O}\big(K^2M+KMlogM+KM\big)$), propagation on condensed graphs ($\mathcal{O}\big(Kn^2dT'\big)$), and refinement training ($\mathcal{O}\big(E_3NKd + E_3K^2nd\big)$). Assuming $E_1,E_2,E_3=\mathcal{O}\big(E\big)$, $T,T'=\mathcal{O}\big(1\big)$, $K=O\big(1\big)$, the time complexity simplifies to $\mathcal{O}\big( ENnd + n^2d + Md + M\log M \big)$.  In fact, our method has a much lower time complexity compared to the baselines, which is analyzed in Appendix~\ref{sec:appendix-time-com}. This is the reason why our method has a very low time cost empirically as shown in the following experiments.


\begin{table}[!t]
\centering
\renewcommand{\arraystretch}{1.0}
\caption{Statistics of Datasets.}\label{tbl:exp-data}
\vspace{-2mm}
\begin{tabular}{l|r|r|r|c}
	\hline
	{\bf Dataset} & \multicolumn{1}{c|}{\bf $N$ } & \multicolumn{1}{c|}{\bf $M$ } & \multicolumn{1}{c|}{\bf $d$ } & \multicolumn{1}{c}{\bf $K$} \\
	\hline
    {\em Cora} & $2,708$ &  $5,429$ & $1,433$ & $7$ \\
    {\em Citeseer} & $3,327$ & $4,732$ & $3,703$ & $6$ \\
    {\em Flickr} & $89,250$ & $899,756$ & $128$ & $7$  \\
    {\em arXiv} & $169,343$ & $1,166,243$ & $128$ & $40$ \\
    {\em Reddit} & $232,965$ & $57,307,946$ & $602$ & $41$ \\
    \hline
\end{tabular}%
\vspace{-2ex}
\end{table}

\section{Experiments}

In this section, we experimentally evaluate our proposed GDD method \algo{} against nine competitors over five real graphs in terms of node classification performance and efficiency. For reproducibility, the source code and datasets are available at \url{https://github.com/HKBU-LAGAS/ClustGDD}. 



\begin{table*}[htbp]
\renewcommand{\arraystretch}{0.9}
  \centering
  \caption{Node classification performance (mean accuracy ($\%$) , $\pm$ standard deviation) and average dataset synthesis time (second) of \algo{} and baselines. (best is highlighted in blue and runner-up in light-blue). The results of methods with $*$ are taken from prior works.}
  \label{tab: effectiveness}
  \vspace{-2ex}
  \resizebox{\linewidth}{!}{
\addtolength{\tabcolsep}{-0.3em}
    \begin{tabular}{lcccccccccccccc}
    \toprule
    \multirow{2}[4]{*}{\textbf{Dataset}} & \multirow{2}[4]{*}{\textbf{Ratio}} & \multicolumn{3}{c}{\textbf{Coreset Methods}} & \multicolumn{8}{c}{\textbf{GDD Methods}} & \multirow{2}[4]{*}{\makecell{\textbf{Whole} \\ \textbf{Dataset}}} \\
    
    \cmidrule(lr){3-5} \cmidrule(lr){6-13} &       & \makecell{\texttt{Random} } & \makecell{\texttt{Herding} } & \makecell{\texttt{K-Center}} & \makecell{\texttt{Gcond}} & \makecell{\texttt{SFGC}}  & \makecell{\texttt{SGDD}*} &\makecell{\texttt{GCSR}} & \makecell{\texttt{GC-SNTK}*} & \makecell{\texttt{GDEM}}& \makecell{\algo{}-$\XM$} & \makecell{\algo{}}& \\
    \midrule
    \multirow{3}[2]{*}{{\em Cora}} & $1.30\%$ & $62.3_{\pm 1.1}$ & $68.0_{\pm0.6}$ & $62.5_{\pm0.7}$ & $80.6_{\pm0.8}$ &  $79.9_{\pm0.5}$ & $79.1_{\pm1.3}$ & $79.7_{\pm1.4}$& \cellcolor{DarkBlue}{$81.7_{\pm0.7}$} & $78.7_{\pm3.1}$ & \cellcolor{Blue}{$81.0_{\pm0.4}$}  & $80.8_{\pm0.3}$ & \multirow{3}[2]{*}{$81.1_{\pm0.4}$} \\
          & $2.60\%$ & $73.1_{\pm0.9}$  & $73.4_{\pm0.4}$ & $73.3_{\pm0.5}$ & $81.1_{\pm0.5}$ &  $80.7_{\pm0.3}$ & $79.0_{\pm1.9}$ &$81.1_{\pm0.4}$ & $81.5_{\pm0.7}$ & $77.7_{\pm1.0}$ & \cellcolor{Blue}{$81.6_{\pm0.6}$} & \cellcolor{DarkBlue}{$\mathbf{81.7_{\pm0.5}}$} &  \\  
          & $5.20\%$ & $76.5_{\pm0.4}$  & $76.6_{\pm0.4}$ &  $77.0_{\pm0.4}$ & $80.7_{\pm0.6}$ & $80.3_{\pm0.4}$ & $80.2_{\pm0.8}$ &$80.8_{\pm0.7}$ & $81.3_{\pm0.2}$ &  $80.3_{\pm1.0}$ & \cellcolor{Blue}{$82.4_{\pm0.5}$} & \cellcolor{DarkBlue}{{$83.0_{\pm0.3}$}} &  \\
    \midrule
    \multirow{3}[2]{*}{{\em Citeseer}} & $0.90\%$ & $49.7_{\pm0.7}$ & $55.8_{\pm0.6}$ & $48.8_{\pm1.1}$ & $72.3_{\pm0.5}$ & $70.5_{\pm0.4}$ & $71.5_{\pm0.9}$ & $71.1_{\pm0.9}$ & $64.8_{\pm0.7}$ & $72.9_{\pm0.7}$& \cellcolor{DarkBlue}{$73.5_{\pm0.2}$}  & \cellcolor{DarkBlue}{$73.5_{\pm0.2}$} & \multirow{3}[2]{*}{$71.8_{\pm0.3}$} \\
          & $1.80\%$ & $57.7_{\pm0.7}$ & $68.2_{\pm0.4}$ & $65.2_{\pm0.6}$ & \cellcolor{Blue}{$72.5_{\pm0.6}$} & $71.8_{\pm0.1}$ & $71.2_{\pm0.7}$ & $71.4_{\pm0.5}$ & $65.9_{\pm0.2}$ & \cellcolor{DarkBlue}{$74.4_{\pm0.3}$} & \cellcolor{DarkBlue}{$74.4_{\pm0.2}$} & \cellcolor{DarkBlue}{$74.4_{\pm0.2}$} &  \\
          & $3.60\%$ & $69.0_{\pm0.5}$ & $69.0_{\pm0.8}$ & $69.5_{\pm0.5}$ & $70.9_{\pm0.5}$ &  $71.1_{\pm0.5}$ & $70.9_{\pm1.2}$ & $73.0_{\pm0.6}$ &$66.3_{\pm0.5}$&$74.1_{\pm1.1}$& \cellcolor{Blue}{$74.5_{\pm0.5}$} & \cellcolor{DarkBlue}{$74.6_{\pm0.3}$} &  \\
    \midrule
    \multirow{3}[2]{*}{{\em arXiv}} & $0.05\%$ & $43.8_{\pm1.5}$ & $51.8_{\pm1.2}$ & $48.0_{\pm1.4}$ & $56.4_{\pm2.5}$ & $63.3_{\pm0.7}$ &  $59.6_{\pm0.5}$& $60.7_{\pm0.7}$ &  \cellcolor{DarkBlue}{$64.2_{\pm0.2}$} & $61.8_{\pm1.8}$ & \cellcolor{Blue}{$63.7_{\pm0.4}$}  & \cellcolor{Blue}{$63.7_{\pm0.1}$} & \multirow{3}[2]{*}{$71.1_{\pm0.2}$} \\
          & $0.25\%$ & $57.0_{\pm0.7}$  & $60.5_{\pm0.7}$ & $56.4_{\pm0.8}$ & $60.7_{\pm2.2}$ &  $67.3_{\pm0.7}$ &  $61.7_{\pm0.3}$ & $64.4_{\pm0.3}$ &  $65.1_{\pm0.8}$ & $61.8_{\pm0.4}$ & \cellcolor{Blue}{$67.4_{\pm0.3}$}  & \cellcolor{DarkBlue}{$67.5_{\pm0.1}$} & \\
          & $0.50\%$ & $58.8_{\pm0.8}$  & $60.5_{\pm0.4}$ & $60.1_{\pm0.5}$ & $60.9_{\pm2.0}$ & \cellcolor{DarkBlue}{$67.9_{\pm0.2}$} &  $58.7_{\pm0.6}$ & $66.4_{\pm0.3}$ &  $65.4_{\pm0.5}$ & $63.3_{\pm0.5}$ & $67.6_{\pm0.1}$ & \cellcolor{Blue}{$67.7_{\pm0.2}$} &  \\
    \midrule
    \multirow{3}[2]{*}{{\em Flickr}} &$0.10\%$ & $42.0_{\pm0.3}$ & $40.5_{\pm0.7}$ & $41.3_{\pm0.7}$ & $46.3_{\pm0.6}$ & $46.7_{\pm0.2}$ & $46.1_{\pm0.3}$ & $46.6_{\pm0.2}$ & $46.7_{\pm0.1}$ & $45.7_{\pm0.6}$ & \cellcolor{Blue}{$46.8_{\pm0.1}$} & \cellcolor{DarkBlue}{$46.9_{\pm0.1}$} & \multirow{3}[2]{*}{$46.9_{\pm0.1}$} \\
          & $0.50\% $& $42.9_{\pm0.5}$ & $42.7_{\pm0.4}$ & $43.6_{\pm0.3}$ &  \cellcolor{Blue}{$46.9_{\pm0.1}$} &  \cellcolor{DarkBlue}{${47.2_{\pm0.1}}$} & $45.9_{\pm0.4}$ & $46.7_{\pm0.1}$ & $46.8_{\pm0.1}$ &  $46.1_{\pm0.6}$ & $46.8_{\pm0.1}$ & \cellcolor{Blue}{$46.9_{\pm0.1}$}  & \\
          & $1.00\% $& $42.8_{\pm0.6}$ & $44.9_{\pm0.2}$ & $43.9_{\pm0.5}$ & \cellcolor{DarkBlue}{$47.1_{\pm0.3}$} & \cellcolor{Blue}{${46.9_{\pm0.2}}$} & $46.4_{\pm0.2}$ & \cellcolor{Blue}{$46.9_{\pm0.3}$} & $46.5_{\pm0.2}$ & $46.4_{\pm0.4}$ & \cellcolor{Blue}{$46.9_{\pm0.2}$} & \cellcolor{DarkBlue}{$47.1_{\pm0.1}$}&  \\
    \midrule
    \multirow{3}[2]{*}{{\em Reddit}} & $0.05\%$ & $44.8_{\pm1.4}$ & $52.2_{\pm1.3}$ & $48.5_{\pm0.7}$ & $87.2_{\pm0.9}$ & $78.8_{\pm1.3}$ & $84.2_{\pm0.7}$ &  \cellcolor{Blue}$90.5_{\pm0.3}$ & $74.3_{\pm0.5}$ & \cellcolor{DarkBlue}{$91.4_{\pm0.4}$} &  $89.8_{\pm0.1}$  & $89.8_{\pm0.1}$ & \multirow{3}[2]{*}{$94.1_{\pm0.0}$} \\ 
          & $0.10\%$ & $60.2_{\pm1.3}$ & $61.4_{\pm1.0}$ & $49.0_{\pm1.2}$ & $89.1_{\pm0.8}$ & $80.7_{\pm1.9}$ & $90.7_{\pm0.1}$ & \cellcolor{Blue}$91.2_{\pm0.4}$ &{$74.8_{\pm0.7}$}  & \cellcolor{DarkBlue}{$91.5_{\pm0.4}$} & \cellcolor{DarkBlue}{$91.5_{\pm0.1}$} & \cellcolor{DarkBlue}{$91.5_{\pm0.1}$} & \\
          & $0.50\%$ &$ 79.3_{\pm1.0} $& $82.1_{\pm0.2}$ & $70.3_{\pm0.7}$ & $90.1_{\pm0.7}$ & $87.1_{\pm0.3}$ & $-$ & $84.6_{\pm1.1}$ & $85.2_{\pm1.2}$  & $91.4_{\pm0.7}$ & \cellcolor{Blue}{$91.6_{\pm0.0}$} & \cellcolor{DarkBlue}{$91.7_{\pm0.0}$}  & \\
    \midrule
    \multicolumn{2}{c}{\textbf{Avg. Time}} & $92.4$ & $92.6$ & $92.0$ & $8,150.4$ & $104,999.0$ & $46,424.8$ & $4,350.4$ & $3,045.4$ & $1,316.4$ & \cellcolor{DarkBlue}{56.7} & \cellcolor{Blue}{$56.8$} &\\
    \bottomrule
    \end{tabular}
}
\end{table*}

\subsection{Experiment Settings}
\stitle{Datasets}
Table \ref{tbl:exp-data} lists the statistics of the five real graph datasets used in our experiments. {\em Cora}~\cite{sen2008collective}, {\em Citeseer}~\cite{sen2008collective}, and {\em arXiv}~\cite{hu2020open} are the classic datasets for transductive node classification, while {\em Flickr} \cite{zeng2019graphsaint} and {\em Reddit} \cite{hamilton2017inductive} are for inductive tasks. 


\stitle{Baselines and Settings}
We evaluate our \algo{} against three classic coreset methods: \texttt{Random} \cite{welling2009herding}, \texttt{Herding} \cite{welling2009herding} and \texttt{K-Center} \cite{sener2017active}, and six GDD baselines as categorized below:
\begin{itemize}[leftmargin=*]
\item {\em Gradient Matching}: \texttt{Gcond} \cite{jin2021graph}, \texttt{SGDD} \cite{yang2023does}, and \texttt{GCSR} \cite{liu2024graph};
\item {\em Trajectory Matching}: \texttt{SFGC} \cite{zheng2024structure};
\item {\em Eigenbasis Matching}: \texttt{GDEM} \cite{liugraph};
\item {\em Performance Matching}: \texttt{GC-SNTK} \cite{wang2024fast}
\end{itemize}
For a fair comparison, we follow the evaluation protocol and settings (e.g., condensation ratios, GNNs, etc.) commonly adopted in previous works~\cite{jin2021graph,liugraph} for all evaluated methods.
We reproduce the performance of \texttt{Gcond}, \texttt{SFGC}, \texttt{GCSR}, and \texttt{GDEM} using the source codes, hyperparameter settings, and synthetic datasets provided by the respective authors. As for the rest of the baselines, we directly use the results reported in  \cite{liu2024graph, wang2025efficient}.
For the interest of space, we defer the details of datasets, implementation, and hyper-parameter settings to Appendix~\ref{sec:more-data-imp-details}, and additional experiments on hyper-parameter analysis and visualization to Appendix~\ref{sec:add-exp}.

\subsection{Condensation Effectiveness Evaluation}
\eat{\renchi{refine the writing. too many paragraphs. need to explain Table \ref{tab: effectiveness} carefully: what datasets tested, what condensation ratios used, reported results are node classification accuracies, best and runner-up are highlighted in blue and darker shades indicate better performance, the results of GNNs on whole datasets are provided. Then list your observations. Explain our performance, e.g., consistently outperform others on different datasets and condensation ratios, give examples (the specific improvements on two examples datasets), comparable to those on the whole datasets. explain why our performance is so good, demonstrate the effectiveness of our proposed what techniques.}}
We conduct comprehensive comparison of node classification performance across various coreset and distillation methods. We train \texttt{GCN} on the synthetic graphs and record the test accuracy for evaluation. The ratio is defined as $n/N$ for transductive datasets, and it is defined as $n/N_{\mathcal{T}}$ for the inductive ones, where $N_{\mathcal{T}}$ is the number of training nodes. 

We can summarize several key points from the results shown in Table \ref{tab: effectiveness}. 
Firstly, Our \algo{} and \algo{}-$\XM$ obtain overall highest/second-highest accuracy on most datasets, and sometimes higher than the results of the whole datasets. For instance, on {\em Citeseer}, the average accuracy of \algo{} on three ratios is $74.2\%$, which is $2.4\%$ higher than original result. We suggest that the success of \algo{} can be primarily attributed to two key factors: during the clustering phase, \algo{} effectively integrates information from both the original attributes and the graph topology through GLS. Through pretraining and clustering, the FID between the condensed attributes and the original attributes is minimized. On the other hand, during the refinement phase, the condensed attributes are further improved by co-training on original and condensed graphs, reducing the Heterophilic Over-smoothing Issue.
Secondly, the  results of \algo{} are slightly better than \algo{}-$\XM$, indicating that the synthetic attributes capture most of the original graph's key information, with the synthetic topology offering supplementary advantages. We also find distillation methods consistently outperform traditional methods, demonstrating the overall effectiveness and robustness of distillation methods.

We present the average time costs of coreset and GDD methods across five datasets with condensation ratios of $2.6\%$, $1.8\%$, $0.05\%$, $0.10\%$, and $0.10\%$, respectively. Experimental findings reveal a notable trade-off: while traditional graph distillation methods achieve high node classification accuracy, their computational overhead remains prohibitive (ranging from $10^3$ to $10^4$ seconds). Conversely, coreset methods demonstrate remarkable efficiency (approximately 92 seconds) at the cost of diminished accuracy. Our proposed \algo{} method effectively bridges this gap, achieving state-of-the-art accuracy in merely 56.8 seconds. Detailed per-dataset performance metrics are available in Appendix \ref{sec:add-exp}.

\subsection{Cross-architecture Generalization}
\eat{\renchi{need to refer to specific data statistics as examples for illustration.}}
To verify the generalization of our distilled graph, we use different GNN architectures to train on synthetic graphs and test their performance on the original graph, and the results are shown in Table \ref{tab:gnn_performance}. 
\algo{} achieves the overall highest average, and the lowest std, showing its stability and effectiveness when datasets and GNNs architecture vary. For example, on the {\em arXiv} dataset, the average accuracy of our method is $2.8\%$ higher than that of the second-place, \texttt{GCSR}, and the std is reduced by $1.7\%$. This validates that \algo{} has the superior generalization ability across GDDs. This may be because the closed-form solution of GLS in the clustering and refinement is more stable than other specific GNNs. 

\subsection{Ablation Study}
\stitle{Varying Clustering Methods}
\eat{\renchi{update and refine the writing}}
We use different GNNs architectures to replace the closed-form solution of GLS as the backbone in the pretraining and refinement stages. In general, \texttt{SGC} and \texttt{APPNP} perform well on the {\em Cora} and {\em Citeseer}. \texttt{ChebyNet} shows good performance on the {\em arXiv} and {\em Reddit} datasets. \texttt{MLP} has the highest accuracy on the {\em Flickr} dataset. Our backbone consistently outperforms the others across all datasets, this may be attributed to several key factors. Firstly, it excels at capturing the global structure of graphs, allowing it to model long-range dependencies and overall connectivity patterns effectively, which is crucial in datasets where important features are spread across the entire graph. Secondly, it assigns weighted influence to nodes based on their importance, enabling the model to focus more on key nodes that drive the graph dynamics and aggregate their features more meaningfully. These combined strengths enable the personalized PageRank backbone to consistently deliver superior performance across diverse datasets.

\begin{table}[!t]
\centering
\caption{Generalization of GDD methods across GNNs.}
\label{tab:gnn_performance}
\vspace{-2ex}
\renewcommand{\arraystretch}{0.9}
\begin{small}
\resizebox{\linewidth}{!}{
\addtolength{\tabcolsep}{-0.3em}
\begin{tabular}{ccccccccccc}
\toprule
\multirow{2}{*}{\textbf{Dataset}} & \multirow{2}{*}{\textbf{Method}} & \multicolumn{5}{c}{\textbf{GNN Architectures}} & \multirow{2}{*}{\textbf{Avg.}} & \multirow{2}{*}{ \textbf{Std.}} \\
\cmidrule{3-7}
 &  & \texttt{GCN} & \texttt{SGC} & \texttt{APPNP} & \texttt{ChebyNet} & \texttt{BernNet} &   \\
\midrule
\multirow{4}{*}{{\em Cora}} 
& \texttt{Gcond} & $81.1$  & $81.0$  & $80.2$ & $60.8$ & $81.0$ & $76.8 $&${8.0}$     \\
& \texttt{SFGC} & $80.7$ & $80.7$ & $80.2$  & $81.8$ & $80.2$  & $80.7$&${0.6}$   \\
& \texttt{GCSR} & $81.1$ & $81.4$  & $81.4$ & $82.3$ & $82.0$ & $81.6$&${0.4}$  \\
& \texttt{GDEM} & $77.7$ & $77.8$ & $78.0$  & $77.3$  & $78.2$ & $77.8$&${0.3}$ \\
& \algo{} & $81.7$  & $82.1$ &  $81.7$  & $82.3$ & $82.3$ & $82.0$&${0.3}$ \\
\midrule
\multirow{4}{*}{{\em Citeseer}} & \texttt{Gcond} & $72.5$ & $71.7$  & $71.6$  & $33.5$  & $70.3$  & $63.9$&${15.2}$\\
& \texttt{SFGC} & $71.8$ & $71.8$ &  $71.3$ & $70.3$ & $66.6$ & $70.4$&${2.0}$ \\
& \texttt{GCSR} & $71.4$ & $71.6$  & $71.6$ & $66.0$ & $63.9$ & $68.9$&${3.3}$   \\
& \texttt{GDEM} & $74.4$ & $74.6$ & $74.7$ & $73.1$ & $72.1$ & $73.8$&${1.0}$   \\
& \algo{} & $74.4$ & $74.3$ & $74.6$ & $73.3$ & $73.1$ & $73.9$&${0.6}$  \\
\midrule
\multirow{4}{*}{{\em arXiv}} & \texttt{Gcond} & $56.4$ & $46.4$ & $58.0$  & $46.4$ & $59.1$   & $53.3$&${5.7}$\\
& \texttt{SFGC} & $63.3$ & $60.7$  & $55.7$  & $61.6$  & $60.2$  & $60.3$&${2.5}$ \\
& \texttt{GCSR} & $60.7$ & $55.7$ & $58.7$  & $61.4$ & $61.3$  & $59.6$&${2.2}$\\
& \texttt{GDEM} & $61.8$  & $60.4$ & $59.2$ & $60.5$  & $56.8$  & $59.7$&${1.7}$\\
& \algo{} & $63.7$ & $61.9$ & $62.4$ & $64.0$  & $63.3$  & $63.1$&${0.8}$ \\
\midrule
\multirow{4}{*}{{\em Flickr}} & \texttt{Gcond} & $46.3$ & $46.3$  & $44.9$    & $41.8$  & $38.1$ & $43.5$&${3.2}$ \\
 & \texttt{SFGC} & $46.7$ & $46.5$ & $46.7$ & $46.4$  & $45.4$  & $46.3$&${0.5}$\\
 & \texttt{GCSR} & $46.6$ & $46.5$ & $46.6$ & $46.2$ & $45.6$ & $46.3$&${0.4}$\\
 & \texttt{GDEM} & $45.7$  & $44.5$  & $44.4$  & $44.2$  & $40.1$ & $43.8$&${1.9}$ \\
 & \algo{} & $46.9$ & $46.9$  & $46.7$ & $46.5$ & $46.4$ & $46.7$&${0.2}$ \\
\midrule
\multirow{4}{*}{{\em Reddit}} & \texttt{Gcond} & $89.1$ & $89.9$  & $87.5$ & $21.1$ & $87.3$ & $75.0$&${27.0}$   \\
 & \texttt{SFGC} & $80.7$  & $62.4$  & $62.6$ & $78.5$ & $75.0$ & $71.8$&${7.8}$  \\
 & \texttt{GCSR} & $91.2$ & $89.3$ & $87.2$ & $70.4$ & $85.5$ & $84.7$&${7.4}$  \\
 & \texttt{GDEM} & $91.5$ & $91.4$ & $91.0$ & $85.8$ & $82.2$ & $88.4$ &${3.8}$ \\
 & \algo{} & $91.5$ & $90.7$ & $88.7$ & $87.5$ & $83.9$ & $88.5$&${2.7}$\\
\bottomrule
\end{tabular}
}
\vspace{-2ex}
\end{small}
\end{table}

\stitle{Attribute Refinement}
\eat{\renchi{update/remove the terms, e.g., TGAS, sparsification, multi-view}}
We then explore the impact of modules in \caar. Firstly, we remove the refinement (w/o \caar), using synthetic attributes and topology from the clustering stage to train GNNs, it leads to a decline in GNN performance across all datasets, especially on datasets with a higher compression ratio and class imbalance, such as {\em Reddit} and {\em arXiv}.
Secondly, we replace the multiple representations $\{\HM^{\prime (y)}\}_{y=1}^{K}$ created via class-specific sampled graphs in \caar by a single $\HM^{\prime}$ (w/o sampling). Due to the lack of information from different views, the performance of GNN declines across all datasets, especially on large graphs such as {\em Reddit} and {\em arXiv}. 
Then, we remove the consistency loss (w/o $\mathcal{L}_{cst}$). It indeed affects the quality of the distilled graph. It's because the consistency of predictions enhances the robustness of the refinement.

\begin{table}[!t]
\centering
\caption{Ablation study results.}
\vspace{-2ex}
\resizebox{\linewidth}{!}{
\addtolength{\tabcolsep}{-0.2em}
\begin{tabular}{@{}llccccc@{}}
\toprule
& \textbf{Modules} & {\em Cora} & {\em Citeseer} & {\em arXiv} & {\em Flickr} & {\em Reddit} \\ \midrule 
  & w/ \texttt{MLP} & $77.8_{\pm0.7}$  & $72.3_{\pm0.3}$ & $58.7_{\pm0.4}$  & $45.5_{\pm0.4}$  & $81.6_{\pm0.3}$  \\
  & w/ \texttt{GCN}  & $78.9_{\pm 0.5}$ & $70.9_{\pm 0.7}$  & $60.9_{\pm 0.3}$ & $ 43.6_{\pm 1.1}$  & $87.7_{\pm 0.4}$  \\ 
& w/ \texttt{SGC}  & $81.7_{\pm 0.5}$  & $72.3_{\pm 0.3}$  & $57.5_{\pm 0.7}$  & $45.1_{\pm 0.7}$ & $77.7_{\pm 0.3}$ \\ 
& w/ \texttt{APPNP}  & $80.6_{\pm 0.2}$  & $72.2_{\pm 0.3}$  & $59.4_{\pm 0.6}$ & $41.4_{\pm 1.2}$ & $86.3_{\pm 0.1}$  \\ 
& w/ \texttt{BernNet}  & $79.2_{\pm 0.4}$  & $72.0_{\pm 0.1}$ & $58.7_{\pm 0.5}$  & $45.0_{\pm 0.2}$ & $84.6_{\pm 0.4}$ \\ 
& w/ \texttt{ChebyNet}  & $79.5_{\pm 0.4}$ & $72.1_{\pm 0.6}$ & $61.4_{\pm 0.2}$  & $45.8_{\pm 0.1}$  & $87.2_{\pm 0.4}$ \\
  \midrule
  & w/o \caar  & $81.3_{\pm 0.4}$ & $74.0_{\pm 0.3}$ & $61.9_{\pm 0.4}$ & $46.3_{\pm 0.2}$  & $87.6_{\pm 0.2}$ \\ 
  & w/o sampling & $81.5_{\pm0.2}$  & $74.1_{\pm0.2}$  & $62.6_{\pm0.3}$ & $46.8_{\pm0.1}$ & $90.4_{\pm0.1}$  \\ 
  & w/o $\mathcal{L}_{cst}$ & $81.6_{\pm0.5}$ & $74.3_{\pm0.2}$ & $63.5_{\pm0.3}$ & $46.7_{\pm0.1}$   & $91.4_{\pm0.1}$ \\
\midrule
& \algo{} & $81.7_{\pm 0.5}$ & $74.4_{\pm 0.2}$ & $63.7_{\pm 0.1}$ & $46.9_{\pm 0.1}$   & $91.5_{\pm 0.1}$\\
  \bottomrule
\end{tabular}
}
\vspace{-2ex}
\end{table}

\subsection{Hyper-parameter Analysis}

Next, we conduct hyper-parameter analysis in our method on {\em arXiv} and {\em Reddit}, evaluating the impact of various hyperparameters on the performance of our machine learning model in Fig.~\ref{fig:hyperparameter}.  First, we explore the propagation number $T$ in of the pretraining stage, which is critical for controlling the degree of information integration between nodes at different distances in the graph. The model’s performance initially increases with $T$ and then decreases on both {\em arXiv} and {\em Reddit}. This means that aggregating information from nearby neighbors is beneficial, while excessive message passing may lead to over-smoothing, resulting in a decline in performance. Then we tune another hyper-parameter $\alpha$,  $\alpha$ represents the weight between exploring neighbors and back to a starting node. A larger $\alpha$ means the model is more likely to explore, leading to results that are influenced by the overall graph structure, which is important for the task on relatively larger graphs. We find that when generating multi-view representations, a certain degree of sparsity $\rho$ can help the model achieve optimal performance, and our method is relatively robust to varying levels of sparsity. We find that the strength of the refinement on attributes, represented by $\beta$, significantly affect the quality of the distilled graph attributes, thereby influencing the model’s performance. In refinement, we find that adjusting the coefficient 
$\gamma$ for the supervision loss has a greater impact than adjusting the coefficient $\lambda$ for the consistency loss. This suggests that the supervision loss on multi‐view graphs may be more important.

 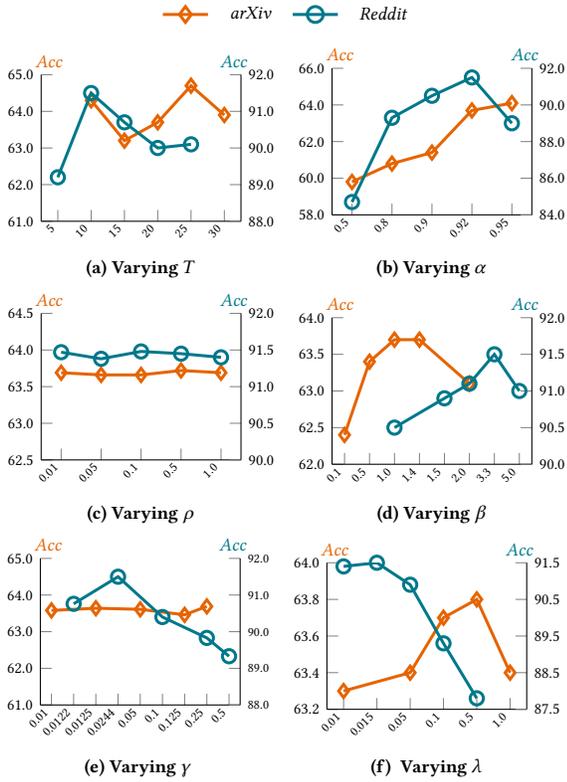
\begin{figure}[!t]
\centering
\begin{small}
\begin{tikzpicture}
    \begin{customlegend}[legend columns=2,
        legend entries={{\em arXiv},{\em Reddit}},
        legend style={at={(0.45,1.35)},anchor=north,draw=none,font=\footnotesize,column sep=0.2cm}]
    \addlegendimage{line width=0.4mm,mark size=3pt,mark=diamond, color=myorange-new}
    \addlegendimage{line width=0.4mm,mark size=3pt,mark=o, color=mygreen-new}
    \end{customlegend}
\end{tikzpicture}
\\[-\lineskip]
\vspace{-2ex}
\subfloat[Varying $T$]{
    \begin{tikzpicture}[scale=1,every mark/.append style={mark size=2.5pt}]
        \begin{axis}[
            height=\columnwidth/2.4,
            width=\columnwidth/2.0,
            ylabel={\it \textcolor{myorange-new}{Acc}},
            xmin=0.5, xmax=6.5,
            ymin=61, ymax=65,
            xtick={1,2,3,4,5,6},
            ytick={61.0,62.0,63.0,64.0,65.0},
            xticklabel style = {font=\tiny, rotate=45, anchor=east},
            yticklabel style = {font=\scriptsize, rotate=0, anchor=east},
            xticklabels={5,10,15,20,25,30},
            yticklabels={61.0,62.0,63.0,64.0,65.0},
            every axis y label/.style={font=\footnotesize,at={(current axis.north west)},right=1mm,above=0mm},
            axis y line*=left,
            axis x line*=bottom,
            legend style={fill=none,font=\small,at={(0.02,0.99)},anchor=north west,draw=none},
        ]
        \addplot[line width=0.4mm, mark=diamond, color=myorange-new] 
            plot coordinates {
                (2, 64.3)
                (3, 63.2)
                (4, 63.7)
                (5, 64.7)
                (6, 63.9)
        };
        \end{axis}
        \begin{axis}[
            height=\columnwidth/2.4,
            width=\columnwidth/2.0,
            ylabel={\it \textcolor{mygreen-new}{Acc}},
            xmin=0.5, xmax=6.5,
            ymin=88, ymax=92,
            xtick={1,2,3,4,5,6},
            ytick={88.0,89.0,90.0,91.0,92.0},
             xticklabel style = {font=\tiny, anchor=east},
            yticklabel style = {font=\scriptsize, rotate=0, anchor=west},
            xticklabels={5,10,15,20,25,30},
            yticklabels={88.0,89.0,90.0,91.0,92.0},
            every axis y label/.style={font=\footnotesize,at={(current axis.north east)},left=1mm,above=0mm},
            axis y line*=right,
            axis x line=none,
            legend style={fill=none,font=\small,at={(0.98,0.99)},anchor=north east,draw=none},
        ]
        \addplot[line width=0.4mm, mark=o, color=mygreen-new]  
            plot coordinates {
                (1, 89.2)
                (2, 91.5)
                (3, 90.7)
                (4, 90.0)
                (5, 90.1)
        };
        \end{axis}
    \end{tikzpicture}
    }
\subfloat[Varying $\alpha$]{
    \begin{tikzpicture}[scale=1,every mark/.append style={mark size=2.5pt}]
        \begin{axis}[
            height=\columnwidth/2.4,
            width=\columnwidth/2.0,
            ylabel={\it \textcolor{myorange-new}{Acc}},
            xmin=0.5, xmax=5.5,
            ymin=58, ymax=66,
            xtick={1,2,3,4,5},
            ytick={58.0,60.0,62.0,64.0,66.0},
            xticklabel style = {font=\tiny, rotate=45, anchor=east},
            yticklabel style = {font=\scriptsize, rotate=0, anchor=east},
            xticklabels={0.5,0.8, 0.9,0.92, 0.95},
            yticklabels={58.0,60.0,62.0,64.0,66.0},
            every axis y label/.style={font=\footnotesize,at={(current axis.north west)},right=1mm,above=0mm},
            axis y line*=left,
            axis x line*=bottom,
            legend style={fill=none,font=\small,at={(0.02,0.99)},anchor=north west,draw=none},
        ]
        \addplot[line width=0.4mm, mark=diamond, color=myorange-new] 
            plot coordinates {
                (1, 59.8)
                (2, 60.8)
                (3, 61.4)
                (4, 63.7)
                (5, 64.1)
        };
        \end{axis}
        \begin{axis}[
            height=\columnwidth/2.4,
            width=\columnwidth/2.0,
            ylabel={\it \textcolor{mygreen-new}{Acc}},
            xmin=0.5, xmax=5.5,
            ymin=84, ymax=92,
            xtick={1,2,3,4,5},
            ytick={84,86,88,90,92},
             xticklabel style = {font=\tiny, rotate=45, anchor=east},
              yticklabel style = {font=\scriptsize, rotate=0, anchor=west},
            xticklabels={0.5,0.8, 0.9,0.92, 0.95},
            yticklabels={84.0,86.0,88.0,90.0,92.0},
            every axis y label/.style={font=\footnotesize,at={(current axis.north east)},left=1mm,above=0mm},
            axis y line*=right,
            axis x line=none,
            legend style={fill=none,font=\small,at={(0.98,0.99)},anchor=north east,draw=none},
        ]
        \addplot[line width=0.4mm, mark=o, color=mygreen-new]  
            plot coordinates {
                (1, 84.7)
                (2, 89.3)
                (3, 90.5)
                (4, 91.5)
                (5, 89.0)
        };
        \end{axis}
    \end{tikzpicture}
}
\vspace{-3mm}
\subfloat[Varying $\rho$]{
    \begin{tikzpicture}[scale=1,every mark/.append style={mark size=2.5pt}]
        \begin{axis}[
            height=\columnwidth/2.4,
            width=\columnwidth/2.0,
            ylabel={\it \textcolor{myorange-new}{Acc}},
            xmin=0.5, xmax=5.5,
            ymin=62.5, ymax=64.5,
            xtick={1,2,3,4,5},
            ytick={62.5,63.0,63.5,64.0,64.5},
            xticklabel style = {font=\tiny, rotate=45, anchor=east},
            yticklabel style = {font=\scriptsize, rotate=0, anchor=east},
            xticklabels={0.01,0.05, 0.1,0.5, 1.0},
            yticklabels={62.5,63.0,63.5,64.0,64.5},
            every axis y label/.style={font=\footnotesize,at={(current axis.north west)},right=1mm,above=0mm},
            axis y line*=left,
            axis x line*=bottom,
            legend style={fill=none,font=\small,at={(0.02,0.99)},anchor=north west,draw=none},
        ]
        \addplot[line width=0.4mm, mark=diamond, color=myorange-new] 
            plot coordinates {
                (1, 63.69)
                (2, 63.66)
                (3, 63.66)
                (4, 63.72)
                (5, 63.69)
        };
        \end{axis}
        \begin{axis}[
            height=\columnwidth/2.4,
            width=\columnwidth/2.0,
            ylabel={\it \textcolor{mygreen-new}{Acc}},
            xmin=0.5, xmax=5.5,
            ymin=90.0, ymax=92,
            xtick={1,2,3,4,5},
            ytick={90.0,90.5, 91.0,91.5,92.0},
             xticklabel style = {font=\tiny, rotate=45, anchor=east},
              yticklabel style = {font=\scriptsize, rotate=0, anchor=west},
            xticklabels={0.5,0.8, 0.9,0.92, 0.95},
            yticklabels={90.0,90.5, 91.0,91.5,92.0},
            every axis y label/.style={font=\footnotesize,at={(current axis.north east)},left=1mm,above=0mm},
            axis y line*=right,
            axis x line=none,
            legend style={fill=none,font=\small,at={(0.98,0.99)},anchor=north east,draw=none},
        ]
        \addplot[line width=0.4mm, mark=o, color=mygreen-new]  
            plot coordinates {
                (1, 91.47)
                (2, 91.38)
                (3, 91.48)
                (4, 91.45)
                (5, 91.40)
        };
        \end{axis}
    \end{tikzpicture}
    }
\subfloat[Varying $\beta$]{
    \begin{tikzpicture}[scale=1,every mark/.append style={mark size=2.5pt}]
        \begin{axis}[
            height=\columnwidth/2.4,
            width=\columnwidth/2.0,
            ylabel={\it \textcolor{myorange-new}{Acc}},
            xmin=0.5, xmax=8.5,
            ymin=62.0, ymax=64.0,
            xtick={1,2,3,4,5,6,7,8},
            ytick={62.0, 62.5 ,63.0, 63.5, 64.0},
            xticklabel style = {font=\tiny, rotate=45, anchor=east},
            yticklabel style = {font=\scriptsize, rotate=0, anchor=east},
            xticklabels={0.1, 0.5,1.0, 1.4, 1.5, 2.0,3.3, 5.0},
            yticklabels={62.0, 62.5 ,63.0, 63.5, 64.0},
            every axis y label/.style={font=\footnotesize,at={(current axis.north west)},right=1mm,above=0mm},
            axis y line*=left,
            axis x line*=bottom,
            legend style={fill=none,font=\small,at={(0.02,0.99)},anchor=north west,draw=none},
        ]
        \addplot[line width=0.4mm, mark=diamond, color=myorange-new] 
            plot coordinates {
                (1, 62.4)
                (2, 63.4)
                (3, 63.7)
                (4, 63.7)
                (6, 63.1)
        };
        \end{axis}
        \begin{axis}[
            height=\columnwidth/2.4,
            width=\columnwidth/2.0,
            ylabel={\it \textcolor{mygreen-new}{Acc}},
            xmin=0.5, xmax=8.5,
            ymin=90.0, ymax=92.0,
            xtick={1,2,3,4,5,6,7,8},
            ytick={90.0,90.5,91.0,91.5,92.0},
             xticklabel style = {font=\tiny, rotate=45, anchor=east},
              yticklabel style = {font=\scriptsize, rotate=0, anchor=west},
            xticklabels={0.5,0.8, 0.9,0.92, 0.95},
            yticklabels={90.0,90.5,91.0,91.5,92.0},
            every axis y label/.style={font=\footnotesize,at={(current axis.north east)},left=1mm,above=0mm},
            axis y line*=right,
            axis x line=none,
            legend style={fill=none,font=\small,at={(0.98,0.99)},anchor=north east,draw=none},
        ]
        \addplot[line width=0.4mm, mark=o, color=mygreen-new]  
            plot coordinates {
                (3, 90.5)
                (5, 90.9)
                (6, 91.1)
                (7, 91.5)
                (8, 91.0)
        };
        \end{axis}
    \end{tikzpicture}
    }
\vspace{-3mm}
\subfloat[Varying $\gamma$]{
    \begin{tikzpicture}[scale=1,every mark/.append style={mark size=2.5pt}]
        \begin{axis}[
            height=\columnwidth/2.4,
            width=\columnwidth/2.0,
            ylabel={\it \textcolor{myorange-new}{Acc}},
            xmin=0.5, xmax=9.5,
            ymin=61.0, ymax=65.0,
            xtick={1,2,3,4,5,6,7,8,9},
            ytick={60.0,61.0,62.0,63.0,64.0,65.0},
            xticklabel style = {font=\tiny, rotate=45, anchor=east},
            yticklabel style = {font=\scriptsize, rotate=0, anchor=east},
            xticklabels={0.01,0.0122,0.0125,0.0244,0.05,0.1,0.125,0.25,0.5},
            yticklabels={60.0,61.0,62.0,63.0,64.0,65.0},
            every axis y label/.style={font=\footnotesize,at={(current axis.north west)},right=1mm,above=0mm},
            axis y line*=left,
            axis x line*=bottom,
            legend style={fill=none,font=\small,at={(0.02,0.99)},anchor=north west,draw=none},
        ]
        \addplot[line width=0.4mm, mark=diamond, color=myorange-new] 
            plot coordinates {
                (1, 63.58)
                (3, 63.64)
                (5, 63.61)
                (7, 63.46)
                (8, 63.69)
        };
        \end{axis}
        \begin{axis}[
            height=\columnwidth/2.4,
            width=\columnwidth/2.0,
            ylabel={\it \textcolor{mygreen-new}{Acc}},
            xmin=0.5, xmax=9.5,
            ymin=88.0, ymax=92.0,
            xtick={1,2,3,4,5,6,7,8,9},
            ytick={88.0,89.0,90.0,91.0,92.0},
            xticklabel style = {font=\tiny, rotate=45, anchor=east},
            xticklabel style = {font=\tiny, rotate=0, anchor=west},
            xticklabels={0.01,0.0122,0.0125,0.0244,0.05,0.1,0.125,0.25,0.5},
            yticklabels={88.0,89.0,90.0,91.0,92.0},
            every axis y label/.style={font=\footnotesize,at={(current axis.north east)},left=1mm,above=0mm},
            axis y line*=right,
            axis x line=none,
            legend style={fill=none,font=\small,at={(0.98,0.99)},anchor=north east,draw=none},
        ]
        \addplot[line width=0.4mm, mark=o, color=mygreen-new]  
            plot coordinates {
                (2, 90.76)
                (4, 91.50)
                (6, 90.40)
                (8, 89.83)
                (9, 89.33)
        };
        \end{axis}
    \end{tikzpicture}
    }
\subfloat[ Varying $\lambda$]{
    \begin{tikzpicture}[scale=1,every mark/.append style={mark size=2.5pt}]
        \begin{axis}[
            height=\columnwidth/2.4,
            width=\columnwidth/2.0,
            ylabel={\it \textcolor{myorange-new}{Acc}},
            xmin=0.5, xmax=6.5,
            ymin=63.2, ymax=64.0,
            xtick={1,2,3,4,5,6},
            ytick={63.0,63.2,63.4,63.6,63.8, 64.0},
            xticklabel style = {font=\tiny, rotate=45, anchor=east},
            yticklabel style = {font=\scriptsize, rotate=0, anchor=east},
            xticklabels={0.01, 0.015, 0.05, 0.1, 0.5, 1.0},
            yticklabels={63.0,63.2,63.4,63.6,63.8, 64.0},
            every axis y label/.style={font=\footnotesize,at={(current axis.north west)},right=1mm,above=0mm},
            axis y line*=left,
            axis x line*=bottom,
            legend style={fill=none,font=\small,at={(0.02,0.99)},anchor=north west,draw=none},
        ]
        \addplot[line width=0.4mm, mark=diamond, color=myorange-new] 
            plot coordinates {
                (1, 63.3)
                (3, 63.4)
                (4, 63.7)
                (5, 63.8)
                (6, 63.4)
        };
        \end{axis}
        \begin{axis}[
            height=\columnwidth/2.4,
            width=\columnwidth/2.0,
            ylabel={\it \textcolor{mygreen-new}{Acc}},
            xmin=0.5, xmax=6.5,
            ymin=87.5, ymax=91.5,
            xtick={1,2,3,4,5,6},
            ytick={87.5, 88.5, 89.5, 90.5, 91.5},
             xticklabel style = {font=\tiny, rotate=45, anchor=east},
              yticklabel style = {font=\scriptsize, rotate=0, anchor=west},
            xticklabels= {0.01,0.015, 0.05,0.5, 1.0},
            yticklabels= {87.5, 88.5, 89.5, 90.5, 91.5},
            every axis y label/.style={font=\footnotesize,at={(current axis.north east)},left=1mm,above=0mm},
            axis y line*=right,
            axis x line=none,
            legend style={fill=none,font=\small,at={(0.98,0.99)},anchor=north east,draw=none},
        ]
        \addplot[line width=0.4mm, mark=o, color=mygreen-new]  
            plot coordinates {
                (1, 91.4)
                (2, 91.5)
                (3, 90.9)
                (4, 89.3)
                (5, 87.8)
        };
        \end{axis}
    \end{tikzpicture}
    }
\end{small}
 \vspace{-3mm}
\caption{Hyper-parameter Analysis.} 
\label{fig:hyperparameter}
\vspace{-1ex}
\end{figure}

\section{Related Work}
\stitle{Graph Data Distillation (GDD)}
Unlike conventional {\em graph reduction} approaches, such as {\em coreset selection}, {\em graph sparsification}, and {\em coarsening}, GDD achieves encouraging performance due to the mechanism of leveraging the gradients/representation of GNNs in downstream tasks for graph dataset synthesis, which can be roughly categorized into several types. Among gradient matching, \texttt{GCond} \cite{jin2021graph},  \texttt{DosCond} \cite{jin2022condensing} align gradient signals between original and condensed graphs to ensure model performance transfer. \texttt{SGDD} \cite{yang2023does}, and \texttt{GCSR} \cite{liu2024graph} are another two structure-aware gradient-based methods, preserving essential topology information. However, gradient matching needs nested optimization, making the balance between high condensation costs and low condensation quality difficult.  
Training trajectory matching approaches like \texttt{SFGC} \cite{zheng2024structure} preserve GNN learning consistency through long-term trajectory alignment but incur high computational costs in expert model training. Distribution matching methods (\texttt{GCDM} \cite{liu2022graph}, \texttt{SimGC} \cite{xiao2024simple}) iteratively align representation distributions between synthetic and original graphs, resulting in suboptimal efficiency. Eigen matching methods such as \texttt{GDEM} \cite{liugraph} maintain graph spectral properties through eigen-decomposition alignment, yet struggle with scalability for large graphs. Performance matching (\texttt{GC-SNTK}  \cite{wang2024fast}, \texttt{KiDD} \cite{xu2023kernel}) provide closed-form solutions via kernel regression but suffer from prohibitive memory demands for kernel matrix storage. Gradient-free methods like \texttt{CGC} \cite{gao2024rethinking} employ clustering-based feature matching but sacrifice task-specific adaptability by avoiding gradient optimization. While existing works explore diverse graph condensation paradigms, they exhibit limitations in computational efficiency (trajectory/eigen decomposition costs), memory consumption (kernel matrices), or optimization flexibility (lack of gradient guidance). These bottlenecks motivate us to design an efficient, structure-aware, and optimization-friendly GDD method.
\eat{\renchi{where is \cite{fang2024exgc}?}: Because EXGC is hard to grouped into traditional gradient matching, distribution matching, etc, and the code quality of EXGC is poor, We should probably not elaborate on EXGC}

\stitle{Data Distillation}
Dataset distillation~\cite{radosavovic2018data,sachdeva2023data,lei2023comprehensive} is a technique that compresses the knowledge of a large dataset into a smaller, synthetic dataset, enabling models to be trained with less data while maintaining comparable performance to models trained on the original dataset. It is not only applied to graph data, but also widely used in other fields e.g., such as visual datasets \cite{yin2024squeeze,yin2024dataset} and text datasets \cite{tao2024textual,li2021data}.For example, \texttt{DaLLME} \cite{tao2024textual} compresses two text datasets {\em IMDB} and {\em AG-News} to $0.1\%$ via language model embedding, with text classification loss decreases less than $10\%$.  
Designing unified data distillation methods for multi-modal data will be a promising direction. 

\section{Conclusion}
\label{conclusion}
This paper proposes a simple yet effective approach \algo{} for GDD. \algo{} achieves superior performance in condensation effectiveness and efficiency over previous GDD solutions on real datasets through two major contributions: a simple clustering method minimizing the WCSS and a lightweight module augmenting synthetic attributes with class-relevant features. As for future work, we intend to extend our solutions to handle graphs that are heterogeneous or from multiple modalities.

\begin{acks}
This work is supported by the National Natural Science Foundation of China (No. 62302414), the Hong Kong RGC ECS grant (No. 22202623), and the Huawei Gift Fund.
\end{acks}

\balance

\bibliographystyle{ACM-Reference-Format}
\bibliography{main}

\clearpage
\newpage
\appendix
\section{A Detailed Discussion of Existing GDD Works}
\label{appendix:formalize_GDD}
We here provide a concise formalization of some representative GDD methods. 

\stitle{Gradient matching} e.g. \texttt{GCond} \cite{jin2021graph} optimizes the synthetic graphs by minimizing the difference between the gradients of GNNs on original and synthetic graphs. 
\begin{small}
\begin{equation*}
\begin{gathered}
\min_{\G^\prime} \mathbb{E}_{ {\boldsymbol{\Theta}} \sim \mathcal{P}_{\boldsymbol{\Theta}}} \left[ 
      \sum^{L}_{l=1} D\left( \nabla_{\boldsymbol{\Theta}} \mathcal{L}(\textsf{GNN}_{\boldsymbol{\Theta}}(\mathbf{A}', \mathbf{X}'), \mathbf{Y}'), \nabla_{\boldsymbol{\Theta}} \mathcal{L}(\textsf{GNN}_{\boldsymbol{\Theta}}(\mathbf{A}, \mathbf{X}), \mathbf{Y}) \right) \right] \\
s.t. \quad \boldsymbol{\Theta}_{\G^\prime} = \argmin{ \boldsymbol{\Theta}}{\mathcal{L}(\textsf{GNN}_{\boldsymbol{\Theta}}(\AM^\prime, \XM^\prime), \YM^\prime)}
\end{gathered}
\end{equation*}
\end{small}
where $L$ is the number of matching steps. In fact, \texttt{GCond} involves a nested optimization. It optimizes the synthetic dataset within the inner loop and trains the GNNs on the original graph in the outer loop, which makes the training process very time-consuming. 

\stitle{Distribution matching} \texttt{GCDM} \cite{liu2022graph} minimizes the distance between distribution of GNN each layer representations,
\begin{equation*}
\begin{gathered}
 \min_{\G^\prime} \mathbb{E}_{ {\boldsymbol{\Theta}} \sim \mathcal{P}_{\boldsymbol{\Theta}}} \left[
    \sum^{T}_{t=1} D\left(\textsf{GNN}_{\boldsymbol{\Theta}}(\mathbf{A}', \mathbf{X}'), \textsf{GNN}_{\boldsymbol{\Theta}}(\mathbf{A}, \mathbf{X}) \right)
    \right]  
\end{gathered}
\end{equation*}
where $T$ is the number of GNN layers. \texttt{GCDM} has lower complexity than \texttt{GCond}, which only needs to calculate the distance between the distributions of the original graph and the synthetic graph, while \texttt{GCond} needs to calculate gradients and update the model parameters in each iteration, which is more expensive. \texttt{GCDM} doesn't fully utilize label information and still needs alternative training on GNN parameter and synthetic graph, resulting in limited performance. 

\stitle{Trajectory matching} \texttt{SFGC} \cite{zheng2024structure} optimizes the synthetic graph by minimizing the difference between the trajectories of GNNs on original and synthetic graphs.
\begin{equation*}
\begin{gathered}
\min_{\G^\prime} \mathbb{E}_{ {\boldsymbol{\Theta}}^{*,i}_{t} \sim \mathcal{P}_{\boldsymbol{\Theta}_{\G}}} \left[ 
     \mathcal{L}_{meta-tt}( {\boldsymbol{\Theta}^{*}_{t}} {|^{p}_{t=t_0}} , {\tilde{\boldsymbol{\Theta}}_{t}} {|^{q}_{t=t_0}})
     \right] \\
s.t. \quad \tilde{\boldsymbol{\Theta}}^*_{\G^\prime} = \argmin{ \tilde{\boldsymbol{\Theta}}}{\mathcal{L}(\textsf{GNN}_{\tilde{\boldsymbol{\Theta}}}(\AM^\prime, \XM^\prime), \YM^\prime)}
\end{gathered}
\end{equation*}
where $\boldsymbol{\Theta}^{*}_{t}|^{p}_{t=t_0}$ are the parameters of $\textsf{GNN}$ on $\G$ in the training interval $[{\boldsymbol{\Theta}}^{*,i}_{t_0}, {\boldsymbol{\Theta}}^{*,i}_{t_0 + p}]$ and  $\tilde{\boldsymbol{\Theta}}_{t}|^{q}_{t=t_0}$ are the parameters of $\textsf{GNN}$ on $\G'$ in the training interval $[{ \tilde{\boldsymbol{\Theta}}}_{t_0}, \tilde{\boldsymbol{\Theta}}_{t_0 + q}]$. Training experts to obtain trajectories is time-consuming, and experts may occupy a large amount of memory.

\stitle{Eigen matching} \texttt{GDEM} \cite{liugraph} firstly conducts {\em singular value decomposition} (SVD) to the original adjacency matrix $\AM$ to get eigen values $\Lambda \in \mathbb{R}^{K}$ and eigen-basis $\UM \in \mathbb{R}^{N \times K}$. Then \texttt{GDEM} aims to synthetic $\XM'$, $\YM'$ and the synthetic eigen-basis $\UM'$, which can be formalized as below,
\begin{small}
    \begin{equation*}
\begin{gathered}
\min_{\G^\prime}
    \sum^{K}_{k=1} D\left( (\XM^{\top}\UM_k)(\XM^{\top}\UM_k)^{\top}, (\XM'^{\top}\UM'_k)(\XM'^{\top}\UM'_k)^{\top} \right) + D\left( \UM'\UM'^{\top}, \IM \right)  \\+
    D'\left(\YM^{\top}\AM\XM, \YM'^{\top}\sum^K_{k=1}(1-\Lambda_k)\UM'_k\UM'^{\top}_k\XM'\right) 
\end{gathered}
\end{equation*}
\end{small}
The primary computational cost of \texttt{GDEM} lies in the SVD decomposition. As the size of the graph increases, the time required for SVD decomposition grows rapidly, far exceeding the time needed for the matching process.

\stitle{Performance matching} \texttt{GC-SNTK} \cite{wang2024fast} matches the performance of models on original and synthetic graphs by {\em kernel ridge regression }(KRR), which obtains the closed-form solution without bi-level optimization process. It can be written as:  
\begin{equation*}
\begin{gathered}
 \min_{\G^\prime} \mathcal{L}_{pm} = \frac{1}{2}\|\mathcal{K}_{\G\G^\prime}(\mathcal{K}_{\G^\prime\G^\prime}+ \epsilon\IM)^{-1}\YM'- \YM\|^2_F
\end{gathered}
\end{equation*}
where $\mathcal{K}_{\G\G^\prime} \in \mathbb{R}^{N \times n}$ and $\mathcal{K}_{\G^\prime\G^\prime} \in \mathbb{R}^{n \times n}$ are two graph kernels, which can be obtained by iterative aggregation on the graph. The high computational and space demands of KRR for graph data mainly stem from the quadratic cost of computing graph kernels, the quadratic storage requirement for the kernel matrix, and the cubic computation needed to invert it, making KRR challenging for large-scale graph datasets.

\stitle{Training free} \texttt{CGC} \cite{gao2024rethinking} is a recently proposed method for fast and high quality graph condensation. It utilizes \texttt{SGC} feature propagation to get node embedding, and uses multi-layer embedding combination for label prediction via MSE. 
\begin{equation*}
\begin{gathered}
 \argmin{\WM} \|\frac{1}{K}\sum^K_{k=1}\NAM^k\XM\WM- \YM\|^2_2
\end{gathered}
\end{equation*}
which has a closed form solution of $\WM$:
$\hat{\WM} = (\frac{1}{T}\sum^T_{t=1}\NAM^t\XM)^{\dagger} \YM $

\texttt{CGC} doing clustering on the class-partitioned embedding $\HM$ to get class-wise aggregation $\RM^{(y)}$, then synthetic the condensed node embedding by $\HM'^{(y)} = \RM^{(y)}\HM^{(y)}$. Through on condensed embedding similarity, the condensed adjacency matrix $\AM'$ can be obtained. Finally, the synthetic node attribute $\XM'$ follows:
\begin{equation*}
\begin{gathered}
 \argmin{\XM'} \mathcal{L} = \|\NAM^T\XM'- \YM\|^2_2 + \alpha \textsf{tr}(\XM'^{\top}\LM'\XM')
\end{gathered}
\end{equation*}
which has the closed-form solution:
\begin{equation*}
\begin{gathered}
\XM' = ((\NAM'^T)^{\top}\NAM'^T + \alpha\LM')^{-1}{(\NAM'^T)}^{\top}\HM'
\end{gathered}
\end{equation*}.

A notable similarity between our \algo{} and \texttt{CGC} is that they both employ clustering to generate node attributes efficiently. However, compared to \texttt{CGC}, our \algo{} has some advantages. In terms of motivation, we choose clustering because we have discovered the inherent theoretical connection between clustering and FID, an indicator for measuring the quality of synthetic graphs, which is different from the starting point of \texttt{CGC}. In terms of specific implementation, there are several key differences. First, \texttt{CGC} incorporates  \texttt{SGC} propagation, which might miss some local information on large graph. Our method, on the other hand, uses the closed-form solution of GLS on the graph to balance local and global information, leading to better node attributes.
Second, to address the issue of heterophilic over-smoothing from clustering, we refine the synthetic node attributes quickly and effectively. While our method sacrifices some speed within an acceptable range, the trade-off is well worth it, as we maintain a higher level of effectiveness.

\begin{table*}[!t]
\caption{Hyperparameters of \algo{}}
\label{tab:hyperparams}
\vspace{-2ex}
\renewcommand{\arraystretch}{0.8}
\begin{small}
\begin{tabular}{ccc ccc ccc cccc c}
\toprule
{\bf Dataset} & {\bf Ratio} & $T$  &  $\alpha$  &  $E_1$  & $\beta$ &  $\rho$ &  $T'$  &  $E_3$  &  $\gamma$ & $\lambda$ & dropout & $H$ \\
\midrule
\multirow{3}{*}{\em Cora}   
& $1.3\%$   & $5$   & $0.8$    & $80$   & $0.01$    & $0.06$     & $2$  & $2000$ & $7.0$ & $0.1$ & $0.6$ & $256$   \\
& $2.6\%$    & $5$   & $0.8$    & $80$   & $0.01$    & $0.4$    & $2$ & $2000$  & $7.0$ & $0.1$ & $0.6$ & $256$ \\
& $5.2\%$   & $20$ & $0.8$ & $80$ & $0.01$ & $0.4$ & $2$ & $2000$ & $7.0$ & $0.1$ & $0.6$ & $256$ \\
\midrule
\multirow{3}{*}{\em Citeseer}   
& $0.9\%$ & $2$ & $0.8$ & $120$ & $0.01$ & $0.06$ & $1$ & $80$ & $6.0$ & $0.1$ & $0.7$ & $256$ \\
& $1.80\%$    & $2$   & $0.5$    & $120$   & $0.01$   & $0.21$ & $1$   & $200$  & $0.3$ & $0.1$ & $0.8$ & $128$ \\
& $3.60\%$   & $2$   & $0.5$    & $120$   & $0.01$    & $0.2$ & $1$   & $200$  & $5.4$  & $0.1$ & $0.7$ & $256$ \\
\midrule
\multirow{3}{*}{\em arXiv} 
& $0.05\%$   & $20$  & $0.92$  & $1000$   & $1.4$   & $0.1$ & $10$   & $1000$  & $0.5$ & $0.1$ & $0.6$ & $256$   \\
& $0.25\%$    & $18$  & $0.91$   & $1000$   & $1.9$   & $0.1$ & $10$ & $1000$  & $1.0$ & $0.1$ & $0.6$ & $256$  \\
& $0.50\%$    & $20$  & $0.92$   & $1000$   & $1.0$   & $0.1$ & $10$  & $1000$  & $1.0$ & $0.1$ & $0.6$  & $256$ \\
\midrule
\multirow{3}{*}{\em Flickr}     
& $0.10\%$   & $2$  & $0.8$  & $2000$   & $0.2$   & $0.5$ & $2$   & $2000$  & $5.6$ & $0.8$ & $0.6$ & $256$   \\
& $0.50\%$    & $2$  & $0.8$   & $2000$   & $0.2$   & $0.5$ & $1$ & $2000$  & $5.6$ & $1.0$ & $0.6$ & $256$  \\
& $1.00\%$    & $2$  & $0.8$   & $2000$   & $0.2$   & $0.5$ & $2$  & $2000$  & $5.6$ & $0.8$ & $0.6$  & $256$ \\
\midrule
\multirow{3}{*}{\em Reddit}   
& $0.05\%$   & $10$  & $0.92$  & $800$   & $3.5$   & $0.1$ & $7$   & $600$  & $1.6$ & $0.015$ & $0.6$ & $512$  \\
& $0.10\%$   & $10$  & $0.92$   & $800$   & $3.3$   & $0.1$ & $7$ & $600$  & $1.0$ & $0.015$ & $0.6$ & $512$  \\
& $0.50\%$   & $20$  & $0.95$   & $800$   & $0.8$   & $0.1$ & $10$  & $1000$  & $2.1$ & $0.01$ & $0.6$ & $256$  \\
\bottomrule
\end{tabular}
\end{small}
\end{table*}

\section{Time Complexity Analysis of Graph Distillation Methods}
\label{sec:appendix-time-com}

We here calculate some graph distillation methods for comparison. For \texttt{GCond}\cite{jin2021graph}, let $T$ is the number of GNN layers, and $r$ is the number of sampled neighbors per node. Denote the outer-loop be $L_o$ and inner loop as $L_i$, the training epoch be $E$. The total complexity of \texttt{GCond} is $\mathcal{O}\big(EL_oL_ir^TNd^2 + EL_oL_in^2d^2\big)$. 


For \texttt{SFGC}, let the product of the number of experts $L_n$ and the training epochs of experts $E$, and the number of outer-loop $L_o$ and the inner loop be $L_i$. The total complexity is $\mathcal{O}\big(EL_nMdT+EL_nNd^2T+L_oL_ind^2T\big)$. 


For \texttt{GCSR}, denote the number of experts as $L_n$, the training epochs of experts as $E_t$, the number of meta matching $L_o$ and GNNs training epochs be $E_i$, the total complexity is $\mathcal{O}\big( E_tL_nr^TNd^2 + E_iL_ond^2T +n^3 \big)$. 


For \texttt{GDEM}, let the eigen basis number be $N_k$, the eigen-basis matching training epoch is $E$, the time complexity of \texttt{GDEM} can be written as $\mathcal{O}\big( N_kN^2 + N_kNd + Md + N_kEn^2 + N_kEd^2 \big)$. 


To simplify, we abbreviate $L_n$, $L_t$, $L_o$ and $L_i$ as $L$, and $E, E_t, E_i$ as $E$. \algo{} has better time complexity in terms of polynomial order ($\mathcal{O}\big( ENnd + n^2d + Md + M\log M \big)$). Compared to \texttt{GCond}, \texttt{SFGC}, \texttt{GCSR}, and \texttt{GDEM}, the complexity terms of \algo{} do not include higher-order terms such as $EL$, $L^2$ or $N_kN^2$, which can lead to significantly higher computational costs as the scale of the data increases. Therefore, \algo{} outperforms other algorithms in terms of efficiency when dealing with large-scale graph data, especially when $N$ and $d$ are large. 

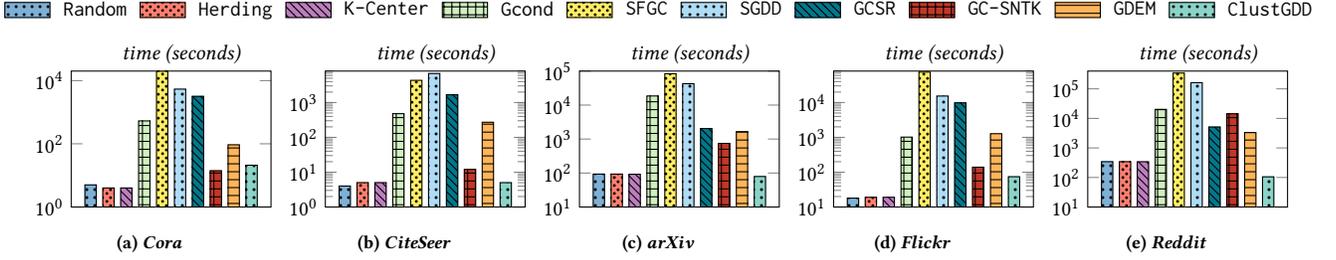
\begin{figure*}
\centering
\begin{small}
\begin{tikzpicture}
\begin{customlegend}[
        legend entries={\texttt{Random},\texttt{Herding}, \texttt{K-Center}, \texttt{Gcond},{\texttt{SFGC}}, {\texttt{SGDD}}, {\texttt{GCSR}}, {\texttt{GC-SNTK}}, {\texttt{GDEM}},\algo{}},
        legend columns=10,
        area legend,
        legend style={at={(0.45,0.65)},anchor=north,draw=none,font=\small,column sep=0.1cm,
        }]
        \addlegendimage{preaction={fill, myblue}, pattern=dots} %
        \addlegendimage{preaction={fill, myred}, pattern={crosshatch dots}} %
        \addlegendimage{preaction={fill, mypurple}, pattern=north west lines} %
        \addlegendimage{preaction={fill, mygreen}, pattern={grid}}  
        \addlegendimage{preaction={fill, myyellow}, pattern={crosshatch dots}}
        \addlegendimage{preaction={fill, myblue2}, pattern=dots} %
        \addlegendimage{preaction={fill, mygreen-new}, pattern=north west lines}
        \addlegendimage{preaction={fill, myred2}, pattern=grid} %
        \addlegendimage{preaction={fill, myorange}, pattern=horizontal lines} 
        \addlegendimage{preaction={fill, mycyan}, pattern=dots} 
    \end{customlegend}
\end{tikzpicture}
\\[-\lineskip]
\vspace{-2ex}
\subfloat[{\em Cora}]{
\begin{tikzpicture}[scale=1]
\begin{axis}[
    height=\columnwidth/2.5,
    width=\columnwidth/2.0,
    xtick=\empty,
    ybar=2.5pt,
    bar width=0.15cm,
    enlarge x limits=true,
    ylabel={\em time (seconds)},
    xticklabel=\empty,
    ymin=1,
    ymax=20000,
    log origin y=infty,
    log basis y={10},
    ymode=log,
    xticklabel style = {font=\small},
    yticklabel style = {font=\small},
    every axis y label/.style={font=\small,at={{(0.25,1.0)}},right=8mm,above=0mm},
    ]
\addplot [preaction={fill, myblue}, pattern={dots}] coordinates {(1,5) }; %
\addplot [preaction={fill, myred}, pattern={crosshatch dots}] coordinates {(1,4) }; %
\addplot [preaction={fill, mypurple}, pattern={north west lines}] coordinates {(1,4) }; %
\addplot [preaction={fill, mygreen}, pattern={grid}] coordinates {(1,529) }; 
\addplot [preaction={fill, myyellow}, pattern={crosshatch dots}] coordinates {(1,19768) }; 
\addplot [preaction={fill, myblue2}, pattern=dots] coordinates {(1,5372) };%
\addplot [preaction={fill, mygreen-new}, pattern=north west lines] coordinates {(1,3166) };
\addplot [preaction={fill, myred2}, pattern={grid}] coordinates {(1,14) }; %
\addplot [preaction={fill, myorange}, pattern=horizontal lines] coordinates {(1,93) }; 
\addplot [preaction={fill, mycyan}, pattern=dots] coordinates {(1,21) }; 

\end{axis}
\end{tikzpicture}
}
\subfloat[{\em CiteSeer}]{
\begin{tikzpicture}[scale=1]
\begin{axis}[
    height=\columnwidth/2.5,
    width=\columnwidth/2.0,
    xtick=\empty,
    ybar=2.5pt,
    bar width=0.15cm,
    enlarge x limits=true,
    ylabel={\em time (seconds)},
    xticklabel=\empty,
    ymin=1,
    ymax=8000,
    log origin y=infty,
    log basis y={10},
    ymode=log,
    xticklabel style = {font=\small},
    yticklabel style = {font=\small},
    every axis y label/.style={font=\small,at={{(0.25,1.0)}},right=8mm,above=0mm},
    ]

\addplot [preaction={fill, myblue}, pattern={dots}] coordinates {(1,4) }; %
\addplot [preaction={fill, myred}, pattern={crosshatch dots}] coordinates {(1,5) }; %
\addplot [preaction={fill, mypurple}, pattern={north west lines}] coordinates {(1,5) }; %
\addplot [preaction={fill, mygreen}, pattern={grid}] coordinates {(1,479) }; 
\addplot [preaction={fill, myyellow}, pattern={crosshatch dots}] coordinates {(1,4334) }; 
\addplot [preaction={fill, myblue2}, pattern=dots] coordinates {(1,6795) };%
\addplot [preaction={fill, mygreen-new}, pattern=north west lines] coordinates {(1,1670) };
\addplot [preaction={fill, myred2}, pattern={grid}] coordinates {(1,12) }; %
\addplot [preaction={fill, myorange}, pattern=horizontal lines] coordinates {(1,271) }; 
\addplot [preaction={fill, mycyan}, pattern=dots] coordinates {(1,5) };

\end{axis}
\end{tikzpicture}
}
\subfloat[{\em arXiv}]{
\begin{tikzpicture}[scale=1]
\begin{axis}[
    height=\columnwidth/2.5,
    width=\columnwidth/2.0,
    xtick=\empty,
    ybar=2.5pt,
    bar width=0.15cm,
    enlarge x limits=true,
    ylabel={\em time (seconds)},
    xticklabel=\empty,
    ymin=10,
    ymax=100000,
    log origin y=infty,
    log basis y={10},
    ymode=log,
    ytick={10,100,1000,10000,100000},
    xticklabel style = {font=\small},
    yticklabel style = {font=\small},
    every axis y label/.style={font=\small,at={{(0.25,1.0)}},right=8mm,above=0mm},
    ]

\addplot [preaction={fill, myblue}, pattern={dots}] coordinates {(1,93) }; %
\addplot [preaction={fill, myred}, pattern={crosshatch dots}] coordinates {(1,93) }; %
\addplot [preaction={fill, mypurple}, pattern={north west lines}] coordinates {(1,92) }; %
\addplot [preaction={fill, mygreen}, pattern={grid}] coordinates {(1,18509) }; 
\addplot [preaction={fill, myyellow}, pattern={crosshatch dots}] coordinates {(1,83441) }; 
\addplot [preaction={fill, myblue2}, pattern=dots] coordinates {(1,42508) };%
\addplot [preaction={fill, mygreen-new}, pattern=north west lines] coordinates {(1,2037) };
\addplot [preaction={fill, myred2}, pattern={grid}] coordinates {(1,745) }; %
\addplot [preaction={fill, myorange}, pattern=horizontal lines] coordinates {(1,1647) }; 
\addplot [preaction={fill, mycyan}, pattern=dots] coordinates {(1,79) }; 

\end{axis}
\end{tikzpicture}
}%
\subfloat[{\em Flickr}]{
\begin{tikzpicture}[scale=1]
\begin{axis}[
    height=\columnwidth/2.5,
    width=\columnwidth/2.0,
    xtick=\empty,
    ybar=2.5pt,
    bar width=0.15cm,
    enlarge x limits=true,
    ylabel={\em time (seconds)},
    xticklabel=\empty,
    ymin=10,
    ymax=80000,
    log origin y=infty,
    log basis y={10},
    ymode=log,
    xticklabel style = {font=\small},
    yticklabel style = {font=\small},
    every axis y label/.style={font=\small,at={{(0.25,1.0)}},right=8mm,above=0mm},
    ]

\addplot [preaction={fill, myblue}, pattern={dots}] coordinates {(1,18) }; %
\addplot [preaction={fill, myred}, pattern={crosshatch dots}] coordinates {(1,19) }; %
\addplot [preaction={fill, mypurple}, pattern={north west lines}] coordinates {(1,19) }; %
\addplot [preaction={fill, mygreen}, pattern={grid}] coordinates {(1,1015) }; 
\addplot [preaction={fill, myyellow}, pattern={crosshatch dots}] coordinates {(1,74780) }; 
\addplot [preaction={fill, myblue2}, pattern=dots] coordinates {(1,15355) };%
\addplot [preaction={fill, mygreen-new}, pattern=north west lines] coordinates {(1,9783) };
\addplot [preaction={fill, myred2}, pattern={grid}] coordinates {(1,139) }; %
\addplot [preaction={fill, myorange}, pattern=horizontal lines] coordinates {(1,1273) }; 
\addplot [preaction={fill, mycyan}, pattern=dots] coordinates {(1,74) }; 

\end{axis}
\end{tikzpicture}
}
\subfloat[{\em Reddit}]{
\begin{tikzpicture}[scale=1]
\begin{axis}[
    height=\columnwidth/2.5,
    width=\columnwidth/2.0,
    xtick=\empty,
    ybar=2.5pt,
    bar width=0.15cm,
    enlarge x limits=true,
    ylabel={\em time (seconds)},
    xticklabel=\empty,
    ymin=10,
    ymax=400000,
    log origin y=infty,
    log basis y={10},
    ymode=log,
    ytick={10,100,1000,10000,100000},
    xticklabel style = {font=\small},
    yticklabel style = {font=\small},
    every axis y label/.style={font=\small,at={{(0.25,1.0)}},right=8mm,above=0mm},
    ]

\addplot [preaction={fill, myblue}, pattern={dots}] coordinates {(1,342) }; %
\addplot [preaction={fill, myred}, pattern={crosshatch dots}] coordinates {(1,342) }; %
\addplot [preaction={fill, mypurple}, pattern={north west lines}] coordinates {(1,340) }; %
\addplot [preaction={fill, mygreen}, pattern={grid}] coordinates {(1,20220) }; 
\addplot [preaction={fill, myyellow}, pattern={crosshatch dots}] coordinates {(1,342672) }; 
\addplot [preaction={fill, myblue2}, pattern=dots] coordinates {(1,162093) };%
\addplot [preaction={fill, mygreen-new}, pattern=north west lines] coordinates {(1,5096) };
\addplot [preaction={fill, myred2}, pattern={grid}] coordinates {(1,14317) }; %
\addplot [preaction={fill, myorange}, pattern=horizontal lines] coordinates {(1,3298) }; 
\addplot [preaction={fill, mycyan}, pattern=dots] coordinates {(1,105) }; 

\end{axis}
\end{tikzpicture}
}
\end{small}
\vspace{-2ex}
\caption{Computational time comparison } \label{fig:efficiency}
\vspace{-1ex}
\end{figure*}

\section{Datasets and Implementations Details}\label{sec:more-data-imp-details}

\subsection{Dataset Deatils}\label{sec:dataset-detail}
\stitle{{\em Cora}}~\cite{sen2008collective} is a citation network consisting of 2,708 papers with 5,429 citation links, where nodes represent papers and edges for citations. The node attributes are 1433-dimensional vectors obtained by using the bag-of-words representation of the abstract. It has 7 categories of different paper topics. {\em Cora} has a training set of 140 nodes ($5.2\%$), a validation set of 500 nodes ($18.5\%$), and a test set of 1,000 nodes ($36.9\%$) for transductive node classification, with the remaining nodes typically used for unsupervised or semi-supervised learning. 

\stitle{{\em Citeseer}}~\cite{sen2008collective} is another citation network. Similar to {\em Cora}, it contains papers and citations of computer science, with 3,327 nodes and 4,732 edges. It has 3,703-dimensional attributes and 6 topic classes. {\em Citeseer} includes a training set of 120 nodes ($3.6\%$), a validation set of 500 nodes ($15.0\%$), and a test set of 1,000 nodes ($30.0\%$) for transductive node classification. 

\stitle{{\em arXiv}}~\cite{hu2020open} is from the Open Graph Benchmark (OGB). It includes 169,343 {\em arXiv} papers divided into 40 subjects, and 1,166,243 citations. It has 128-dimensional Word2Vec node attributes.  
The {\em arXiv} dataset is for transductive node classification with temporally splits. Its training set consisting of papers published up to 2017 (90,941 nodes, $53.7\%$), the validation set consisting of papers in 2018 (29,799 nodes, $17.6\%$), and the test set including papers from 2019 (48,603 nodes, $28.7\%$), emphasizing temporal generalization. 

\stitle{{\em Flickr}}~\cite{zeng2019graphsaint} is a social network consisting of 89,250 users and 899,75\\6 interactions. In this network, nodes stand for users, edges indicate follower relationships, and the node attributes are 500-dimensional visual features extracted from the images uploaded by users. The nodes can be categorized into 7 classes. {\em Flickr} is randomly split into a training set of 44,625 nodes ($50\%$), a validation set of 22,312 nodes ($25\%$), and a test set of 22,313 nodes ($25\%$), used for evaluating inductive node classification task. 

\stitle{{\em Reddit}}~\cite{hamilton2017inductive} is another social network, with 232,965 users and 11,606,919 interactions. The users are grouped into 41 sub{\em Reddit} categories. User activity and text  are represented by 602-dimensional node attributes.The {\em Reddit} dataset is also randomly divided, with a training set of 152,410 nodes ($65.4\%$), a validation set of 23,699 nodes ($10.2\%$), and a test set of 55,334 nodes ($23.7\%$), designed for large-scale inductive node classification tasks.

\subsection{Implementation Details}\label{sec:implement-detail}
We implement the GNN models and graph distillation by PyTorch Geometric. We collect graphs {\em Cora}, {\em Citeseer}, {\em arXiv}, {\em Flickr} and {\em Reddit} from as \texttt{GCond}. We obtain the synthetic graphs generated by \texttt{GCond}\footnote{\url{https://github.com/ChandlerBang/GCond}},
\texttt{SFGC}\footnote{\url{https://github.com/Amanda-Zheng/SFGC}},
\texttt{GCSR}\footnote{\url{https://github.com/zclzcl0223/GCSR}}, and
\texttt{GDEM}\footnote{\url{https://github.com/liuyang-tian/GDEM}}. Main experiments are conducted on a Linux machine equipped with an Intel(R) Xeon(R) Gold 6226 CPU @ 2.70GHz and a 32 GB Nvidia Tesla V100 GPU. 

We set learning rate of GNNs to $0.01$, the weight decay is $5e-4$. For ease of training, $\WM$ and $\WM'$ are two three-layer linear layers, with the middle layer having a dimension of $H$ is in $\{128,256,512\}$. The dropout is in $\{0.6, 0.7, 0.8\}$. 

We implement K-Means based on the Sklearn, the default maximum number of iterations $E_2$ for the K-means algorithm is $300$. The algorithm will stop iterating if it converges earlier, meaning the change in cluster centers falls below the specified tolerance threshold, which is set to $1e-4$ by default. For large datasets {\em arXiv} and {\em Reddit}, we use Mini-Batch K-Means as an alternative, the batch size is set to $1000$. 

In the distillation stage, we conduct an extensive search to identify the optimal hyper-parameters. Specifically, we explore the following ranges for each hyper-parameter: the propagation times of the closed-form solution of GLS in clustering \( T \) are set to \(\{2, 5, 10, 15, \\18, 20, 25\}\); the coefficient in the propagation \(\alpha\) is varied across \(\{0.5, 0.8, 0.9, 0.91, 0.92, 0.95, 0.98\}\); the number of pretraining epochs \( E_1 \) is tested with values in \(\{80, 100, 120, 800, 1000, 2000\}\); the combination weight of the refinement vector \(\beta\) is adjusted within the range \([0, 4]\); the propagation times in the refinement stage \( T'\) are set to \(\{1, 2, 5, 7, 10\}\); the sampling rate in multi-view sparsification \(\rho\) is varied within the range \([0, 0.5]\); the number of refinement epochs \( E_2 \) is tested with values in \(\{80, 100, 120, 800, 1000, 2000\}\); the weight of the supervision loss on the multi-view condensed graphs \(\gamma\)  is adjusted within the range \([0, 10]\) and the weight of the consistency loss \(\lambda\) is adjusted within the range \([0, 1]\). This comprehensive search strategy allow us to fine-tune the model's performance by systematically evaluating the impact of each hyper-parameter on the overall results. We record the hyper-parameters in Table~\ref{tab:hyperparams}. 

For the evaluation stage, the evaluation model is set to \texttt{GCN} by default, which has two layers with $256$ hidden dimensions, and dropout$=0.5$, learning rate$=0.01$, weight decay$=1e-5$, the training epochs is $600$. In the generalization test, other GNN models have the same hidden dimensions, number of layers, learning rate, and dropout as \texttt{GCN}.

\section{Additional Experiments}\label{sec:add-exp}

\subsection{Detailed Dataset Synthesis Time}

In Fig.~\ref{fig:efficiency}, we display the time costs required for GDD by \algo{} and the other four competitive baselines (i.e., \texttt{Gcond}, \texttt{SFGC}, \texttt{GCSR}, and \texttt{GDEM}) on all five datasets with a condensation ratio of $2.6\%$, $1.8\%$, $0.05\%$, $0.10\%$, and $0.10\%$, respectively. The results for other condensation ratios are quantitatively similar and thus are omitted. The $x$-axis corresponds to different GDD methods, while the $y$-axis represents the condensation time in seconds (s) in a logarithmic scale. As shown in Fig.~\ref{fig:efficiency}, the computational time needed by \algo{} is significantly (often orders of magnitude) less than those by the competitors on all datasets. For example, on {\em Citeseer}, \algo{} takes about 5s, which stands in stark contrast to the 479s needed by \texttt{Gcond} and 4334s for \texttt{SFGC}. On the {\em arXiv} dataset, the running time of \algo{} is around 79s, while the others are up to more than 10,000 seconds.

\begin{table}[H]
    \centering
    \caption{Statistics of original and condensed datasets.}
\vspace{-2ex}
\resizebox{\linewidth}{!}{
    \begin{tabular}{ccccc}
        \toprule
        {\em Cora} & Original & \texttt{GCond} & \texttt{GCSR} & \algo{}  \\
        \midrule
        Nodes & 2,708 & 70 & 70 & 70 \\
        Edges & 5,429  & 4830 & 4900 & 1384 \\
        Avg Degree & 2.00  & 69.0 & 70.0 & 19.77  \\
        Weighted Homophily & 0.81 & 0.53 & 0.94 & 0.77  \\
        Storage(MB) & 14.90 & 0.40  & 0.40 & 0.40   \\
        \bottomrule
    \end{tabular}
}
\resizebox{\linewidth}{!}{
    \begin{tabular}{ccccc}
        \toprule
        {\em arXiv}  & Original & \texttt{GCond} & \texttt{GCSR} & \algo{}  \\
        \midrule
        Nodes & 169,343   & 90 & 90 & 90   \\
        Edges & 1,166,243  & 8010 & 8084 & 5384  \\
        Avg Degree & 6.89  & 89.00 & 89.82 & 59.82  \\
        Weighted Homophily & 0.65 & 0.08 & 0.08 & 0.58  \\
        Storage(MB) & 100.40 & 0.08  & 0.08 & 0.08  \\
        \bottomrule
    \end{tabular}
}
\resizebox{\linewidth}{!}{
    \begin{tabular}{ccccc}
        \toprule
        {\em Reddit}  & Original & \texttt{GCond} & \texttt{GCSR} & \algo{}  \\
        \midrule
        Nodes & 232,965 & 153 & 153 & 153  \\
        Edges & 57,307,946  & 9,190  & 23,409 & 18,093 \\
        Avg Degree & 245.99  & 60.07 & 153.00 & 118.25  \\
        Weighted Homophily & 0.78 & 0.10 & 0.21 & 0.65  \\
        Storage(MB) & 435.50 & 0.44 & 0.44  & 0.44  \\
        \bottomrule
    \end{tabular}
}
\label{tab:stat}
\vspace{-1ex}
\end{table}

\subsection{Statistics of the Synthetic Graphs}
As shown in Table~\ref{tab:stat}, we compare some statistics of the original datasets and synthetic graphs obtained by different graph distillation methods, \texttt{GCond}, \texttt{GCSR}, and \algo{}, including the number of nodes, edges, average degree, weighted homophily, and the memory required for storage. It should be noted that both \texttt{GCond} and \texttt{GCSR} provide synthetic graphs that have not been sparsified. These methods significantly reduce the number of nodes and edges, thereby decreasing storage requirements. For larger datasets, such as {\em arXiv} and {\em Reddit}, the reduction in storage is particularly noticeable. We also compare the weighted homophily of different graphs, which is the ratio of the sum of the weights of all edges between same class nodes to the sum of the weights of all edges. We find that \algo{} is closer to the weighted homophily of the original graph than \texttt{GCond} and \texttt{GCSR}, reflecting that our synthetic graph retains key information about the topology of the original graph.

\begin{figure}[H]
    \centering
    \subfloat[{\em Cora} w/o prop]{%
        \includegraphics[width=0.16\textwidth]{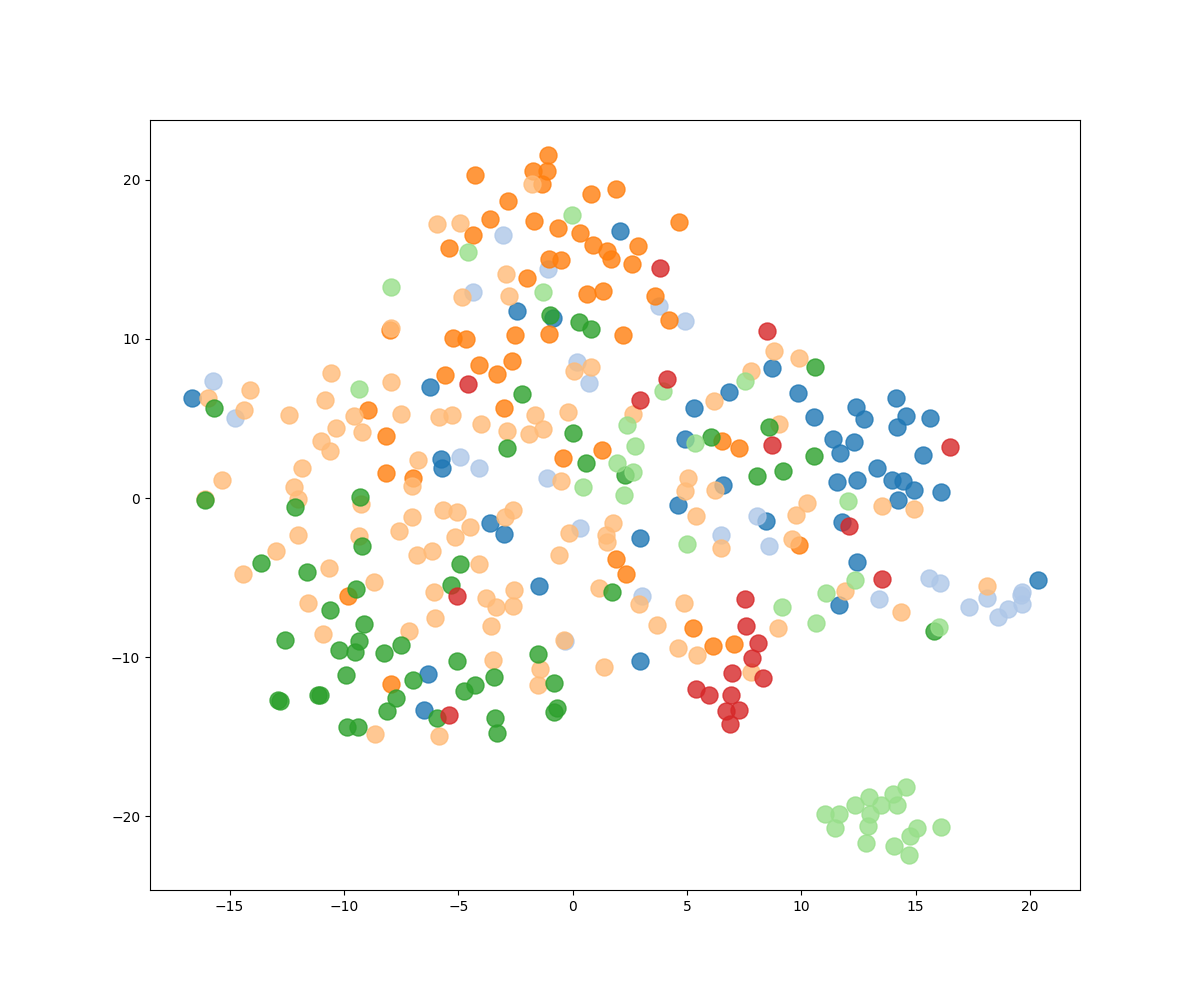}
    }
    \subfloat[{\em Cora} w prop]{%
        \includegraphics[width=0.16\textwidth]{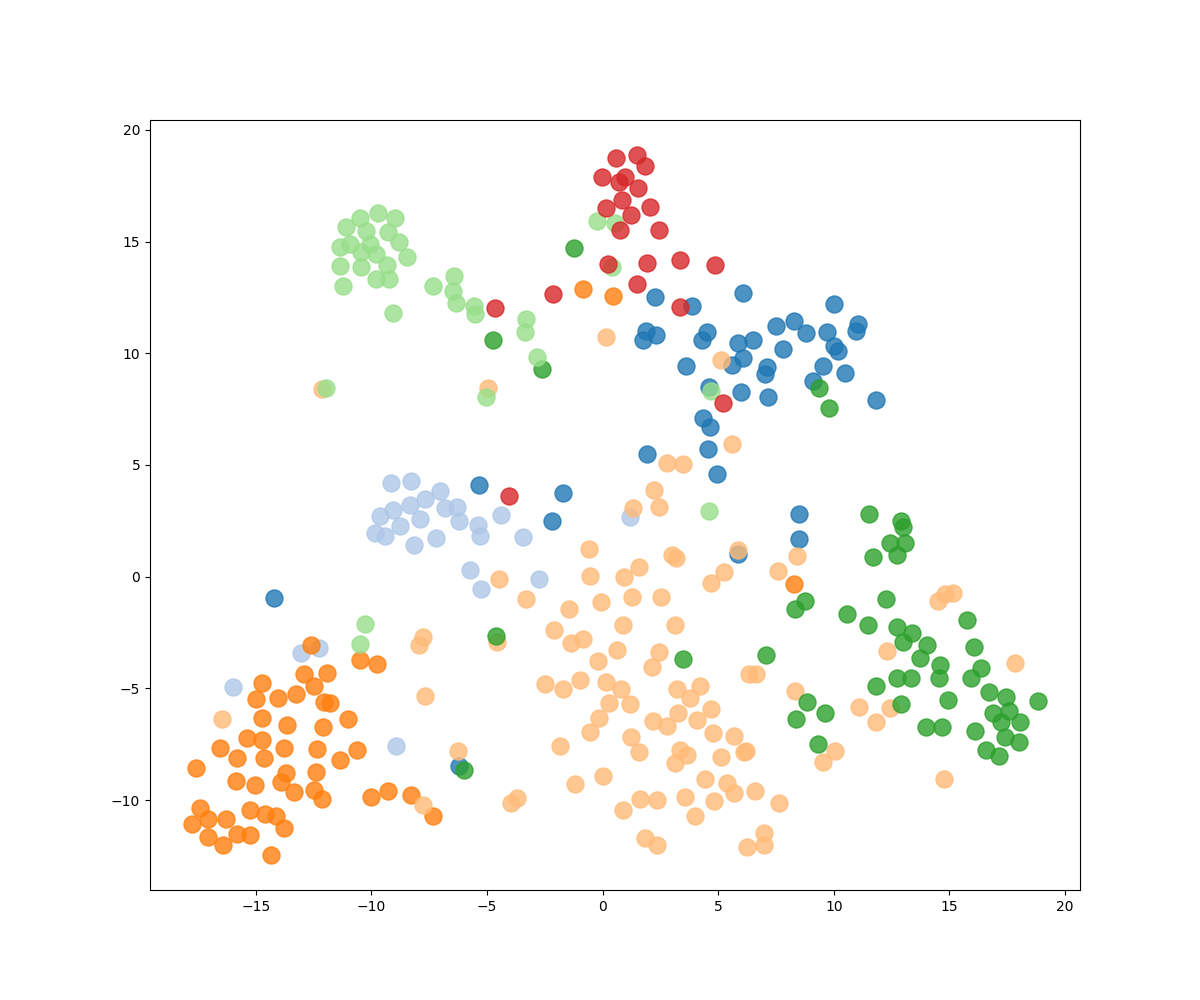}
    }
    \subfloat[{\em Cora} syn]{%
        \includegraphics[width=0.16\textwidth]{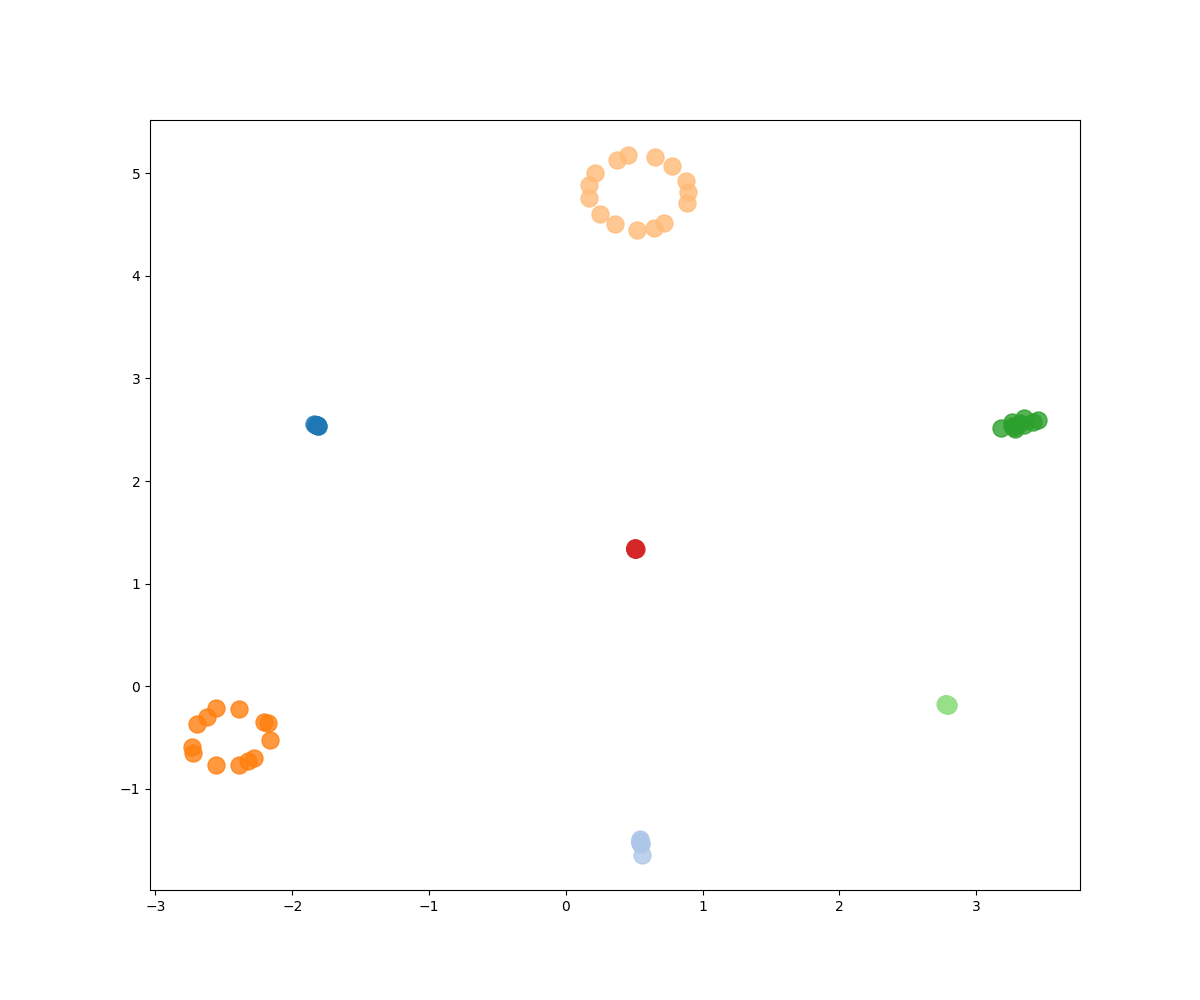}
    }

    \subfloat[{\em arXiv} w/o prop]{%
        \includegraphics[width=0.16\textwidth]{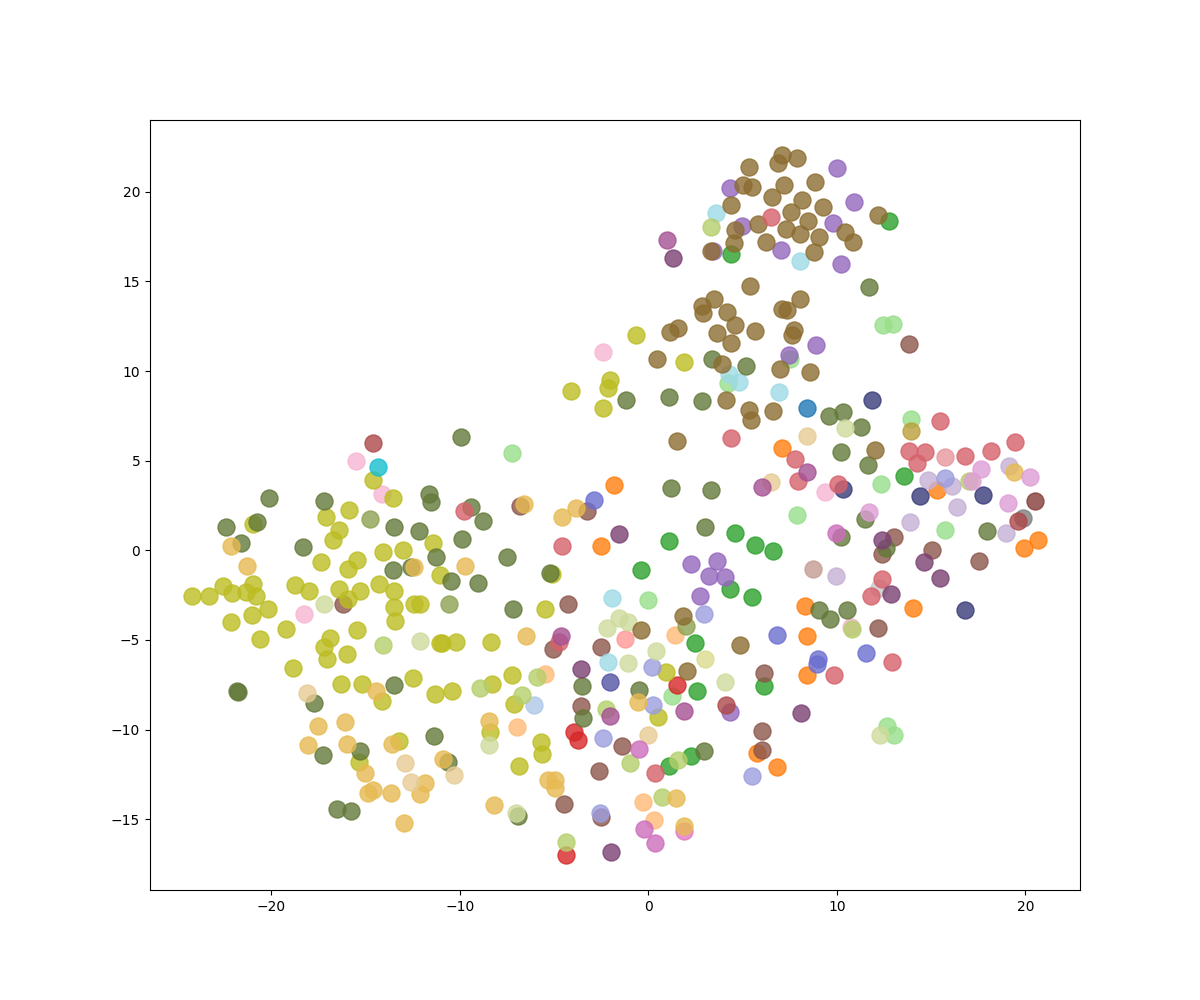}
    }
    \subfloat[{\em arXiv} w prop]{%
        \includegraphics[width=0.16\textwidth]{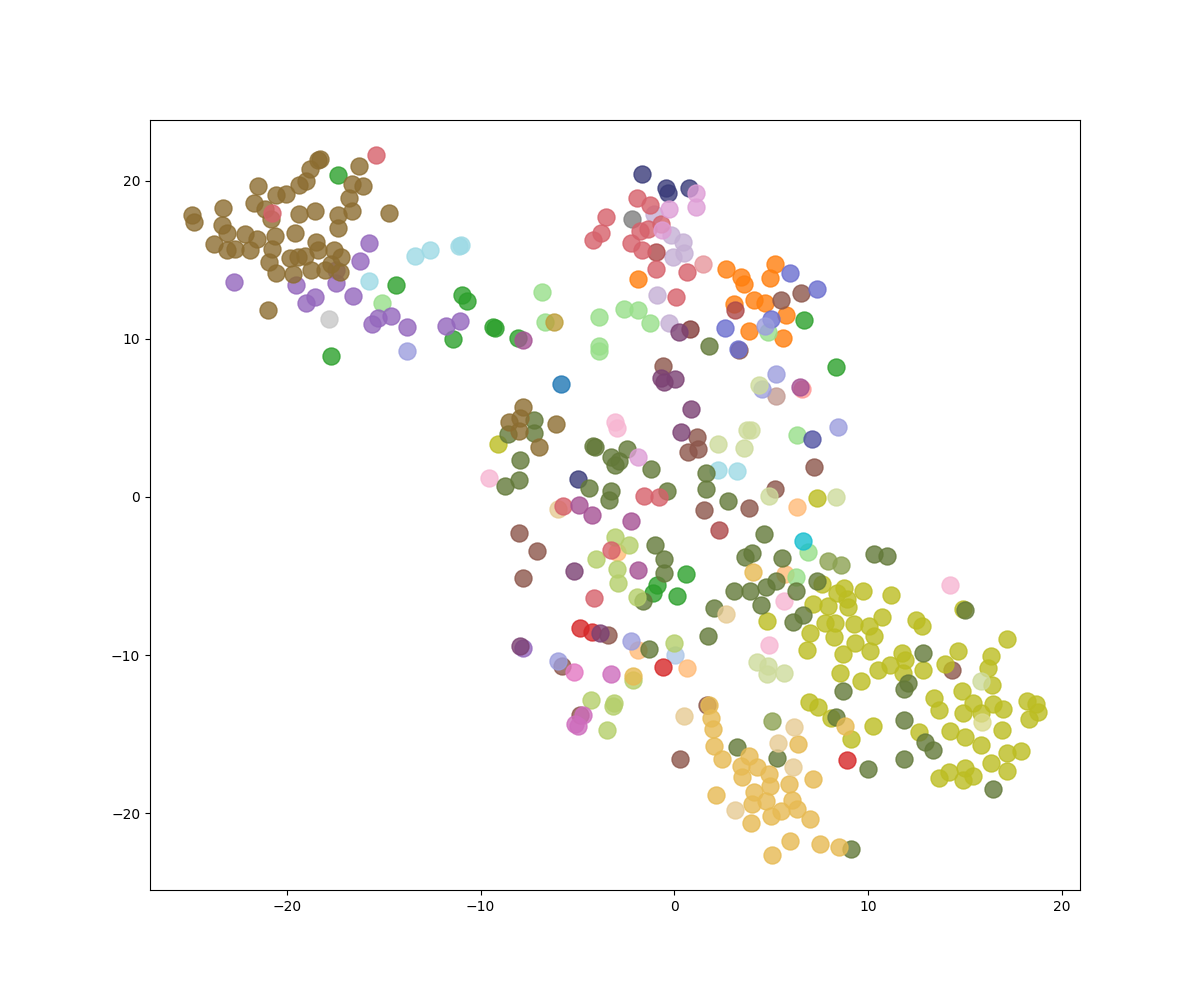}
    }
    \subfloat[{\em arXiv} syn]{%
        \includegraphics[width=0.16\textwidth]{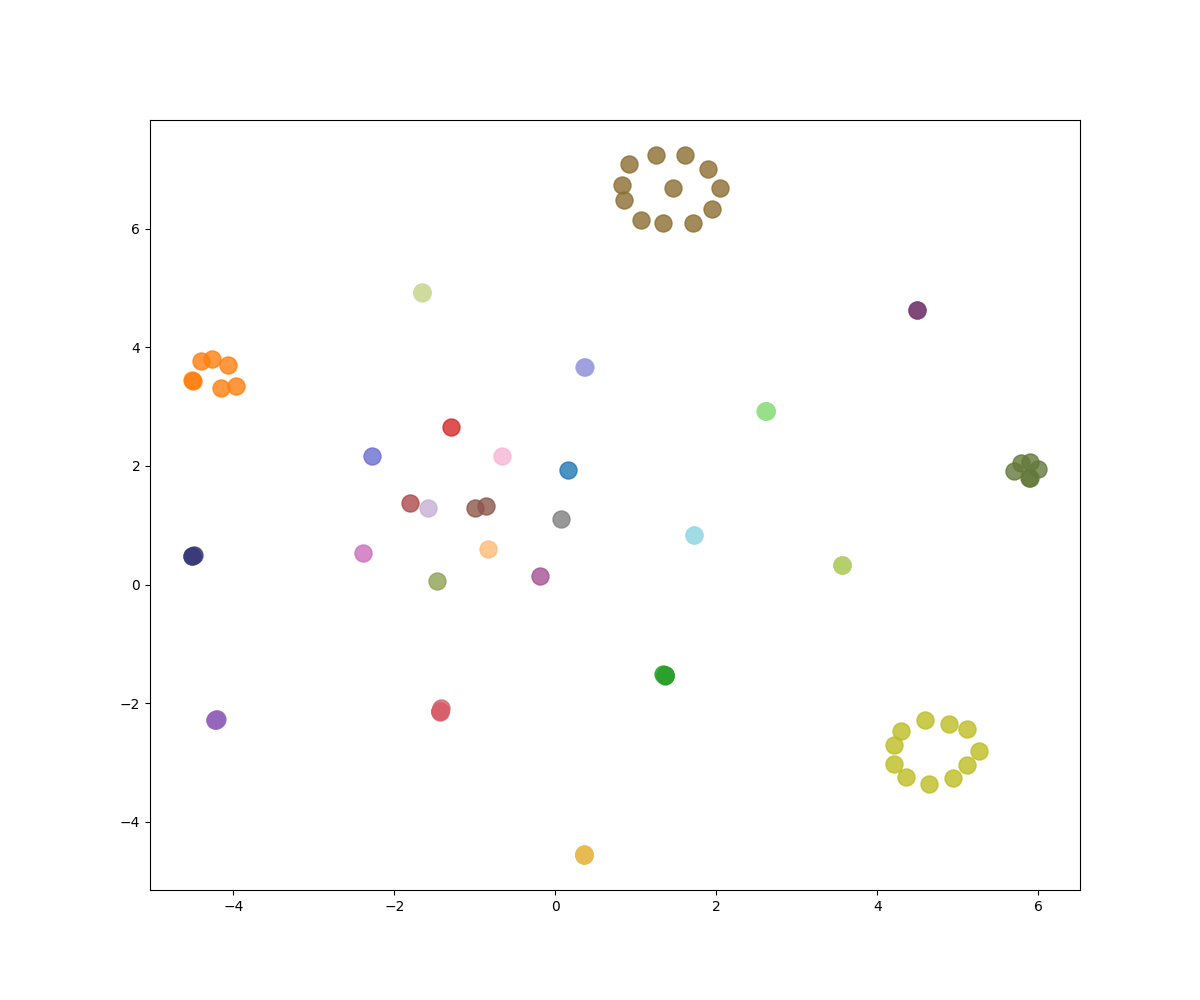}
    }

    \subfloat[{\em Reddit} w/o prop]{%
        \includegraphics[width=0.16\textwidth]{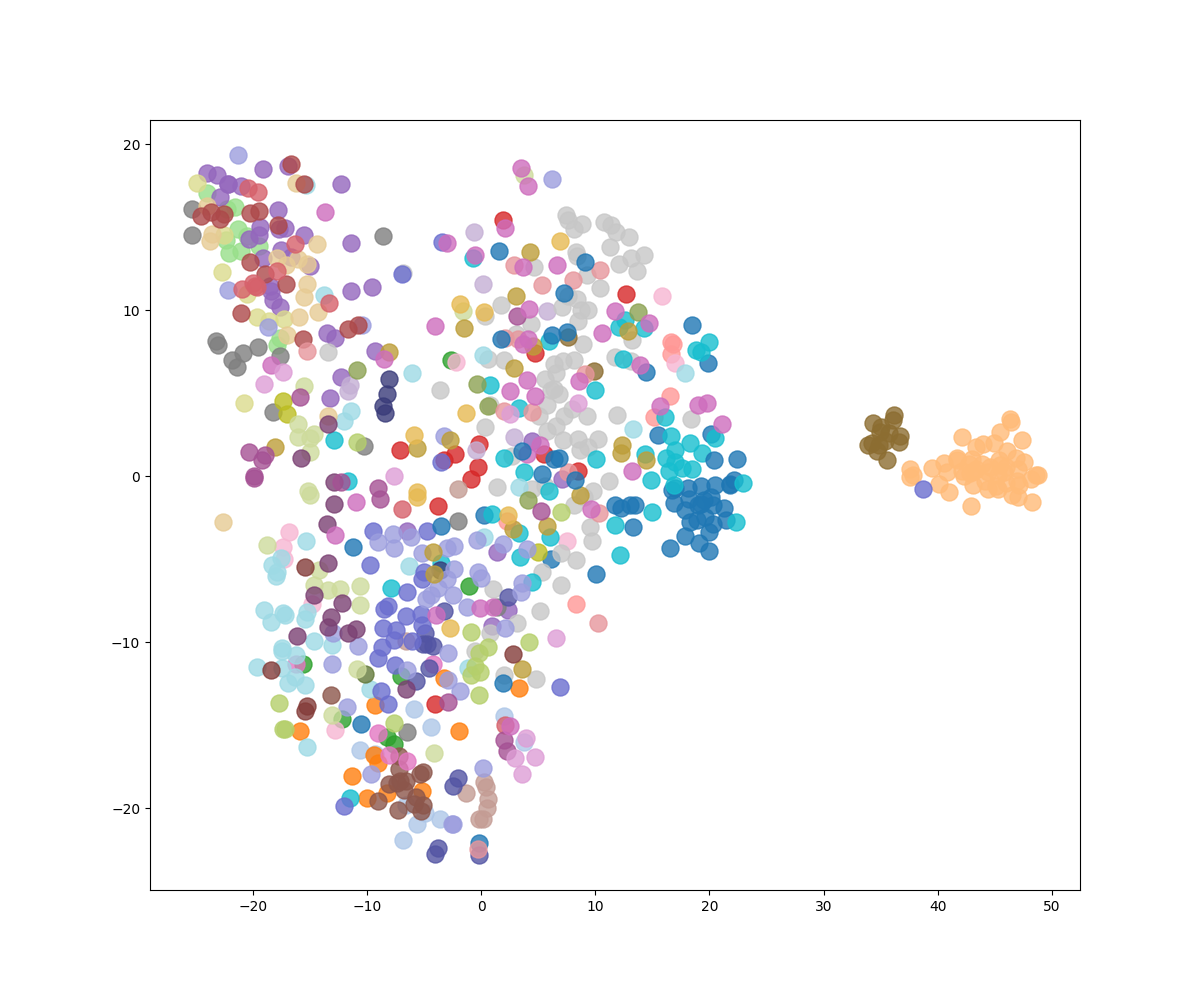}
    }
    \subfloat[{\em Reddit} w prop]{%
        \includegraphics[width=0.16\textwidth]{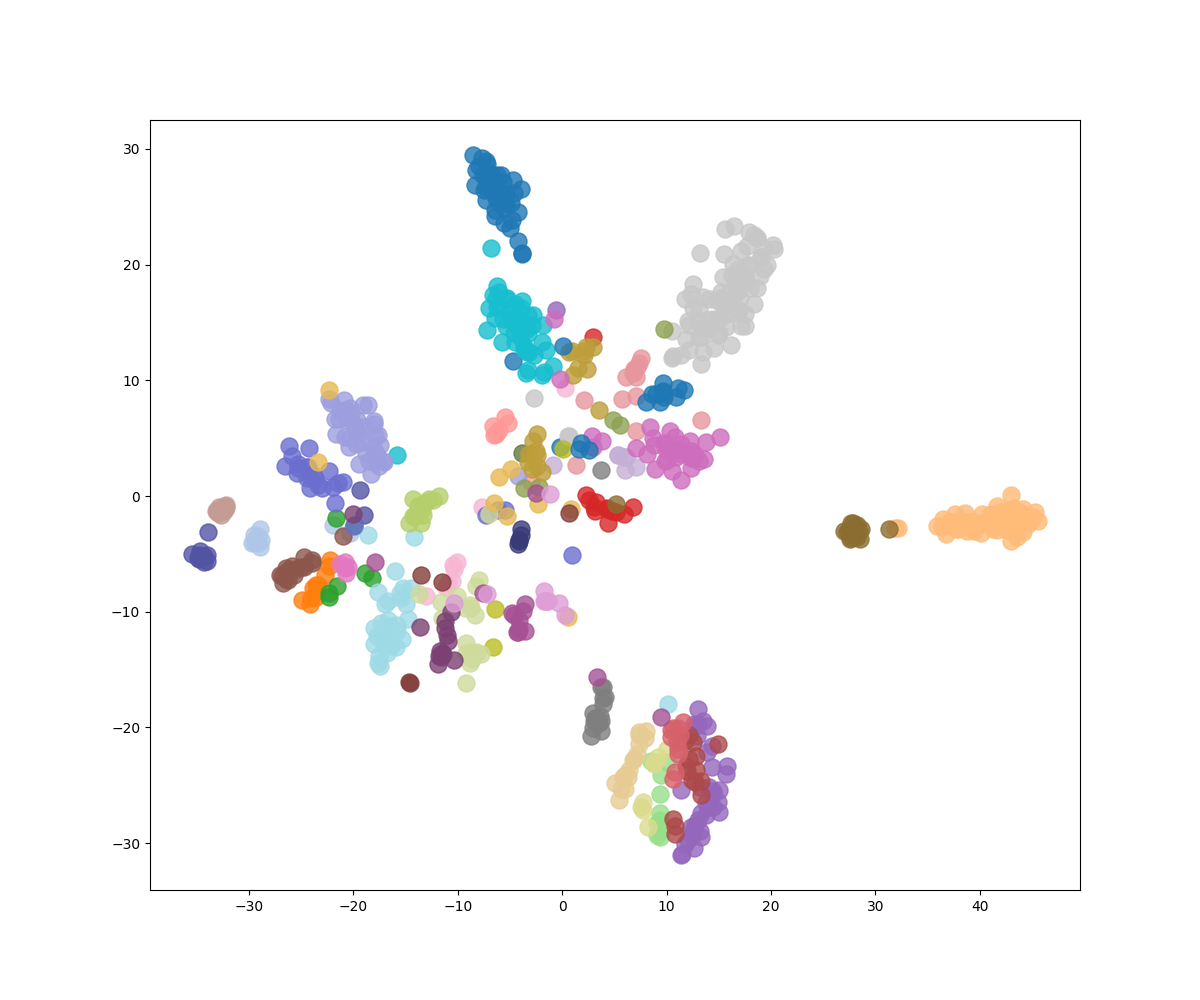}
    }
    \subfloat[{\em Reddit} syn]{%
        \includegraphics[width=0.16\textwidth]{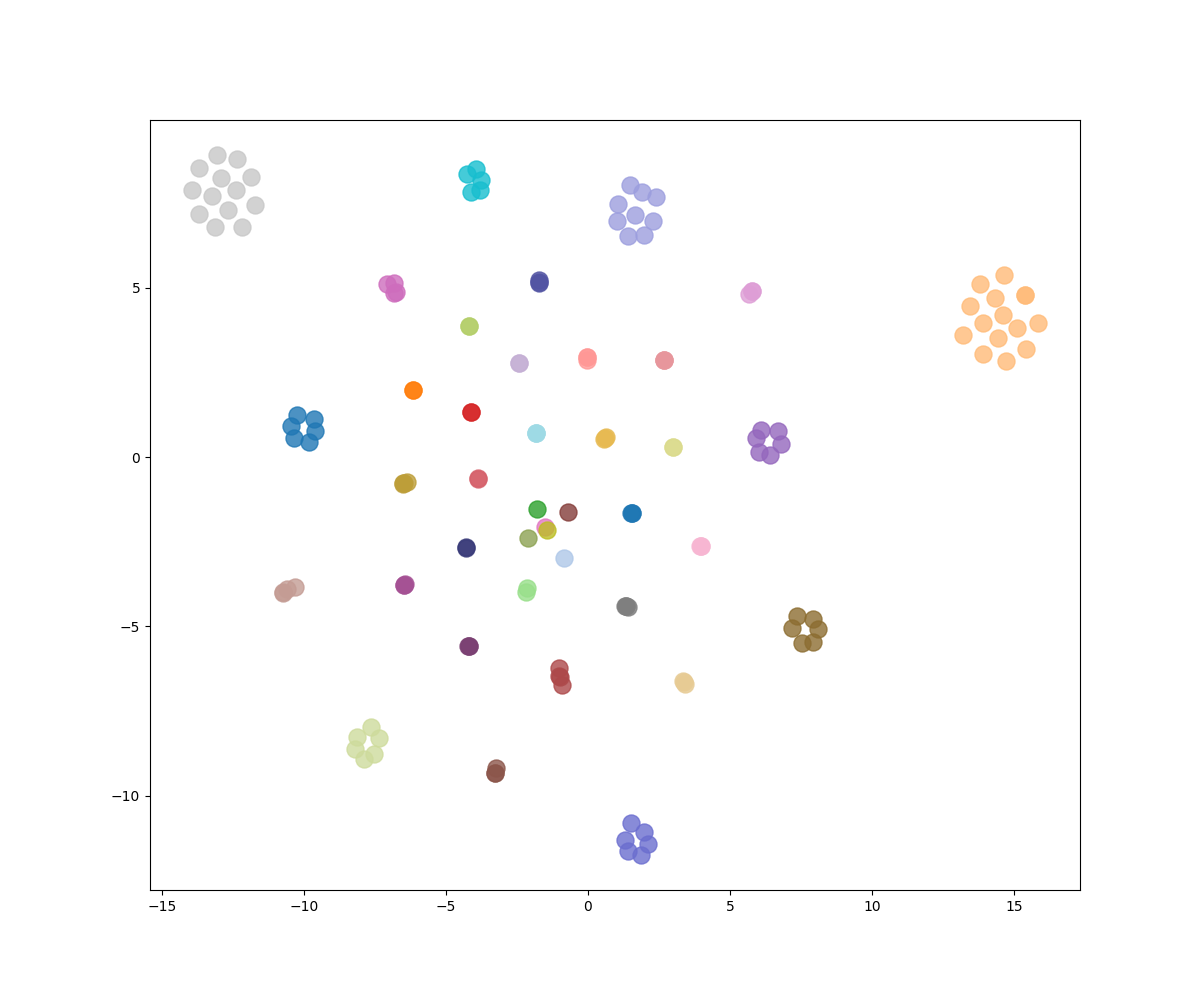}
    }
     \caption{The visualization of prediction logits on non-propagated/propagated/synthetic attributes}
    \vspace{-1ex}
    \label{fig:subfigures_prop}
\end{figure}

\subsection{Visualization}

\stitle{Propagation and Clustering}
We first employ the T-SNE algorithm to visualize the logits of node attributes that have and have not undergone propagation after passing through a linear layer during the pretraining phase, as well as the logits of the synthetic node attributes after being processed by the \texttt{GCN}. As shown in Fig.~\ref{fig:subfigures_prop}, scatters of different colors represent nodes from different categories. We can clearly observe that the logits obtained from attributes processed by propagation exhibit much clearer classification boundaries compared to those that have not undergone propagation processing. Regarding the logits from the \texttt{GCN} on the synthetic graph, although they are fewer in number, they display the clearest classification boundaries and also exhibit diversity within the classes. This reflects the high quality of the synthetic nodes attributes.

\stitle{Synthetic Graph Visualization}
Next, we visualize the synthetic subgraph in Fig.~\ref{fig:subfigures_vis}. We choose to visualize edges with weights above a certain threshold, with the color of the edges becoming darker as the weight increases. We find that the intra-class and inter-class connections of nodes in the synthetic subgraph reflect the homophily of the original graph, while the density of these connections indicates the density of the original graph. This demonstrates that our method has successfully synthetic informative graph data.

\begin{figure}[H]
    \centering
    \subfloat[{\em Cora}]{%
        \includegraphics[width=0.25\textwidth]{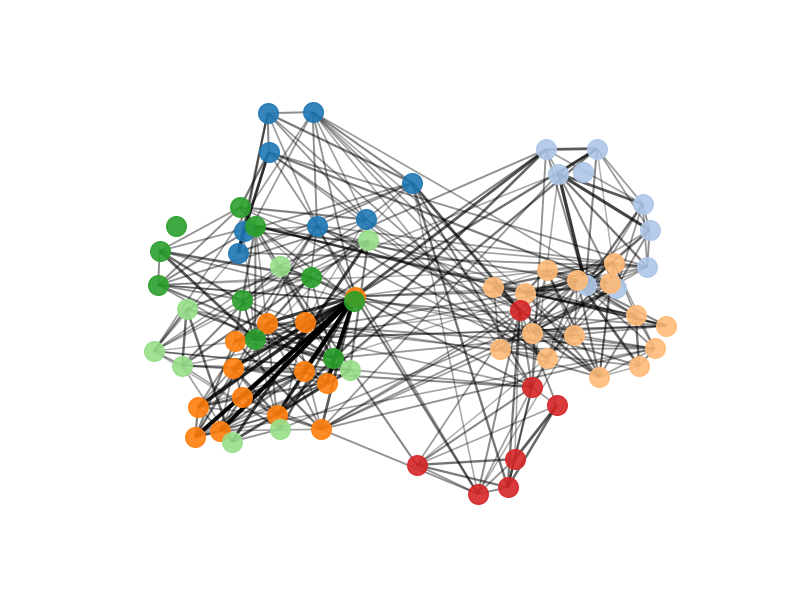}
    }
    \subfloat[{\em Citeseer}]{%
        \includegraphics[width=0.25\textwidth]{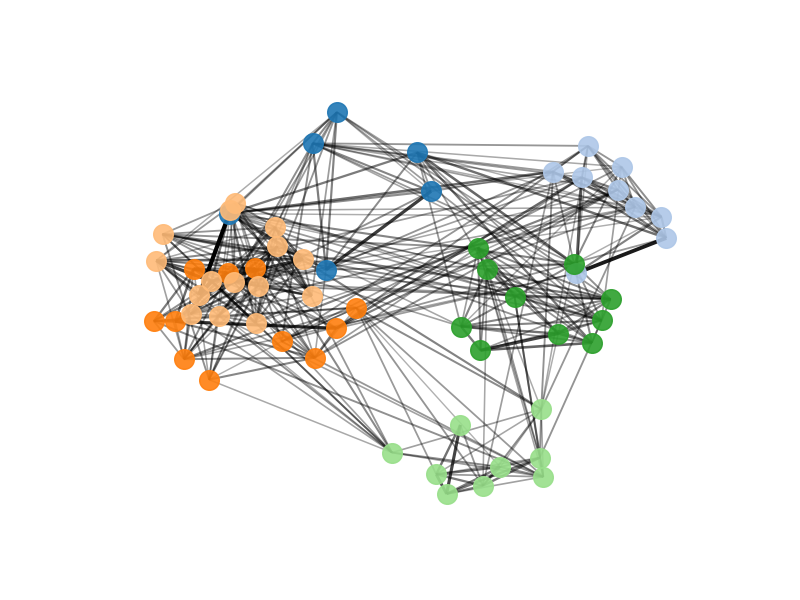} 
    }
    
    \subfloat[{\em arXiv}]{%
        \includegraphics[width=0.25\textwidth]{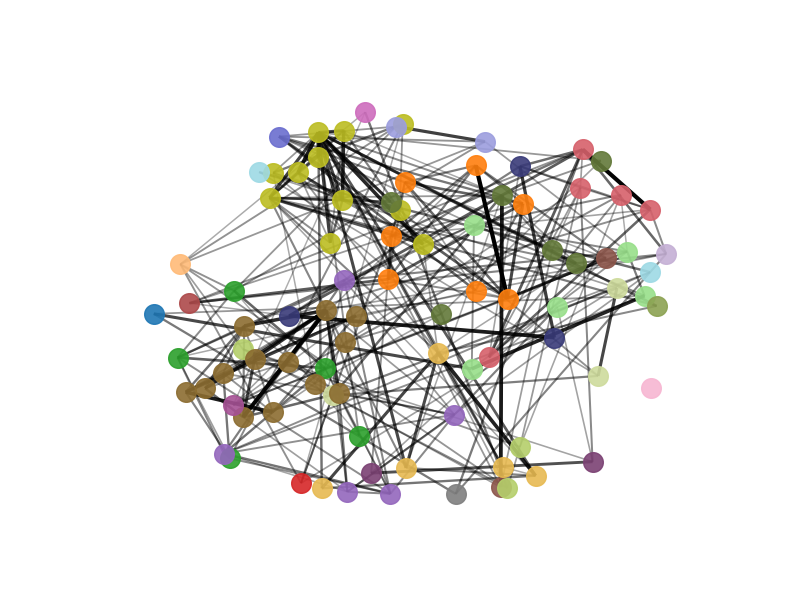}
    }
    \subfloat[{\em Reddit}]{%
        \includegraphics[width=0.25\textwidth]{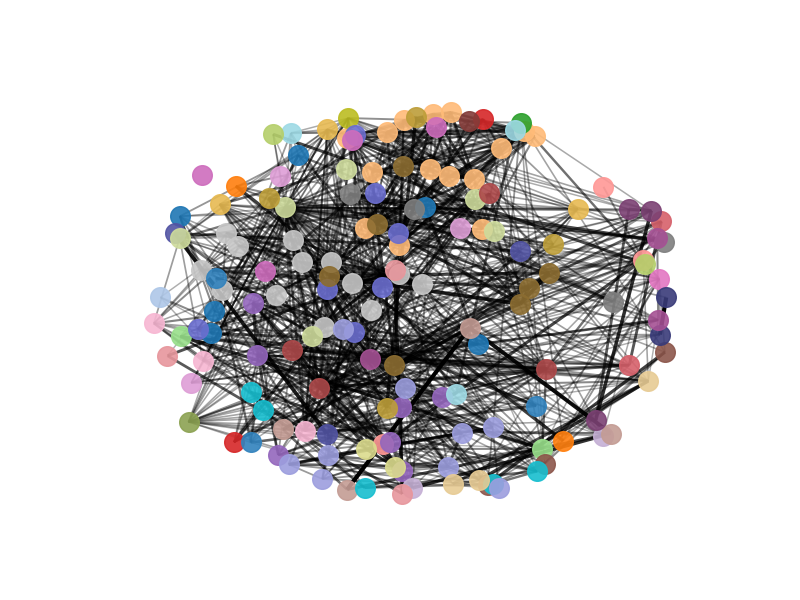}
    }
     \caption{The visualization synthetic graphs}
    \vspace{2ex}
    \label{fig:subfigures_vis}
\end{figure}

\section{Theoretical Proofs}
\label{sec:proof}
\begin{proof}[\bf Proof of Lemma~\ref{lem:mu}]
Using the Cauchy–Schwarz inequality, we can get
\begin{small}
\begin{align*}
& \|\muvec^{\text{org}} - \muvec^{\text{syn}}\|^2_2 =\left\|\frac{1}{N}\sum_{v_j\in  \G}{\HM_{j}} - \frac{1}{n}\sum_{u_i\in  \G'}{\HM'_{i}} \right\|^2_2 \\
&= \left\|\frac{1}{N}\sum_{i=1}^{n}{\sum_{v_j\in  C_i}{\HM_{j}}} - \frac{1}{n}\sum_{u_i\in  \G'}{\HM'_{i}} \right\|^2_2 = \left\|\sum_{i=1}^{n}{\left(\frac{{\HM'_{i}}}{n} - \sum_{v_j\in  C_i}{\frac{\HM_{j}}{N}}\right) }\right\|^2_2\\
&= \left\|\sum_{i=1}^{n}{\frac{1}{n}\cdot \left({{\HM'_{i}}} - \sum_{v_j\in  C_i}{\frac{n}{N}\cdot \HM_{j}}\right) }\right\|^2_2 \le \frac{1}{n}\sum_{i=1}^n\left\|{{\HM'_{i}}}- \sum_{v_j\in  C_i}{\frac{n}{N}\cdot \HM_{j}} \right\|^2_2\\
& = \frac{1}{n}\sum_{i=1}^n\left\|\sum_{v_j\in  C_i}{\left(\frac{1}{|C_i|}-\frac{n}{N}\right)\cdot \HM_{j}} \right\|^2_2.
\end{align*}
\end{small}
Since the Euclidean norm $\|\cdot\|_2^2$ is a convex function, using Jensen's Inequality leads to
\begin{small}
\begin{align*}
\frac{1}{n}\sum_{i=1}^n\left\|\sum_{v_j\in  C_i}{\left(\frac{1}{|C_i|}-\frac{n}{N}\right) \HM_{j}} \right\|^2_2 & \le \frac{1}{n}\sum_{i=1}^n\sum_{v_j\in  C_i}{\left(\frac{1}{|C_i|}-\frac{n}{N}\right)^2 \left\|\HM_{j} \right\|^2_2} \\
& = \frac{1}{n}\sum_{i=1}^n\sum_{v_j\in  C_i}{\left(\frac{1}{|C_i|}-\frac{n}{N}\right)^2} \\
& = \sum_{i=1}^n{\left( \frac{1}{n}-\frac{|C_i|}{N}\right)^2} = \frac{1}{N^2}\sum_{i=1}^n{\left( \frac{N}{n}-{|C_i|}\right)^2}.
\end{align*}
\end{small}
The lemma is proved.
\end{proof}

\begin{proof}[\bf Proof of Lemma~\ref{lem:conv-bound}]
First, according to the definitions in Eq.~\eqref{eq:conv}, $\SigM^{\textnormal{org}}$ and $\SigM^{\textnormal{syn}}$ are symmetric and positive semi-definite.
By their positive semi‐deﬁniteness and Araki–Lieb–Thirring inequality~\cite{araki1990inequality}, we have the following upper bound:
\begin{align}
& \Tr(\SigM^{\textnormal{org}}) + \Tr(\SigM^{\textnormal{syn}}) - 2\Tr((\SigM^{\textnormal{org}} \SigM^{\textnormal{syn}})^{\frac{1}{2}}) \notag\\
& \leq \Tr(\SigM^{\textnormal{org}}) + \Tr(\SigM^{\textnormal{syn}}) - 2\Tr({\SigM^{\textnormal{org}}}^{\frac{1}{2}} {\SigM^{\textnormal{syn}}}^{\frac{1}{2}}) \notag\\
& = \Tr\left( ({\SigM^{\textnormal{org}}}^{\frac{1}{2}}-{\SigM^{\textnormal{syn}}}^{\frac{1}{2}})^2 \right) \notag = \|{\SigM^{\textnormal{org}}}^{1/2}-{\SigM^{\textnormal{syn}}}^{1/2}\|^2_F \notag\\
& \le \|{\SigM^{\textnormal{org}}}^{1/2}\|^2_F+\|{\SigM^{\textnormal{syn}}}^{1/2}\|^2_F = \Tr(\SigM^{\textnormal{org}}) + \Tr(\SigM^{\textnormal{syn}}) \label{eq:org-syn-bound}.
\end{align}

Next, we need the following two lemmata.
\begin{lemma}\label{lem:org-conv}
The following equation holds:
\begin{small}
\begin{equation}\label{eq:trsigorg}
\Tr(\SigM^{\text{org}}) = \frac{1}{N} \sum_{i=1}^n\sum_{v_j\in C_i} \|\HM_j -\HM'_i\|^2_2 +  \sum_{i=1}^n \frac{|C_i|}{N}  \|\HM'_i - \muvec^{\text{org}}\|^2_2
\end{equation}    
\end{small}
\end{lemma}

\begin{lemma}\label{lem:syn-conv}
\(\sum_{i=1}^{n} ||\HM'_i - \muvec^{\text{org}}||^2_2 = n\cdot \left(\Tr(\SigM^{\text{syn}}) + \|\muvec^{\text{syn}}-\muvec^{\text{org}}\|_2^2\right)\)
\end{lemma}

Combining Lemmata~\ref{lem:org-conv}, ~\ref{lem:syn-conv}, and Eq.~\eqref{eq:org-syn-bound} leads to
\begin{align}
& \Tr(\SigM^{\text{org}}) +\Tr(\SigM^{\text{syn}})\notag\\
&= \frac{1}{N} \sum_{i=1}^n\sum_{v_j\in C_i} \|\HM_j -\HM'_i\|^2_2  + \sum_{i=1}^n \frac{|C_i|}{N}  \|\HM'_i - \muvec^{\text{org}}\|^2_2 + \Tr(\SigM^{\text{syn}}) \notag\\
&= \frac{1}{N} \sum_{i=1}^n\sum_{v_j\in C_i} \|\HM_j -\HM'_i\|^2_2 + \frac{c_{\max}}{N} \sum_{i=1}^n \|\HM'_i -\muvec^{\text{org}}\|^2_2+ \Tr(\SigM^{\text{syn}}) \notag\\
&= \frac{1}{N} \sum_{i=1}^n\sum_{v_j\in C_i} \|\HM_j -\HM'_i\|^2_2 + \left(\frac{nc_{\max}}{N}+1\right) \cdot \Tr(\SigM^{\text{syn}}) \notag\\
&\quad + \frac{nc_{\max}}{N}\cdot \|\muvec^{\text{syn}}-\muvec^{\text{org}}\|_2^2.\label{eq:tr-tr-sum}
\end{align}
From Lemma~\ref{lem:syn-conv}, we can further derive that 
\begin{align*}
\Tr(\SigM^{\text{syn}}) &\leq \frac{1}{n} \sum_{i=1}^n \|\HM'_i - \muvec^{\text{org}}\|^2_2 \leq \frac{1}{n}\cdot \frac{N}{c_{\min}}  \sum_{i=1}^n \frac{|C_i|}{N} \|\HM'_i - \muvec^{\text{org}}\|^2_2 \\
& \leq \frac{N}{n c_{\min}}\cdot\Tr(\SigM^{\text{org}}).
\end{align*}
Plugging the above inequality into Eq.~\eqref{eq:tr-tr-sum} finishes the proof.
\eat{
\begin{align*}
\Tr(\SigM^{\text{syn}}) &\leq \frac{1}{n} \sum_{i=1}^n \|\HM'_i - \muvec^{\text{org}}\|^2_2 \leq \frac{1}{n}\cdot \frac{N}{c_{\min}}  \sum_{i=1}^n \frac{|C_i|}{N} \|\HM'_i - \muvec^{\text{org}}\|^2_2 \\
& \leq \frac{N}{n c_{\min}}\cdot\Tr(\SigM^{\text{org}}).
\end{align*}
Plugging the above inbequality into Eq.~\eqref{eq:org-syn-bound} finishes the proof.
}
\end{proof}

\begin{proof}[\bf Proof of Lemma~\ref{lem:org-conv}]
Recall that $\SigM^{\text{org}}_{a,b}$ represents the covariance between the dimension $a$ and dimension $b$ of $\HM$. Then, 
\begin{equation}
    \SigM^{\text{org}}_{a,b}= \frac{1}{N}\sum^{N}_{v_j\in \G}(\HM_{j,a}-\muvec^{\text{org}}_{a})(\HM_{j,b}-\muvec^{\text{org}}_{b})
    \label{eq-elesigorg}
\end{equation}
where $\muvec^{\text{org}}_{a}, \muvec^{\text{org}}_{b}$ are the means of dimension $a$ and $b$ respectively. The covariance matrix can be decomposed into an intra-cluster covariance matrix and an inter-cluster covariance matrix. 
Note that we have
\begin{equation*}
\begin{gathered}
    \HM_{j,a} = \HM'_{i,a}+(\HM_{j,a}-\HM'_{i,a}),\\
    \HM_{j,b} = \HM'_{i,b}+(\HM_{j,b}-\HM'_{i,b}),\\
    \sum_{v_j\in C_i}(\HM_{j,a}-\HM'_{i,a})=\sum_{v_j\in C_i}(\HM_{j,b}-\HM'_{i,b})=0.
\end{gathered}
\end{equation*}
Then $\SigM^{\text{org}}_{a,b}$ can be rewritten as:
\begin{align}
    \SigM^{\text{org}}_{a,b}= & \frac{1}{N}\sum^{n}_{i=1}\sum_{v_j\in C_i}(
    (\HM_{j,a}-\HM'_{i,a})+(\HM'_{i,a}-\muvec^{\text{org}}_{a})) \notag\\
    & \cdot ((\HM_{j,b}-\HM'_{i,b})+(\HM'_{i,b}-\muvec^{\text{org}}_{b})) \notag\\
    =&\sum^{n}_{i=1}\frac{|C_i|}{N}\cdot (\HM'_{i,a}-\muvec^{\text{org}}_{a})(\HM'_{i,b}-\muvec^{\text{org}}_{b})\notag\\
    & +\frac{1}{N}\sum^{n}_{i=1}\sum_{v_j\in C_i}(\HM_{j,a}-\HM'_{i,a})(\HM_{j,b}-\HM'_{i,b})\notag\\
    & + \frac{1}{N}\sum^{n}_{i=1}\sum_{v_j\in C_i}(\HM_{j,a}-\HM'_{i,a})(\HM'_{i,b}-\muvec^{\text{org}}_{b})\notag\\
    & + \frac{1}{N}\sum^{n}_{i=1}\sum_{v_j\in C_i}(\HM_{j,b}-\HM'_{i,b})(\HM'_{i,a}-\muvec^{\text{org}}_{a})\notag
\end{align}
\begin{align}
    \SigM^{\text{org}}_{a,b}=
    =&\sum^{n}_{i=1}\frac{|C_i|}{N}\cdot (\HM'_{i,a}-\muvec^{\text{org}}_{a})(\HM'_{i,b}-\muvec^{\text{org}}_{b})\notag\\
    & +\frac{1}{N}\sum^{n}_{i=1}\sum_{v_j\in C_i}(\HM_{j,a}-\HM'_{i,a})(\HM_{j,b}-\HM'_{i,b})\label{eq:rwttrorg}
\end{align}
According to Eq.~\eqref{eq:rwttrorg}, 
\begin{equation*}
\Tr(\SigM^{\text{org}}) = \frac{1}{N} \sum_{i=1}^n\sum_{v_j\in C_i} \|\HM_j -\HM'_i\|^2_2 +  \sum_{i=1}^n \frac{|C_i|}{N}  \|\HM'_i - \\\muvec^{\text{org}}\|^2_2,
\end{equation*}
which completes the proof.
\end{proof}

\begin{proof}[\bf Proof of Lemma~\ref{lem:syn-conv}]
\eat{
Similar to Eq.\eqref{eq:trsigorg}, the trace of the $\SigM^{\text{syn}}$ can be calculated by,
\begin{equation*}
    \Tr(\SigM^{\text{syn}}) = \frac{1}{n} \sum_{i=1}^n \|\HM'_i -\muvec^{\text{syn}}\|^2_2 = \frac{1}{n} \sum_{i=1}^{n} (\HM'_i - \muvec^{\text{syn}})^\top (\HM'_i - \muvec^{\text{syn}})
\end{equation*}
}
According to Eq.~\eqref{eq:mu-org-syn}, \( \sum_{i=1}^{n} ||\HM'_i - \muvec^{\text{org}}||^2_2\) can be expanded as
\begin{equation*}
   \sum_{i=1}^{n} ||\HM'_i - \muvec^{\text{org}}||^2_2 = \sum_{i=1}^{n} (\HM'_i - \muvec^{\text{org}})^\top (\HM'_i - \muvec^{\text{org}}) .
\end{equation*}
Let  \(\Delta = \muvec^{\text{org}} - \muvec^{\text{syn}}\), we have 
\begin{align*}
& \sum_{i=1}^{n} (\HM'_i - \muvec^{\text{org}} )^\top (\HM'_i - \muvec^{\text{org}})\\
&=\sum_{i=1}^{n} (\HM'_i - \muvec^{\text{syn}} +\Delta)^\top (\HM'_i - \muvec^{\text{syn}} +\Delta)\\&
    =\sum_{i=1}^{n} \left( (\HM'_i - \muvec^{\text{syn}})^\top (\HM'_i - \muvec^{\text{syn}}) + 2(\HM'_i - \muvec^{\text{syn}})^\top \Delta + \Delta^\top\Delta \right)
\end{align*}
Since \(\sum_{i=1}^{n} (\HM'_i - \muvec^{\text{syn}}) = 0\),
we have 
\begin{align*}
\sum_{i=1}^{n} ||\HM'_i - \muvec^{\text{org}}||^2_2 & = \sum_{i=1}^n \left(\|\HM'_i -\muvec^{\text{syn}}\|^2_2 + \|\muvec^{\text{syn}}-\muvec^{\text{org}}\|_2^2\right)\\
& = n\cdot \left(\Tr(\SigM^{\text{syn}}) + \|\muvec^{\text{syn}}-\muvec^{\text{org}}\|_2^2\right).
\end{align*}
The lemma then follows.
\eat{
\begin{align*}
&\frac{1}{n} \sum_{i=1}^{n} 2(\HM'_i - \muvec^{\text{syn}})^\top \Delta + \frac{1}{n} \sum_{i=1}^{n} \Delta^\top\Delta = \frac{1}{n} \sum_{i=1}^{n} \Delta^\top\Delta \geq 0.
\end{align*}
Accordingly,
\begin{align*}
\Tr(\SigM^{\text{syn}})
     & =\frac{1}{n} \sum_{i=1}^{n} ||\HM'_i - \muvec^{\text{syn}}||^2_2 \leq \frac{1}{n} \sum_{i=1}^{n} \left( ||\HM'_i - \muvec^{\text{syn}}||^2_2+ \Delta^\top\Delta\right)\\
     &= \frac{1}{n} \sum_{i=1}^{n} ||\HM'_i - \muvec^{\text{org}}||^2_2.
\end{align*}
}
\end{proof}

\begin{proof}[\bf Proof of Lemma~\ref{lem:GSL-sol}]
First, the second term is equivalent to
\begin{equation*}
\sum_{(v_i,v_j)\in \EDG}{\left\|\frac{\HM_i}{\sqrt{d(v_i)}}-\frac{\HM_j}{\sqrt{d(v_j)}}\right\|_F^2} = \Tr(\HM^{\top}(\IM-\NAM)\HM).
\end{equation*}
By setting its derivative w.r.t. $\HM$ to zero, we obtain the optimal $\HM$ as:
\begin{align}
& \frac{\partial{\{(1-\alpha)\cdot\|\HM - \XM\WM\|^2_F + \alpha \cdot \Tr(\HM^{\top}(\IM-\NAM)\HM)\}}}{\partial{\HM}}=0 \notag\\
& \Longrightarrow (1-\alpha)\cdot(\HM - \XM\WM) + \alpha (\IM-\NAM)\HM = 0 \notag\\
& \Longrightarrow \HM = (1-\alpha)\cdot \left(\IM-\alpha\NAM\right)^{-1} \XM\WM. \label{eq:Z-derivative}
\end{align}
By the property of Neumann series, we have 
$(\IM - \alpha \NAM)^{-1} = \sum_{\ell=0}^{\infty}{\alpha^t \NAM^{\ell}}$. Plugging it into Eq.~\eqref{eq:Z-derivative} completes the proof.
\end{proof}

\begin{proof}[\bf Proof of Lemma~\ref{lem:homophily-DE}]
Recall that the definition of homophily ratio over graph $\G$ is the fraction of edges whose endpoints are in the same class. Thus,
\begin{align*}
\Omega(\G) & = \frac{\sum_{(v_i,v_j)\in \EDG}{\sum_{k=1}^{K}\YM_{i,k}\cdot \YM_{j,k}}}{m}\\
& = - \frac{\sum_{(v_i,v_j)\in \EDG}{\sum_{k=1}^{K}{\YM_{i,k}^2-\YM_{i,k}^2+\YM_{j,k}^2-\YM_{j,k}^2-2\YM_{i,k}\cdot \YM_{j,k}}}}{2M} \\
& = \frac{\sum_{(v_i,v_j)\in \EDG}{\sum_{k=1}^{K}{\YM_{i,k}^2+\YM_{j,k}^2}}}{2m}\\
& \quad - \frac{\sum_{(v_i,v_j)\in \EDG}{\sum_{k=1}^{K}{\YM_{i,k}^2+\YM_{j,k}^2-2\YM_{i,k}\cdot \YM_{j,k}}}}{2M} \\
& = 1 - \frac{\sum_{(v_i,v_j)\in \EDG}{\sum_{k=1}^{K}{(\YM_{i,k}-\YM_{j,k})^2}}}{2m}\\
& = 1 - \frac{\sum_{(v_i,v_j)\in \EDG}{\|\YM_i-\YM_j\|_2^2}}{2M}\\
& = 1 - \frac{\sum_{(v_i,v_j)\in \EDG}{\left\|\frac{(\DM^{1/2}\YM)_i}{\sqrt{d_i}}-\frac{(\DM^{1/2}\YM)_j}{\sqrt{d_j}}\right\|_2^2}}{2M}.
\end{align*}
which finishes the proof.
\end{proof}

\end{document}